\documentclass[11pt]{article}
\usepackage{fullpage}
\usepackage{mathtools, amsfonts, amsthm, amssymb,amsmath}
\usepackage{thm-restate}
\usepackage{color}
\usepackage{todonotes}
\usepackage{algorithm}
\usepackage[noend]{algpseudocode}
\usepackage[numbers]{natbib}
\usepackage{hyperref}
\usepackage{bm}
\usepackage{enumitem}
\usepackage{caption}
\usepackage{booktabs}
\usepackage{tabularx}
\usepackage{bbm}
\usepackage{subcaption}

\newcommand \cA {\mathcal{A}}

\newcommand \cB {\mathcal{B}}
\newcommand \cF {\mathcal{F}}
\newcommand \cG {\mathcal{G}}

\newcommand \cN {\mathcal{N}}
\newcommand \cP {\mathcal{P}}

\newcommand \reals {\mathbb{R}}
\newcommand \R {\mathbb{R}}
\newcommand \expect {\mathbb{E}}
\newcommand \E {\mathop{\mathbb{E}}}
\newcommand \prob {\operatorname*{Pr}}
\newcommand \OPT {\mathrm{OPT}}
\newcommand \vol {\operatorname{Vol}}
\newcommand \ind [1]{\mathbb{I}\{#1\}}
\newcommand \argmin {\operatorname*{argmin}}
\newcommand \argmax {\operatorname*{argmax}}
\newcommand \sample {\mathcal{S}}
\newcommand \dist {\mathcal{D}}

\renewcommand \vec [1]{\bm{#1}}

\newcommand \uniform {\operatorname{Uniform}}
\newcommand \partition {\mathcal{P}}

\newcommand \empvec [1] {\hat{\vec{#1}}}
\newcommand \configs {\mathcal{C}}
\newcommand \norm [1] {\Vert#1\Vert}

\newcommand \expmu {\mu_{\rm exp}}
\newcommand \expmf {f_{\rm exp}}

\newcommand \alg {\mathcal{A}}
\newcommand \intalg {\mathcal{A}_{\rm integrate}}
\newcommand \samplealg {\mathcal{A}_{\rm sample}}
\newcommand \reldist {D_\infty}
\newcommand \setc [2]{\{ #1 \,:\, #2 \}}
\newcommand \uslin {u_{\rm slin}}
\newcommand \uowr {u_{\rm owr}}
\newcommand \nn {\operatorname{NN}}
\newcommand \pdim {\operatorname{Pdim}}
\newcommand \numfunctions {T}

\def\qp{A}
\DeclareMathOperator{\sign}{sgn}

\theoremstyle{plain}
\newtheorem{thm}{Theorem}
\newtheorem*{thm*}{Theorem}
\newtheorem{theorem}[thm]{Theorem}
\newtheorem{lem}{Lemma}
\newtheorem{claim}{Claim}
\newtheorem{cor}{Corollary}
\newtheorem{assumption}{Assumption}

\theoremstyle{definition}
\newtheorem{defn}{Definition}

\theoremstyle{remark}
\newtheorem{remark}{Remark}

\allowdisplaybreaks

\begin{document}

\title{Dispersion for Data-Driven Algorithm Design, Online Learning, and Private Optimization}
\author{Maria-Florina Balcan \and Travis Dick \and Ellen Vitercik}

\maketitle

\begin{abstract}
A crucial problem in modern data science is data-driven algorithm design, where
the goal is to choose the best algorithm, or algorithm parameters, for a
specific application domain. In practice, we often optimize over a parametric
algorithm family, searching for parameters with high performance on a collection
of typical problem instances. While effective in practice, these procedures
generally have not come with provable guarantees. A recent line of work
initiated by a seminal paper of Gupta and Roughgarden \cite{Gupta17:PAC}
analyzes application-specific algorithm selection from a theoretical
perspective. We progress this research direction in several important settings.
We provide upper and lower bounds on regret for algorithm selection in online
settings, where problems arrive sequentially and we must choose parameters
online.
We also consider differentially private algorithm selection, where the
goal is to find good parameters for a set of problems without divulging too
much sensitive information contained therein.

We analyze several important parameterized families of algorithms, including
SDP-rounding schemes for problems formulated as integer quadratic programs as
well as greedy techniques for several canonical subset selection problems. The
cost function that measures an algorithm's performance is often a volatile
piecewise Lipschitz function of its parameters, since a small change to the
parameters can lead to a cascade of different decisions made by the algorithm.
We present general techniques for optimizing the sum or average
of piecewise Lipschitz functions when the underlying functions satisfy a
sufficient and general condition called \emph{dispersion}. Intuitively, a set of piecewise Lipschitz functions is dispersed if no small region contains many of the functions' discontinuities.

Using dispersion, we improve over the best-known online learning regret bounds
for a variety problems, prove regret bounds for problems not previously studied,
and provide matching regret lower bounds. In the private optimization setting,
we show how to optimize performance while preserving privacy for several
important problems, providing matching upper and lower bounds on performance
loss due to privacy preservation. Though algorithm selection is our primary
motivation, we believe the notion of dispersion may be of independent interest.
Therefore, we present our results for the more general problem of optimizing
piecewise Lipschitz functions. Finally, we uncover dispersion in domains beyond
algorithm selection, namely, auction design and pricing, providing online and
privacy guarantees for these problems as well.
\end{abstract}

\setcounter{page}{0}
\thispagestyle{empty}
\newpage

\section{Introduction}
Data-driven algorithm design, that is, choosing the best algorithm for a specific application, is a critical problem in modern data science and algorithm design. Rather than use off-the-shelf algorithms with only worst-case guarantees, a practitioner will often optimize over a family of parametrized algorithms, tuning the algorithm's parameters based on typical problems from his domain. Ideally, the resulting algorithm will have high performance on future problems, but these procedures have historically come with no guarantees. In a seminal work, Gupta and Roughgarden~\cite{Gupta17:PAC} study algorithm selection in a distributional learning setting. Modeling an application domain as a distribution over typical problems, they show that a bound on the intrinsic complexity of the algorithm family prescribes the number of samples sufficient to ensure that any algorithm's empirical and expected performance are close.

We advance the foundations of algorithm selection in several important directions: online and private algorithm selection. In the online setting, problem instances arrive one-by-one, perhaps
adversarially. The goal is to select parameters for each instance in order to
minimize \emph{regret}, which is the difference between the cumulative
performance of those parameters and the optimal
parameters in hindsight.
We also
study private algorithm selection, where the goal is to find
high-performing parameters over a set of problems without
revealing sensitive information contained therein. Preserving privacy is
crucial when problems depend on individuals' medical or
purchase data, for example.

We analyze several important, infinite families of parameterized
algorithms. These include greedy techniques for canonical subset selection problems such as the knapsack and maximum weight independent set problems. We also study SDP-rounding schemes for problems that can be formulated as integer quadratic programs, such as max-cut, max-2sat, and correlation clustering.
In these cases, our goal is to optimize, online or privately, the utility function that measures an algorithm's performance as a function of its parameters, such as the value of the items added to the knapsack by a parameterized knapsack algorithm. The key challenge is the volatility of this function: a small tweak to the algorithm's parameters can
cause a cascade of changes in the algorithm's behavior.
 For example, greedy algorithms typically build a
solution by iteratively adding items that
maximize a scoring rule. Prior work has proposed parameterizing
these scoring rules and tuning the parameter to obtain the best performance for
a given application~\citep{Gupta17:PAC}. Slightly adjusting the parameter can cause the algorithm to select items in a completely different order, potentially causing a sharp change in the quality of the selected items.

Despite this challenge, we show that in many cases, these utility functions are well-behaved in several respects and thus can be optimized online and privately. Specifically, these functions are piecewise Lipschitz and moreover, they satisfy a condition we call \emph{dispersion}. Roughly speaking, a collection of piecewise Lipschitz functions is \emph{dispersed} if no small region of space contains discontinuities for many of the functions. We provide general techniques for online and private optimization of the sum or average of dispersed piecewise Lipschitz functions. Taking advantage of dispersion in online learning, we improve over the best-known regret bounds for a
variety problems, prove regret bounds for problems not
previously studied, and provide matching regret lower bounds. In the privacy setting, we show how to optimize performance while preserving privacy for several important problems,
giving matching upper and lower bounds on performance loss due to privacy.

Though our main motivation is algorithm
selection, we expect dispersion is even more widely applicable, opening up
an exciting research direction. For this reason, we present our main results
more generally for optimizing piecewise Lipschitz functions. We also uncover
dispersion in domains beyond algorithm selection, namely, auction design and
pricing, so we prove online and privacy guarantees
for these problems as well.
 Finally,
we answer several open questions: Cohen-Addad and Kanade~\cite{Cohen-Addad17:Online} asked how to
optimize piecewise Lipschitz functions and Gupta and Roughgarden~\cite{Gupta17:PAC} asked which
algorithm selection problems can be solved with no regret algorithms. As a
bonus, we also show that dispersion implies generalization guarantees in the
distributional setting. In this setting, the configuration procedure is given an iid sample of
problem instances drawn from an unknown distribution $\dist$, and the goal is to
find the algorithm parameters with highest expected utility. By bounding the empirical
Rademacher complexity, we show that the sample and expected utility for all
algorithms in our class are close, implying that the optimal algorithm on the
sample is approximately optimal in expectation.

\subsection{Our contributions}\label{sec:IntroSetting}

In order to present our contributions, we briefly outline the notation we will use.
Let $\mathcal{A}$ be an infinite set of algorithms
parameterized by a set $\configs \subseteq \R^d$. For example, $\mathcal{A}$
might be the set of knapsack greedy algorithms that add items to the knapsack in
decreasing order of $v(i)/s(i)^\rho$, where $v(i)$ and $s(i)$ are the value and
size of item $i$ and $\rho$ is a parameter. Next, let $\Pi$ be a set of problem
instances for $\alg$, such as knapsack problem instances, and let $u : \Pi \times
\configs \to [0,H]$ be a utility function where $u(x, \vec{\rho})$ measures the
performance of the algorithm with parameters $\vec{\rho}$ on problem instance $x \in
\Pi$. For example, $u(x, \rho)$ could be the value of the items chosen by
the knapsack algorithm with parameter $\rho$ on input $x$.

We now summarize our main contributions. Since our results apply beyond application-specific algorithm selection, we describe them for the more general problem of optimizing piecewise Lipschitz functions.

\paragraph{Dispersion} Let $u_1, \dots, u_\numfunctions$ be a set of functions mapping a set $\configs \subseteq \R^d$ to $[0,H]$. For example, in the application-specific algorithm selection setting, given a collection of
problem instances $x_1, \dots, x_\numfunctions \in \Pi$ and a utility function $u : \Pi \times
\configs \to [0,H]$, each function $u_i(\cdot)$ might equal the function $u(x_i, \cdot)$, measuring an algorithm's performance on a fixed problem instance as a function of its parameters. Dispersion is a constraint on the functions
$u_1, \dots, u_\numfunctions$. We assume that for each function $u_i$, we can partition $\configs$ into sets $\configs_1, \dots, \configs_K$ such
that $u_i$ is $L$-Lipschitz on each piece, but $u_i$ may have
discontinuities at the boundaries between pieces.
In our applications, each set $\configs_i$ is connected, but our general results
hold for arbitrary sets. Informally, the functions $u_1, \dots, u_\numfunctions$ are $(w,k)$-dispersed if every Euclidean ball of radius $w$ contains
discontinuities for at most $k$ of those functions (see
Section~\ref{sec:dispersion} for a formal definition). This guarantees that
although each function $u_i$ may have discontinuities, they do not
concentrate in a small region of space. Dispersion is sufficient to prove strong
learning generalization guarantees, online learning regret bounds, and private
optimization bounds when optimizing the empirical utility
$\frac{1}{\numfunctions}\sum_{i=1}^\numfunctions u_i$. In our applications, $w = \numfunctions^{\alpha - 1}$ and $k = \tilde O(\numfunctions^{\alpha})$ with high
probability for any $1/2 \leq \alpha \leq 1$, ignoring  problem-specific multiplicands.

\paragraph{Online learning} We prove that dispersion implies strong regret
bounds in online learning, a fundamental area of machine
learning~\citep{NicoloGabor06:PLG}. In this setting, a sequence of functions $u_1, \dots, u_\numfunctions$ arrive one-by-one. At time $t$, the learning
algorithm chooses a parameter vector $\vec{\rho}_t$ and then either observes the
function $u_t$ in the full information setting or the scalar $u_t(\vec{\rho}_t)$ in the bandit setting. The goal is to minimize expected regret:
$\E[\max_{\vec{\rho} \in \configs} \sum u_t(\vec{\rho}) - u_t(\vec{\rho}_t)]$. Under full information, we show that the exponentially-weighted
forecaster~\citep{NicoloGabor06:PLG} has regret bounded by $\tilde O(H(\sqrt{\numfunctions d} +
k) + \numfunctions Lw)$. When $w = 1/\sqrt{\numfunctions}$ and $k = \tilde O(\sqrt{\numfunctions})$, this results in
$\tilde O(\sqrt{\numfunctions}(H\sqrt{d}+L))$ regret. We also prove a matching lower bound.
This algorithm also preserves $(\epsilon,\delta)$-differential privacy with
regret bounded by $\tilde O (H(\sqrt{\numfunctions}d/\epsilon + k+ \delta) + \numfunctions Lw)$. Finally,
under bandit feedback, we show that a discretization-based algorithm achieves
regret at most $\tilde O(H(\sqrt{d \numfunctions(3R/w)^d} + k) + \numfunctions Lw)$. When $w =
\numfunctions^{-1/(d+2)}$ and $k = \tilde O(\numfunctions^{(d+1)/(d+2)})$, this gives a bound of $\tilde
O(\numfunctions^{(d+1)/(d+2)}(H\sqrt{d(3R)^d} + L))$, matching the dependence on $\numfunctions$ of a
lower bound by Kleinberg et al.~\cite{Kleinberg2008:MetricBandits} for (globally) Lipschitz
functions.

Online algorithm selection is generally not possible: Gupta and Roughgarden~\cite{Gupta17:PAC} give an algorithm selection problem for which
no online algorithm can achieve sub-linear regret. Therefore,
additional structure is necessary to prove guarantees, which we
characterize using dispersion.

\paragraph{Private batch optimization} We demonstrate that it is possible to
optimize over a set of dispersed functions while preserving \emph{differential
privacy}~\citep{Dwork06:Calibrating}. In this setting, the goal is to find the
parameter $\vec{\rho}$ that maximizes average utility on a set $\sample = \{u_1,
\dots, u_\numfunctions\}$ of functions $u_i : \mathcal{C} \to \R$ without divulging much information about any single
function $u_i$. Providing privacy at the granularity of
functions is suitable when each function encodes sensitive information about one
or a small group of individuals and each individual's information is used to define only a small
number of functions. For example, in the case of auction design and pricing problems, each function $u_i$ is defined by a set of buyers' bids or valuations for a set of items. If a single buyer's information is only encoded by a single function, then we preserve her privacy by not revealing sensitive information about any one function $u_i$. This will be the case, for example, if the buyers do not repeatedly return to buy the same items day after day. This is a common assumption in online auction design and pricing~\citep{Blum05:Near, Blum04:Online, Bubeck17:Online, Cesa-Bianchi15:Regret, Kleinberg03:Value, Roughgarden16:Minimizing, Dudik17:Oracle} because it means the buyers will not be strategic, aiming to trick the algorithm into setting lower prices in the future.

Differential privacy requires that an algorithm is randomized and its output
distribution is insensitive to changing a single point in the input data.
Formally, two multi-sets $\sample$ and $\sample'$ of $\numfunctions$ functions are
\emph{neighboring}, denoted $\sample \sim \sample'$, if $|\sample \Delta
\sample'| \leq 1$. A randomized algorithm $\alg$ is
\emph{$(\epsilon,\delta)$-differentially private} if, for any neighboring
multi-sets $\sample \sim \sample'$ and set $\mathcal{O}$ of outcomes,
$\prob(\alg(\sample) \in \mathcal{O}) \leq e^{\epsilon} \prob(\alg(\sample') \in
\mathcal{O}) + \delta$. In our setting, the algorithm's input is a set $\sample$
of $\numfunctions$ functions, and the output is a point $\vec{\rho} \in \configs$ that
approximately maximizes the average of those functions. We show that the
exponential mechanism~\citep{McSherry07:Mechanism} outputs $\hat{\vec{\rho}} \in
\configs$ such that with high probability $\frac{1}{\numfunctions}\sum_{i = 1}^\numfunctions
u_i(\hat{\vec{\rho}}) \geq \max_{\vec{\rho} \in \configs}\frac{1}{\numfunctions} \sum_{i =
1}^\numfunctions u_i(\vec{\rho}) - \tilde O(\frac{H}{\numfunctions}(\frac{d}{\epsilon} + k) + Lw)$ while
preserving $(\epsilon,0)$-differential privacy. We also give a matching lower
bound. Our private algorithms always preserve privacy, even when dispersion does
not hold.

\paragraph{Computational efficiency}  In our settings, the
functions have additional structure that enables us to design efficient
implementations of our algorithms: for one-dimensional problems, there is a
closed-form expression for the integral of the piecewise Lipschitz functions on
each piece and for multi-dimensional problems, the functions are
piecewise concave. We leverage tools from high-dimensional geometry
\citep{Bassily14:ERM, Vempala06:logconcave} to efficiently implement the
integration and sampling steps required by our algorithms. Our algorithms have
running time linear in the number of pieces of the utility function and
polynomial in all other parameters.

\medskip
\subsection{Dispersion in algorithm selection problems}\label{sec:intro_config}

\paragraph{Algorithm selection.}
We study algorithm selection for integer quadratic programs (IQPs) of the form
$\max_{\vec{z} \in \{\pm 1\}^n} \vec{z}^\top A \vec{z}$, where $A \in \R^{n \times n}$ for some $n$. Many classic NP-hard
problems can be formulated as IQPs, including max-cut
\citep{Goemans95:Improved}, max-2SAT \citep{Goemans95:Improved}, and correlation
clustering \citep{Charikar04:Maximizing}. Many IQP approximation algorithms are
semidefinite programming (SDP) rounding schemes; they solve the SDP relaxation
of the IQP and round the resulting vectors to binary values. We study two
families of SDP rounding techniques: $s$-linear rounding~\citep{Feige06:RPR2}
and outward rotation~\citep{Zwick99:Outward}, which include the
Goemans-Williamson algorithm \citep{Goemans95:Improved} as a special case. Due to these
algorithms' inherent randomization, finding an optimal rounding function over
$\numfunctions$ problem instances with $n$ variables amounts to optimizing the sum of
$(1/\numfunctions^{1-\alpha}, \tilde O(n\numfunctions^\alpha))$-dispersed functions for $1/2 \leq \alpha < 1$. This holds even for
adversarial (non-stochastic) instances, implying strong online learning
guarantees.

We also study greedy algorithm selection for two canonical subset selection
problems: the knapsack and maximum weight independent set (MWIS)
problems. Greedy algorithms are typically defined by a scoring rule
determining the order the algorithm adds elements to the solution
set. For example, Gupta and Roughgarden~\cite{Gupta17:PAC} introduce a parameterized knapsack algorithm that adds items in
decreasing order of $v(i)/s(i)^\rho$, where $v(i)$ and $s(i)$ are the
value and size of item $i$.
Under mild conditions --- roughly, that the items' values are drawn from distributions with bounded density functions and that each item's size is independent from its value --- we
show that the utility functions induced by $\numfunctions$ knapsack instances with $n$ items are
$(1/\numfunctions^{1-\alpha}, \tilde O(n\numfunctions^\alpha))$-dispersed for any $1/2 \leq \alpha < 1$.

\paragraph{Pricing problems and auction design}
Market designers use machine learning to design auctions and set prices \citep{Yee15:Aerosolve,He14:Practical}. In the online setting, at each time step there is a set of goods for sale and a set of consumers who place bids for those goods. The goal is to set auction parameters, such as reserve prices, that are nearly as good as the best fixed parameters in hindsight. Here, ``best'' may be defined in terms of revenue or social welfare, for example. In the offline setting, the algorithm receives a set of bidder valuations sampled from an unknown distribution and aims to find parameters that are nearly optimal in expectation (e.g., \citep{Elkind07:Designing,Cole14:Sample, Huang15:Making, Medina14:Learning, Morgenstern15:Pseudo, Roughgarden15:Ironing,Devanur16:Sample, Gonczarowski17:Efficient, Bubeck17:Online, Morgenstern16:Learning, Balcan16:Sample, Balcan18:General}). We analyze multi-item, multi-bidder second price auctions with reserves, as well as pricing problems, where the algorithm sets prices and buyers decide what to buy based on their utility functions. These classic mechanisms have been studied for decades in
both economics and computer science. We note that data-driven mechanism design problems are effectively algorithm design problems with incentive constraints: the input to a mechanism is the buyers' bids or valuations, and the output is an allocation of the goods and a description of the payments required of the buyers. For ease of exposition, we discuss algorithm and mechanism design separately.

\medskip
\subsection{Related work}\label{sec:related}

Gupta and Roughgarden~\cite{Gupta17:PAC} and Balcan et al.~\cite{Balcan17:Learning} study
algorithm selection in the distributional learning setting, where there is a distribution $\dist$ over problem
instances. A learning
algorithm receives a set $\sample$ of samples from $\dist$. Those two works provide \emph{uniform convergence guarantees}, which bound the difference between the average performance over $\sample$
of any algorithm in a class $\alg$ and its expected performance on $\dist$. It is known that regret bounds imply generalization guarantees for various online-to-batch conversion algorithms~\citep{Cesa-Bianchi02:Generalization}, but in this work, we also show that dispersion can be used to explicitly provide uniform convergence guarantees via Rademacher complexity. Beyond this connection, our work is a significant departure from these works since we give guarantees for private algorithm selection and we give no regret algorithms, whereas Gupta and Roughgarden~\cite{Gupta17:PAC} only study online MWIS algorithm
selection, proving their algorithm has small constant per-round regret.

\paragraph{Private empirical risk minimization (ERM)} The goal of private
ERM is to find the best machine learning model parameters based on private data.
Techniques include objective and output
perturbation~\citep{Chaudhuri11:Differentially}, stochastic gradient descent,
and the exponential mechanism~\citep{Bassily14:ERM}. These works focus on minimizing data-dependent convex functions, so parameters near the optimum also have high
utility, which is not the case in our settings.

\paragraph{Private algorithm configuration} Kusner et al.~\cite{Kusner15:Differentially} develop private Bayesian optimization techniques
for tuning algorithm parameters. Their methods implicitly assume that
the utility function is differentiable. Meanwhile, the class of functions we
consider have discontinuities between pieces, and it is not enough to privately
optimize on each piece, since the boundaries themselves are data-dependent.

\paragraph{Online optimization} Prior work on online
algorithm selection focuses on significantly
more restricted settings. Cohen-Addad and Kanade~\cite{Cohen-Addad17:Online} study
single-dimensional piecewise constant functions under a ``smoothed adversary,'' where the adversary chooses a
distribution per boundary from which that boundary is drawn. Thus, the
boundaries are independent. Moreover, each distribution must have bounded
density. Gupta and Roughgarden~\cite{Gupta17:PAC} study online MWIS greedy
algorithm selection under a smoothed adversary, where the adversary
chooses a distribution per vertex from which its weight is drawn. Thus, the
vertex weights are independent and again, each distribution must have bounded
density. In contrast, we allow for more correlations among the elements of each
problem instance. Our analysis also applies to the substantially more
general setting of optimizing piecewise Lipschitz functions. We show several new
applications of our techniques in algorithm selection for SDP rounding schemes,
price setting, and auction design, none of which were covered by prior work. Furthermore,
we provide differential privacy results and generalization guarantees.

Neither Cohen-Addad and Kanade~\cite{Cohen-Addad17:Online} nor Gupta and Roughgarden~\cite{Gupta17:PAC} develop a general theory of dispersion, but we can map their analysis into our setting. In essence, Cohen-Addad and Kanade~\cite{Cohen-Addad17:Online}, who provide the tighter analysis, show that if the functions the algorithm sees map from $[0,1]$ to $[0,1]$ and are $(w,1)$-dispersed, then the regret of their algorithm is bounded by $O(\sqrt{T\ln(1/w)})$. Under a smoothed adversary, the functions are $(w,1)$-dispersed for an appropriate choice of $w$. In this work, we show that using the more general notion of $(w,k)$-dispersion is essential to proving tight learning bounds for more powerful adversaries. We provide a sequence of piecewise constant functions $u_1, \dots, u_T$ mapping $[0,1]$ to $[0,1]$ that are $(1/8, \sqrt{T} + 1)$-dispersed, which means that our regret bound is $O(\sqrt{T\log (1/w)} + k) = O(\sqrt{T})$. However, these functions are not $(w,1)$-disperse for any $w \geq 2^{-T}$, so the regret bound by Cohen-Addad and Kanade~\cite{Cohen-Addad17:Online} is trivial, since $\sqrt{T\log (1/w)}$ with $w = 2^{-T}$ equals $T$. Similarly, Weed et al.~\cite{Weed16:Online} and Feng et al.~\cite{Feng17:Learning} use a notion similar to $(w,1)$-dispersion to prove learning guarantees for the specific problem of learning to bid, as do Rakhlin et al.~\cite{Rakhlin11:Online} for learning threshold functions under a smoothed adversary.

Our online bandit results are related to those of Kleinberg~\cite{Kleinberg2004:CAB} for
the ``continuum-armed bandit'' problem. They consider bandit problems where the
set of arms is the interval $[0,1]$ and each payout function is uniformly
locally Lipschitz. We relax this requirement, allowing each payout function to
be Lipschitz with a number of discontinuities. In exchange, we require that the
overall sequence of payout functions is fairly nice, in the sense that their
discontinuities do not tightly concentrate. The follow-up work on Multi-armed
Bandits in Metric Spaces~\citep{Kleinberg2008:MetricBandits} considers the
stochastic bandit problem where the space of arms is an arbitrary metric
space and the mean payoff function is Lipschitz. They introduce the zooming
algorithm, which has better regret bounds than the discretization approach
of Kleinberg~\cite{Kleinberg2004:CAB} when either the max-min covering dimension or the
(payout-dependent) zooming dimension are smaller than the covering dimension. In
contrast, we consider optimization over $\reals^d$ under the $\ell_2$ metric,
where this algorithm does not give improved regret in the worst case.

\paragraph{Auction design and pricing}
Several works
\citep{Blum05:Near, Blum04:Online, Bubeck17:Online, Cesa-Bianchi15:Regret, Kleinberg03:Value, Roughgarden16:Minimizing} present stylized online learning algorithms for revenue maximization under specific auction classes. In contrast, our online algorithms are highly general and apply
to many optimization problems beyond auction design. Dud{\'\i}k et al.~\cite{Dudik17:Oracle}
also provide online algorithms for auction design. They discretize each set of
mechanisms they consider and prove their algorithms have low regret over the
discretized set. When the bidders have
simple valuations (unit-demand and single-parameter)
minimizing regret over the discretized set amounts to minimizing regret over the
entire mechanism class. In contrast, we study bidders with fully general
valuations, as well as additive and unit-demand valuations.

A long line of work has studied \emph{generalization guarantees} for auction design and pricing problems (e.g., \citep{Elkind07:Designing,Cole14:Sample, Huang15:Making, Medina14:Learning, Morgenstern15:Pseudo, Roughgarden15:Ironing,Devanur16:Sample, Gonczarowski17:Efficient, Bubeck17:Online, Morgenstern16:Learning, Goldner16:Prior, Balcan16:Sample, Balcan18:General}). These works study the distributional setting where there is an unknown distribution over buyers' values and the goal is to use samples from this distribution to design a mechanism with high expected revenue. Generalization guarantees bound the difference between a mechanism's empirical revenue over the set of samples and expected revenue over the distribution. For example, several of these works \citep{Medina14:Learning, Morgenstern15:Pseudo, Morgenstern16:Learning, Balcan16:Sample, Balcan18:General, Medina17:Revenue, Syrgkanis17:Sample} use learning theoretic tools such as pseudo-dimension and Rademacher complexity to derive these generalization guarantees. In contrast, we study online and private mechanism design, which requires a distinct set of analysis tools beyond those used in the distributional setting.

Bubeck et al.~\cite{Bubeck17:Online} study auction design in both the online and distributional settings when there is a single item for sale. They take advantage of structure exhibited in this well-studied single-item setting, such as the precise form of the optimal single-item auction \citep{Myerson81:Optimal}. Meanwhile, our algorithms and guarantees apply to the more general problem of optimizing piecewise Lipschitz functions.

\section{Dispersion condition}\label{sec:dispersion}
In this section we formally define $(w,k)$-dispersion using the same notation as
in Section~\ref{sec:IntroSetting}. Recall that $\Pi$ is a set of instances,
$\configs \subset \reals^d$ is a parameter space, and $u$ is an abstract utility
function. Throughout this paper, we use the $\ell_2$ distance and let
$B(\vec{\rho},r) = \setc{\vec{\rho}' \in
\reals^d}{\norm{\vec{\rho}-\vec{\rho}'}_2 \leq r}$ denote a ball of radius $r$
centered at $\vec{\rho}$.

\begin{restatable}{defn}{dispersionDef}
  \label{def:dispersion}
  Let $u_1, \dots, u_\numfunctions : \configs \to [0,H]$ be a collection of functions where
  $u_i$ is piecewise Lipschitz over a partition $\partition_i$ of $\configs$. We
  say that $\partition_i$ splits a set $A$ if $A$ intersects with at least two
  sets in $\partition_i$ (see Figure~\ref{fig:splitting}). The collection of functions is
  \emph{$(w,k)$-dispersed} if every ball of radius $w$ is split by at most $k$
  of the partitions $\partition_1, \dots, \partition_\numfunctions$. More generally, the
  functions are \emph{$(w,k)$-dispersed at a maximizer} if there exists a
  point $\vec{\rho}^* \in \argmax_{\vec{\rho} \in \configs} \sum_{i = 1}^\numfunctions
  u_i(\vec{\rho})$ such that the ball $B(\vec{\rho}^*,w)$ is split by at most
  $k$ of the partitions $\partition_1, \dots, \partition_\numfunctions$.
\end{restatable}
\begin{figure}
    \centering
    \includegraphics[width=0.35\textwidth]{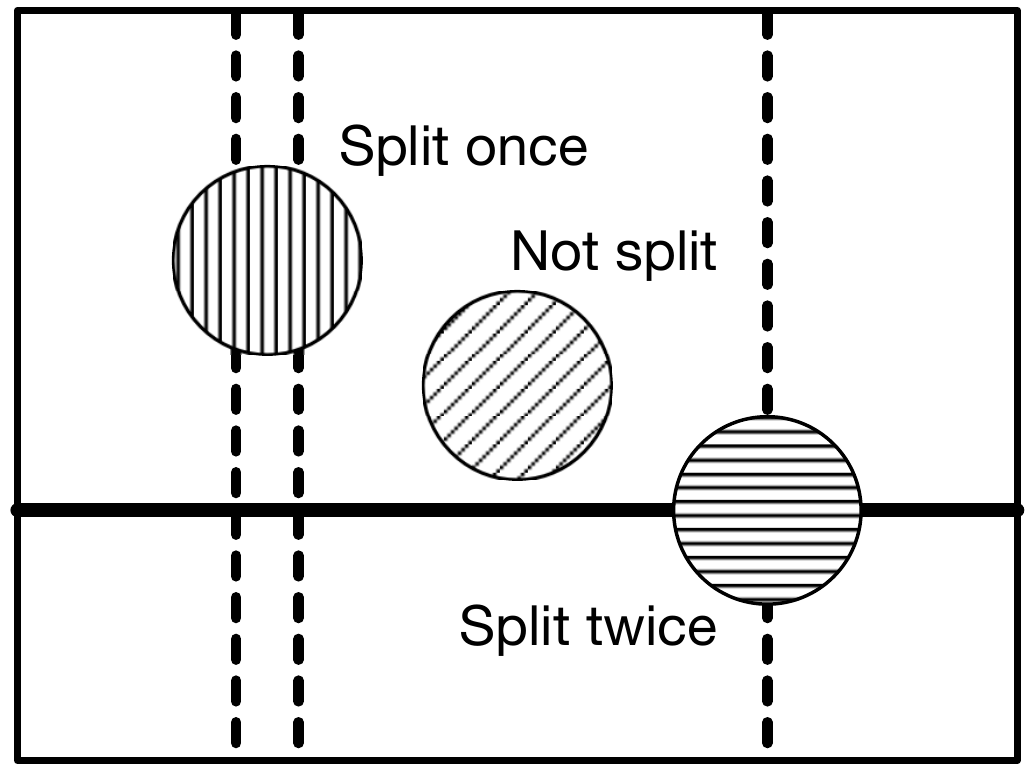}
    \caption{The dashed and solid lines correspond to two partitionings of the
    rectangle. Each of the displayed balls is either not split, split by one
    partition, or split by both.}
    \label{fig:splitting}
\end{figure}

Given $\sample = \{x_1, \dots, x_\numfunctions\} \subseteq \Pi$ and a utility function $u :
\Pi \times \configs \to [0,H]$, we equivalently say that $u$ is
\emph{$(w,k)$-dispersed for $\sample$ (at a maximizer)} if $\{u(x_1, \cdot),
\dots, u(x_\numfunctions, \cdot)\}$ is $(w,k)$-dispersed (at a maximizer).

We often show that the discontinuities of a piecewise Lipschitz function $u :
\reals \to \reals$ are random variables with \emph{$\kappa$-bounded distributions}. A
density function $f:\R \to \R$ corresponds to a
$\kappa$-bounded distribution if $\max\{f(x)\} \leq \kappa$.\footnote{For example,
for all $\mu \in \R$, $\mathcal{N}(\mu, \sigma)$ is $\frac{1}{2\pi
\sigma}$-bounded.} To
prove dispersion we will use the following probabilistic lemma, showing that
samples from $\kappa$-bounded distributions do not tightly concentrate.
\begin{restatable}{lem}{dispersionLem} \label{lem:dispersion}
  Let $\cB = \{\beta_1, \dots, \beta_r\} \subset \reals$ be a collection of
  samples where each $\beta_i$ is drawn from a $\kappa$-bounded distribution with density function $p_i$. For any $\zeta \geq 0$, the following statements hold
  with probability at least $1-\zeta$:
  \begin{enumerate}
    \item If the $\beta_i$ are independent, then every interval of width $w$
    contains at most $k = O(r w \kappa + \sqrt{r \log(1/\zeta)})$ samples. In
    particular, for any $\alpha \geq 1/2$ we can take $w = 1/(\kappa
    r^{1-\alpha})$ and $k = O(r^\alpha \sqrt{\log(1/\zeta)})$.
    \item If the samples can be partitioned into $P$ buckets $\cB_1, \dots,
    \cB_P$ such that each $\cB_i$ contains independent samples and $|\cB_i| \leq
    M$, then every interval of width $w$ contains at most $k = O(P M w \kappa +
    \sqrt{M \log(P/\zeta)}$. In particular, for any $\alpha \geq 1/2$ we can
    take $w = 1/(\kappa M^{1-\alpha})$ and $k = O(P M^\alpha \sqrt{
    \log(P/\zeta)})$.
  \end{enumerate}

\end{restatable}
\begin{proof}[Proof sketch]
  If the $\beta_i$ are independent, the expected number of samples in any
  interval of width $w$ is at most $r \kappa w$.
  Since the VC-dimension of intervals is 2, it follows that with probability at
  least $1-\zeta$, no interval contains more than $r \kappa w + O(\sqrt{r
  \log(1/\zeta)})$ samples.

  The second claim follows by applying this counting argument to each of the
  buckets $\cB_i$ with failure probability $\zeta' = \zeta/P$ and taking the
  union bound over all buckets. With probability at least $1-\zeta$, every
  interval of width $w$ contains at most $M \kappa w + O(\sqrt{M
  \log(P/\zeta)})$ samples from each bucket, and at most $k = PM \kappa w + O(P
  \sqrt{M\log(P/\zeta)})$ samples in total from all $P$ buckets.
\end{proof}

Lemma~\ref{lem:dispersion} allows us to provide dispersion
guarantees for ``smoothed adversaries'' in online learning. Under this type of
adversary, the discontinuity locations for each function $u_i$ are random
variables, due to the smoothness of the adversary.
In our algorithm selection applications, the randomness of discontinuities may
be a byproduct of the randomness in the algorithm's inputs. For example, in the
case of knapsack algorithm configuration, the item values and sizes may be drawn
from distributions chosen by the adversary. This induces randomness in the
discontinuity locations of the algorithm's cost function. We can thus apply
Lemma~\ref{lem:dispersion} to guarantee dispersion.

We also use Lemma~\ref{lem:dispersion} to guarantee dispersion even when the
adversary is not smoothed. Surprisingly, we show that dispersion holds for IQP
algorithm configuration without \emph{any} assumptions on the input instances.
In this case, we exploit the fact that the algorithms are themselves randomized.
This randomness implies that the discontinuities of the algorithm's cost
function are random variables, and thus Lemma~\ref{lem:dispersion} implies
dispersion.

\section{Online optimization}\label{sec:online}
In this setting, a sequence of functions $u_1, \dots, u_\numfunctions$ arrive one-by-one. At time $t$, the learning
algorithm chooses a vector $\vec{\rho}_t$ and then either observes the
function $u_t(\cdot)$ in the full information setting or the value $u_t(\vec{\rho}_t)$ in the bandit setting. The goal is to minimize expected regret:
$\E[\max_{\vec{\rho} \in \configs} \sum_{t=1}^\numfunctions (u_t(\vec{\rho}) - u_t(\vec{\rho}_t))]$.
In our applications, the functions $u_1$, \dots, $u_\numfunctions$ are random, either due to internal randomization in the algorithms we are
configuring or from assumptions on the adversary\footnote{As we describe in Section~\ref{sec:related}, prior research~\citep{Gupta16:PAC, Cohen-Addad17:Online} also makes assumptions on the adversary. For example, \citet{Cohen-Addad17:Online} focus on adversaries that choose distributions with bounded densities from
which the discontinuities of $u_t$ are drawn. In
Lemma~\ref{lem:smooth2disperse} of Appendix~\ref{app:online}, we show that their
smoothness assumption implies dispersion with
high probability.}. We show that the functions are $(w,k)$-dispersed with
probability $1-\zeta$ over the choice of $u_1$, \dots, $u_\numfunctions$. The
following regret bounds hold in expectation with an
additional term of $H\numfunctions\zeta$ bounding the effect of the rare event where the functions are not dispersed.

\medskip \noindent\textbf{Full information.} The
\emph{exponentially-weighted forecaster} algorithm samples the vectors
$\vec{\rho}_t$ from the distribution $p_t(\vec{\rho}) \propto \exp(\lambda
\sum_{s = 1}^{t-1} u_s(\vec{\rho}))$.  We prove the following regret bound. The full proof is in Appendix~\ref{app:online}.

\begin{restatable}{thm}{OnedOnline}\label{thm:1_d_online}
  Let $u_1, \dots, u_\numfunctions : \configs \to [0,H] $ be any sequence of piecewise
  $L$-Lipschitz functions that are $(w,k)$-dispersed at the maximizer
  $\vec{\rho}^*$. Suppose $\configs \subset \reals^d$ is contained in
  a ball of radius $R$ and
  $B(\vec{\rho^*},w) \subset \configs$. The exponentially weighted forecaster
  with $\lambda = \sqrt{d\ln(R/w)/\numfunctions}/H$ has expected regret bounded by
  \[O\left(H\left(\sqrt{\numfunctions d\log\frac{R}{w}} + k\right) + \numfunctions Lw\right).\]

  For all rounds $t \in [\numfunctions]$, suppose $\sum_{s=1}^t u_s$ is
  piecewise Lipschitz over at most $K$ pieces. When $d=1$ and $\exp(\sum_{s=1}^t
  u_s)$ can be integrated in constant time on each of its pieces, the running
  time is $O(\numfunctions K)$. When $d > 1$ and $\sum_{s=1}^t u_s$ is piecewise
  concave over convex pieces, we provide an efficient approximate
  implementation. For approximation parameters $\eta = \zeta = 1/\sqrt{T}$ and
  $\lambda = \sqrt{d\ln(R/w)/\numfunctions}/H$, this algorithm has the same
  regret bound as the exact algorithm and runs in time $\tilde
  O(\numfunctions(K\cdot \operatorname{poly}(d, 1/\eta) +
  \operatorname{poly}(d,L, 1/\eta)).$
\end{restatable}

\begin{proof}[Proof sketch]
Let $U_t$ be the function $\sum_{i = 1}^{t-1} u_i(\cdot)$ and let $W_t = \int_\configs \exp(\lambda U_t(\vec{\rho})) \, d\vec{\rho}$.
  We use $(w,k)$-dispersion to lower bound $W_{\numfunctions+1}/W_1$ in terms
  of the optimal parameter's total payout. Combining this with a standard
  upper bound on $W_{\numfunctions+1}/W_1$ in terms of the learner's expected payout gives the
  regret bound.
  To lower bound $W_{\numfunctions+1}/W_1$,
let $\vec{\rho}^*$ be the optimal parameter and let $\OPT = U_{T+1}(\vec{\rho}^*)$. Also, let $\mathcal{B}^*$ be the ball of radius $w$ around $\vec{\rho}^*$. From $(w,k)$-dispersion, we know that for all $\vec{\rho} \in \mathcal{B}^*$, $U_{T+1}(\vec{\rho}) \geq \OPT - Hk - LTw$. Therefore, \begin{align*}
W_{T+1} &= \int_\configs \exp(\lambda U_{T+1}(\vec{\rho})) \, d\vec{\rho} \geq \int_{\mathcal{B}^*} \exp(\lambda U_{T+1}(\vec{\rho})) \, d\vec{\rho}\\
&\geq \int_{\mathcal{B}^*} \exp(\lambda (\OPT - Hk - LTw))d\vec{\rho}\\
&\geq \vol(B(\vec{\rho}^*, w))\exp(\lambda (\OPT - Hk - LTw)).
\end{align*}
Moreover, $W_1 = \int_{\configs} \exp(\lambda U_1(\vec{\rho})) \, d\vec{\rho} \leq \vol(B(\vec{0}, R))$. Therefore, \[\frac{W_{T+1}}{W_1} \geq \frac{\vol(B(\vec{\rho}^*, w))}{\vol(B(\vec{0}, R))} \exp(\lambda (\OPT - Hk - LTw)).\] The volume ratio is equal to
  $(w/R)^d$, since the volume of a ball of radius $r$ in $\R^d$ is proportional
  to $r^d$. Therefore, $W_{T+1}/W_1 \geq \left(w/R\right)^d \exp(\lambda (\OPT - Hk - LTw)).$ Combining the upper and lower bounds on $\frac{W_{T+1}}{W_1}$ gives the result.

  Our efficient algorithm (Algorithm~\ref{alg:multi_d_online} of Appendix~\ref{app:online}) approximately samples from $p_t$. Let
  $\configs_1, \dots, \configs_K$ be the partition of $\configs$ over which $\sum u_t(\cdot)$
  is piecewise concave.  Our algorithm picks
  $\configs_I$ with probability approximately proportional to $\int_{\configs_I} p_t$ \citep{Vempala06:logconcave} and outputs a sample from the conditional distribution of $p_t$ on
  $\configs_I$ \citep{Bassily14:ERM}. Crucially, we prove that the algorithm's output distribution is close to $p_t$, so every event concerning the outcome of the approximate algorithm
  occurs with about the same probability as it does under
  $p_t$.
\end{proof}

The requirement that $B(\vec{\rho^*},w) \subset \configs$ is for convenience. In
Lemma~\ref{lem:interiortransformation} of Appendix~\ref{app:online} we show how
to transform the problem to satisfy this. Setting $\lambda = \sqrt{d/\numfunctions}/H$,
which does not require knowledge of $w$, has regret $O(H(\sqrt{\numfunctions d}\log(R/w) + k) +
\numfunctions Lw).$ Under alternative settings of $\lambda$, we show that our algorithms are
$(\epsilon,\delta)$-differentially private with regret bounds of
$\tilde{O}(H\sqrt{\numfunctions}/\epsilon + Hk + L\numfunctions w)$ in the single-dimensional setting and
$\tilde{O}(H\sqrt{\numfunctions} d/\epsilon + H(k + \delta) + L\numfunctions w)$ in the $d$-dimensional
setting (see Theorems~\ref{thm:online_1d_private} and
\ref{thm:online_multiD_private} in Appendix~\ref{app:online}).

Next, we prove a matching lower bound. We warm up with a proof for the single-dimensional case in Appendix~\ref{app:online_lb_single} and then generalize that intuition to the multi-dimensional case in Appendix~\ref{app:online_lb_multi}.

\begin{restatable}{theorem}{onlineLB}
Suppose $\numfunctions \geq d$. For any algorithm, there are piecewise constant functions $u_1, \dots, u_\numfunctions$ mapping $[0,1]^d$ to $[0,1]$ such that if $D = \{(w,k) : \{u_1, \dots, u_\numfunctions\}$ is $(w,k)$-dispersed at the maximizer$\},$ then \[\max_{\vec{\rho} \in [0,1]^d}\E\left[\sum_{t = 1}^\numfunctions u_t\left(\vec{\rho}\right) - u_t\left(\vec{\rho}_t\right)\right] = \Omega\left(\inf_{(w,k) \in D}\left\{ \sqrt{\numfunctions d\log \frac{1}{w}} + k \right\}\right),\] where the expectation is over the random choices $\vec{\rho}_1, \dots, \vec{\rho}_{\numfunctions}$ of the adversary.
\end{restatable}

\begin{proof}[Proof sketch]
For each dimension, the adversary plays a sequence of axis-aligned halfspaces with thresholds that divide the set of optimal parameters in two. The adversary plays each halfspace $\Theta(\frac{\numfunctions}{d})$ times, randomly switching which side of the halfspace has a positive label, thus forcing regret of at least $\frac{\sqrt{Td}}{64}$. We prove that the resulting set of optimal parameters is contained in a hypercube of side length $\frac{1}{2}$. The adversary then plays $\sqrt{\numfunctions} +d$ copies of the indicator function of a ball of radius $2^{-T}$ at the center of this cube. This ensures the functions are not $(w,0)$-dispersed at the maximizer for any $w \geq 2^{-\numfunctions}$, and thus prior regret analyses~\citep{Cohen-Addad17:Online}
give a trivial bound of $\numfunctions$. In order to prove the theorem, we need to show that $\frac{\sqrt{Td}}{64} = \Omega\left(\inf_{(w,k) \in D}\left\{ \sqrt{Td\log \frac{1}{w}} + k \right\}\right)$. Therefore, we need to show that the set of functions played by the adversary is $(w,k)$-dispersed at the maximizer $\vec{\rho}^*$ for $w = \Theta(1)$ and $k = O\left(\sqrt{\numfunctions d}\right).$ The reason this is true is that the only functions with discontinuities in the ball $\left\{\vec{\rho} : ||\vec{\rho}^* - \vec{\rho}|| \leq \frac{1}{8}\right\}$ are the final $\sqrt{T} + d$ functions played by the adversary. Thus, the theorem statement holds.
\end{proof}

\medskip
\noindent\textbf{Bandit feedback.}
We now study online optimization under bandit feedback.
\begin{thm}
  \label{thm:banditRegret}
  Let $u_1, \dots, u_\numfunctions : \configs \to [0,H]$ be any sequence of piecewise
  $L$-Lipschitz functions that are $(w,k)$-dispersed at the maximizer
  $\vec{\rho}^*$. Moreover, suppose that $\configs \subset \reals^d$ is
  contained in a ball of radius $R$ and that $B(\vec{\rho^*},w) \subset
  \configs$. There is a bandit-feedback online optimization algorithm with
  regret \[O\left(H \sqrt{\numfunctions d\left(\frac{3R}{w}\right)^d\log\frac{R}{w}} + \numfunctions Lw + Hk\right).\] The per-round
  running time is $O((3R/w)^d)$.
\end{thm}
\begin{proof}
  Let $\vec{\rho}_1$, \dots, $\vec{\rho}_M$ be a $w$-net for $\configs$. The
  main insight is that $(w,k)$-dispersion implies that the difference in utility
  between the best point in hindsight from the net and the best point in
  hindsight from $\configs$ is at most $Hk + \numfunctions Lw$. Therefore, we only need to
  compete with the best point in the net. We use the Exp3
  algorithm~\citep{Auer2003:exp3} to choose parameters $\hat{\vec{\rho}}_1$,
  \dots, $\hat{\vec{\rho}}_\numfunctions$ by playing the bandit with $M$ arms, where on
  round $t$ arm $i$ has payout $u_t(\vec{\rho}_i)$. The expected regret of Exp3
  is $\tilde O(H \sqrt{\numfunctions M \log M})$ relative to our net. In
  Lemma~\ref{lem:netSize} of Appendix~\ref{app:online}, we show $M \leq
  (3R/w)^d$, so the overall regret is $\tilde O(H \sqrt{\numfunctions d(3R/w)^d\log(R/w)} +
  \numfunctions Lw + Hk)$ with respect to $\configs$.
\end{proof}
If $w = \numfunctions^{\frac{d+1}{d+2}-1} = \frac{1}{T^{1/(d+2)}}$ and $k =
\tilde O\left(\numfunctions^{\frac{d+1}{d+2}}\right)$, Theorem~\ref{thm:banditRegret} gives
the optimal exponent on $\numfunctions$. Specifically, the regret is $\tilde
O\left(\numfunctions^{(d+1)/(d+2)}\left(H \sqrt{d (3R)^d} + L\right)\right)$, and no
algorithm can have regret $O\left(\numfunctions^\gamma\right)$ for $\gamma < (d+1)/(d+2)$
for the special case of (globally) Lispchitz functions
\citep{Kleinberg2008:MetricBandits}.

\section{Differentially private optimization}\label{sec:algorithm}
We show that the exponential mechanism, which is $(\epsilon, 0)$-differentially
private, has high utility when optimizing the mean of dispersed functions. In
this setting, the algorithm is given a collection of functions $u_1, \dots, u_\numfunctions :
\configs \to [0,H]$, each of which depends on some sensitive
information. In cases where each function $u_i$ encodes sensitive information about one or a small group of individuals and each individual is present in a small number of functions, we can give meaningful privacy guarantees by providing differential privacy for each function in the collection.
We say that two sets of $\numfunctions$ functions are
neighboring if they differ on at most one function. Recall that the exponential
mechanism outputs a sample from the distribution with density proportional to
$\expmf^\epsilon(\vec{\rho}) = \exp\bigl(\frac{\epsilon}{2\Delta\numfunctions} \sum_{i=1}^\numfunctions
u_i(\vec{\rho})\bigr)$, where $\Delta$ is the sensitivity of the average
utility. Since the functions $u_i$ are bounded, the sensitivity of $\frac{1}{\numfunctions} \sum_{i=1}^{\numfunctions} u_i$ satisfies $\Delta \leq
H/\numfunctions$. The following theorem states our utility guarantee. The full proof is in
Appendix~\ref{app:algorithm}.

\begin{restatable}{theorem}{cexpmutility}
  \label{thm:cexpmutility}
  Let $u_1, \dots, u_\numfunctions : \configs \to [0,H]$ be piecewise $L$-Lipschitz and
  $(w,k)$-dispersed at the maximizer $\vec{\rho}^*$, and suppose that $\configs
  \subset \reals^d$ is convex, contained in a ball of radius $R$, and
  $B(\vec{\rho}^*,w) \subset \configs$. For any $\epsilon > 0$, with probability
  at least $1-\zeta$, the output $\hat{\vec{\rho}}$ of the exponential mechanism
  satisfies \[\frac{1}{\numfunctions}\sum_{i=1}^\numfunctions u_i\left(\hat{\vec{\rho}}\right) \geq
  \frac{1}{\numfunctions}\sum_{i=1}^\numfunctions u_i\left(\vec{\rho}^*\right) - O\left(\frac{H}{\numfunctions\epsilon} \left(d \log
  \frac{R}{w} + \log \frac{1}{\zeta} \right) + Lw + \frac{Hk}{\numfunctions}\right).\]

  When $d=1$, this algorithm is efficient, provided $\expmf^\epsilon$ can be
  efficiently integrated on each piece of $\sum_i u_i$. For $d>1$ we also
  provide an efficient approximate sampling algorithm when $\sum_{i} u_i$ is
  piecewise concave defined on $K$ convex pieces. This algorithm preserves
  $(\epsilon, \delta)$-differential privacy for $\epsilon > 0$, $\delta > 0$
  with the same utility guarantee (with $\zeta = \delta$). The running time of
  this algorithm is $\tilde O(K\cdot\operatorname{poly}(d, 1/\epsilon) +
  \operatorname{poly}(d, L, 1/\epsilon))$.
\end{restatable}

\begin{proof}[Proof sketch]
 The exponential mechanism can fail to output a good parameter if there are
drastically more bad parameters than good. The key insight is that due to
dispersion, the set of good parameters is not too small.
In particular, we
 know that every $\vec{\rho} \in B(\vec{\rho^*}, w)$ has $\frac{1}{\numfunctions}\sum_{i}
  u_i(\vec{\rho}) \geq \frac{1}{\numfunctions}\sum_{i} u_i(\vec{\rho}^*) - \frac{Hk}{\numfunctions} -
  Lw$ because at most $k$ of the functions $u_i$ for have discontinuities in
  $B(\vec{\rho^*}, w)$ and the rest are $L$-Lipschitz.

In a bit more detail, for a constant $c$ fixed later on, the
  probability that a sample from $\expmu$ lands in $E = \{\vec{\rho} \,:\,
  \frac{1}{\numfunctions}\sum_{i}
  u_i(\vec{\rho}) \leq c\}$ is $F / Z$, where $F = \int_E \expmf$ and $Z =
  \int_\configs \expmf$.
 We know that $F \leq \exp\left(\frac{\numfunctions \epsilon c}{2H}\right) \vol(E) \leq
  \exp\left(\frac{\numfunctions \epsilon c}{2H}\right) \vol\bigl(B(0,R)\bigr),$ where
  the second inequality follows from the fact that a ball of radius $R$ contains
  the entire space $\configs$.
To lower bound $Z$, we use the fact that at most $k$ of the functions $u_1, \dots, u_{\numfunctions}$ have discontinuities in the ball $B(\vec{\rho^*},w)$ and the rest of the functions are $L$-Lipschitz. It follows that for any $\vec{\rho} \in
  B(\vec{\rho^*},w)$, we have $\frac{1}{\numfunctions}\sum_{i}
  u_i(\vec{\rho}) \geq \frac{1}{\numfunctions}\sum_{i}
  u_i(\vec{\rho}^*) - \frac{Hk}{|\sample|} - Lw$. This is because each of the $k$
  functions with boundaries can affect the average utility by at most
  $H/|\numfunctions|$ and otherwise $\frac{1}{\numfunctions}\sum_{i}
  u_i(\cdot)$ is $L$-Lipschitz. Since
  $B(\vec{\rho^*},w) \subset \configs$, this gives $Z \geq
  \exp\bigl(\frac{\numfunctions \epsilon}{2H}(\frac{1}{\numfunctions}\sum_{i}
  u_i(\vec{\rho}^*)) -
  \frac{Hk}{\numfunctions} - Lw)\bigr) \vol\bigl(B(\vec{\rho^*},w)\bigr)$.

  Putting the bounds together, we have that $F/Z \leq
  \exp\bigl(\frac{\numfunctions\epsilon}{2H}(c - \frac{1}{\numfunctions}\sum_{i}
  u_i(\vec{\rho}^*) +
  \frac{Hk}{\numfunctions} + Lw\bigr) \cdot
  \frac{\vol(B(0,R))}{\vol(B(\vec{\rho^*},w))}$. The volume ratio is equal to
  $(R/w)^d$, since the volume of a ball of radius $r$ in $\R^d$ is proportional
  to $r^d$. Setting this bound to $\zeta$ and solving for $c$ gives the result.

  Our efficient implementation
  (Algorithm~\ref{alg:efficient} in Appendix~\ref{app:algorithm}) relies on the
  same tools as our approximate implementation of the exponentially weighted
  forecaster. The main step is proving the distribution of $\hat{\vec{\rho}}$ is
  close to the distribution with density $\expmf$.
\end{proof}

In Appendix~\ref{app:discretized}, we also give a discretization-based
computationally inefficient algorithm in $d$ dimensions that satisfies
$(\epsilon,0)$-differential privacy.

We can tune the value of $w$ to make the dependence on $L$ logarithmic: if $\numfunctions
\geq \frac{2Hd}{w\epsilon L}$, then with probability $1-\zeta$,
$\frac{1}{\numfunctions}\sum_{i}u_i(\hat{\vec{\rho}}) \geq
\frac{1}{\numfunctions}\sum_{i}u_i(\vec{\rho}^*) - O\left(\frac{Hd}{\numfunctions\epsilon}\log
\frac{L\epsilon R\numfunctions}{Hd} + \frac{Hk}{\numfunctions} + \frac{H}{\numfunctions\epsilon}\log\frac{1}{\zeta}\right)$ (Corollary~\ref{cor:expmutilityoptimized} in
Appendix~\ref{app:algorithm}).

Finally, we provide a matching lower bound. See Appendix~\ref{app:algorithm} for
the full proof. When $d = 1$, we can instantiate these lower bounds using MWIS
instances.

\begin{restatable}{thm}{thmPrivacyLowerBound}\label{thm:lower_bound}
  For every dimension $d \geq 1$, privacy parameter $\epsilon > 0$, failure
  probability $\zeta > 0$, $T \geq \frac{d}{\epsilon}(\frac{\ln 2}{2} - \ln
  \frac{1}{\zeta}))$ and $\epsilon$-differentially private optimization
  algorithm $\cA$ that takes as input a collection of $T$ piecewise constant
  functions mapping $B(0,1) \subset \reals^d$ to $[0,1]$ and outputs an
  approximate maximizer, there exists a multiset $\sample$ of such functions so
  that with probability at least $1-\zeta$, the output $\hat{\vec{\rho}}$ of
  $\cA(\sample)$ satisfies
  \[
  \frac{1}{\numfunctions}\sum_{u \in \sample} u(\hat{\vec{\rho}}) \leq
  \max_{\vec{\rho} \in B(0,1)} \frac{1}{\numfunctions}\sum_{u \in \sample} u(\vec{\rho})
  - \Omega\biggl(
  \inf_{(w,k)}\frac{d}{\numfunctions\epsilon}\biggl(\ln \frac{1}{w} - \ln \frac{1}{\zeta}\biggr) + \frac{k}{\numfunctions}
  \biggr),
  \]
  where the infimum is taken over all $(w,k)$-dispersion at the maximizer
  parameters satisfied by $\sample$.
\end{restatable}

\begin{proof}[Proof sketch]
  We construct $M = 2^d$ multi-sets of functions $\sample_1, \dots, \sample_M$,
  each with $\numfunctions$ piecewise constant functions. For every pair
  $\sample_i$ and $\sample_j$, $|\sample_i \Delta \sample_j|$ is small but the
  set $I_{\sample_i}$ of parameters maximizing $\sum_{u \in \sample_i}
  u(\vec{\rho})$ is disjoint from $I_{\sample_j}$. Therefore, for every pair
  $\sample_i$ and $\sample_j$, the distributions $\alg(\sample_i)$ and
  $\alg(\sample_j)$ are similar, and since $I_{\sample_1}, \dots, I_{\sample_t}$
  are disjoint, this means that for some $\sample_i$, with high probability, the
  output of $\alg(\sample_i) \not\in I_{\sample_i}$. The key challenge is
  constructing the sets $\sample_i$ so that the suboptimality of any point not
  in $I_{\sample_i}$ is $\frac{d}{\numfunctions\epsilon} \log \frac{R}{w} +
  \frac{k}{\numfunctions}$, where $w$ and $k$ are dispersion parameters for
  $\sample_i$. We construct $\sample_i$ so that this suboptimality is
  $\Theta(\frac{d}{\numfunctions\epsilon})$, which gives the desired result if
  $w = \Theta(R)$ and $k = \Theta(\frac{d}{\epsilon})$. To achieve these
  conditions, we carefully fill each $\sample_i$ with indicator functions of
  balls centered packed in the unit ball $B(0,1)$.
\end{proof}

\section{Dispersion in application-specific algorithm selection}\label{sec:applications}
We now analyze dispersion for a range of algorithm configuration problems. In the private setting, the algorithm receives
samples $\sample \sim\dist^\numfunctions$, where $\dist$ is an
arbitrary distribution over problem instances $\Pi$. The goal is to privately find a value
$\hat{\vec{\rho}}$ that nearly maximizes $\sum_{x \in \sample} u(x, \vec{\rho})$. In our
applications, prior work~\citep{Morgenstern16:Learning, Gupta17:PAC,
Balcan17:Learning} shows that $\hat{\vec{\rho}}$ nearly maximizes
$\E_{x\sim \dist}[u(x, \vec{\rho})]$. In the online
setting, the goal is to find a value $\vec{\rho}$ that is nearly optimal in
hindsight over a stream $x_1, \dots, x_\numfunctions$ of instances, or equivalently, over a
stream $u_1 = u(x_1, \cdot), \dots, u_\numfunctions = u(x_\numfunctions, \cdot)$ of functions. Each $x_t$ is drawn from a distribution $\dist^{(t)}$, which may be
adversarial. Thus in both settings, $\{x_1, \dots, x_\numfunctions\} \sim \dist^{(1)} \times \cdots
\times \dist^{(\numfunctions)}$, but in the
private setting, $\dist^{(1)}= \cdots = \dist^{(\numfunctions)}$.

\medskip
\noindent\textbf{Greedy algorithms.}
We study greedy algorithm configuration for two important problems: the maximum weight independent set (MWIS) and knapsack problems.
In MWIS, there is a graph
and a weight $w\left(v\right) \in \R_{\geq 0}$ for each vertex $v$. The goal is to
find a set of non-adjacent vertices with maximum
weight. The classic greedy algorithm repeatedly adds a vertex $v$ which maximizes $w\left(v\right)/\left(1 + \deg\left(v\right)\right)$ to the independent set and deletes $v$ and its neighbors from the graph.
Gupta and Roughgarden~\cite{Gupta17:PAC} propose
the greedy heuristic $w\left(v\right)/\left(1 + \deg\left(v\right)\right)^{\rho}$ where
$\rho \in \configs = [0, B]$ for some $B \in \R$. When $\rho = 1$, the approximation ratio is $1/D$, where $D$ is the graph's maximum degree~\citep{Sakai03:Note}. We represent a graph as a tuple $\left(\vec{w}, \vec{e}\right) \in \R^n \times \left\{0,1\right\}^{{n \choose 2}}$, ordering the vertices
$v_1, \dots, v_n$ in a fixed but arbitrary way.
The function $u\left(\vec{w}, \vec{e}, \cdot\right)$ maps a parameter $\rho$ to the weight of the vertices in the set returned by the algorithm parameterized by $\rho$.

\begin{restatable}{thm}{MWIS}\label{thm:MWIS_upper}
  Suppose all vertex weights are in $(0,1]$ and for each $\dist^{(i)}$, every pair of vertex weights has a $\kappa$-bounded joint
  distribution. For any $\vec{w}$ and $\vec{e}$, $u\left(\vec{w}, \vec{e}, \cdot\right)$ is piecewise 0-Lipschitz and for any $\alpha \geq 1/2$,
  with probability $1-\zeta$ over $\sample \sim \bigtimes_{i = 1}^\numfunctions \dist^{(i)}$, $u$ is \[\left(\frac{1}{\numfunctions^{1-\alpha}\kappa \ln n}, O\left(n^4
  \numfunctions^\alpha \sqrt{\ln \frac{n}{\zeta}}\right)\right)\text{-dispersed}\] with respect to $\sample$.
\end{restatable}

\begin{proof}[Proof sketch]
  The utility $u\left(\vec{w}^{\left(t\right)}, \vec{e}^{\left(t\right)},\rho\right)$ has a discontinuity when
  the ordering of two vertices under the greedy score swaps. Thus, the
  discontinuities have the form \[\frac{\ln\left(w_i^{\left(t\right)}\right) - \ln\left(w_j^{\left(t\right)}\right)}{\ln\left(d_1\right) -
  \ln\left(d_2\right)}\] for all $t \in [\numfunctions]$ and $i,j,d_1, d_2 \in [n]$, where
  $w_j^{\left(t\right)}$ is the weight of the $j^{th}$ vertex of $\left(\vec{w}^{\left(t\right)},
  \vec{e}^{\left(t\right)}\right)$~\citep{Gupta16:PAC}. We show that when pairs of vertex weights
  have $\kappa$-bounded joint distributions, then the discontinuities each have
  $\left(\kappa \ln n\right)$-bounded distributions. Let
  $\cB_{i,j,d_1,d_2}$ be the set of discontinuities contributed by vertices $i$
  and $j$ with degrees $d_1$ and $d_2$ across all instances in $\sample$.  The
  buckets $\cB_{i,j,d_1,d_2}$ partition the discontinuities into $n^4$ sets of
  independent random variables. Therefore, applying
  Lemma~\ref{lem:dispersion} with $P = n^4$ and $M = \numfunctions$ proves the
  claim.
\end{proof}

In Appendix~\ref{app:applications}, we prove Theorem~\ref{thm:MWIS_upper} and demonstrate that it implies strong optimization guarantees. The analysis for the knapsack problem is similar (see Appendix~\ref{app:knapsack}).

\medskip
\noindent\textbf{Integer quadratic programming (IQP) algorithms.}
We now apply our dispersion analysis to two
popular IQP approximation algorithms: $s$-linear~\citep{Feige06:RPR2} and
outward rotation rounding algorithms~\citep{Zwick99:Outward}. The goal is to maximize a function $\sum_{i,j \in [n]} a_{ij} z_i
z_j$ over $\vec{z} \in \left\{\pm1\right\}^n$, where the matrix $A = \left(a_{ij}\right)$ has
non-negative diagonal entries.
Both algorithms are generalizations of the Goemans-Williamson (GW) max-cut
algorithm~\citep{Goemans95:Improved}. They first solve the SDP relaxation
$\sum_{i,j \in [n]} a_{ij} \langle \vec{u}_i, \vec{u}_j \rangle$ subject to the
constraint that $\norm{\vec{u}_i} = 1$ for $i \in [n]$ and then round the vectors $\vec{u}_i$ to $\left\{\pm 1\right\}$. Under $s$-linear rounding, the algorithm
samples a standard Gaussian $\vec{Z} \sim \mathcal{N}_n$ and sets $z_i = 1$ with
probability $1/2 + \phi_s\left(\langle \vec{u}_i, \vec{Z}
\rangle\right)/2$ and $-1$ otherwise, where $\phi_s\left(y\right) = -\mathbbm{1}_{y < -s} +
\frac{y}{s} \cdot \mathbbm{1}_{-s \leq y \leq s} +  \mathbbm{1}_{y >  s}$ and
$s$ is a parameter. The outward rotation algorithm first maps
each $\vec{u}_i$ to $\vec{u}_i' \in \reals^{2n}$ by $\vec{u}_i' = [\cos\left(\gamma\right)
\vec{u}_i \,;\, \sin\left(\gamma\right) \vec{e}_i]$ and sets $z_i = \sign\left(\langle
\vec{u}_i', \vec{Z}\rangle\right)$, where $\vec{e}_i$ is the $i^{\rm th}$ standard
basis vector, $\vec{Z} \in \reals^{2n}$ is a standard Gaussian, and $\gamma \in
[0, \pi/2]$ is a parameter. Feige and Langberg~\cite{Feige06:RPR2} and Zwick~\cite{Zwick99:Outward}
prove that these rounding functions provide a better worst-case approximation
ratio on graphs with ``light'' max-cuts, where the max-cut does not constitute a
large fraction of the edges.

Our utility $u$ maps the algorithm parameter (either $s$ or $\gamma$)
to the objective value obtained. We exploit the randomness of these algorithms
to guarantee dispersion. To facilitate this analysis, we imagine that the
Gaussians $\vec{Z}$ are sampled ahead of time and included as part of the
problem instance. For $s$-linear rounding, we write the utility as
$\uslin(A,\vec{Z},s) = \sum_{i = 1}^n a_i^2 + \sum_{i \not= j} a_{ij} \phi_s(v_i) \phi_s(v_j)$, where $v_i =
\langle \vec{u_i}, \vec{Z} \rangle$. For outward rotations,
$\uowr(A,\vec{Z},\gamma) = \sum_{i,j} a_{ij} \sign(v'_i) \sign(v'_j)$, where
$v'_i = \langle \vec{u}_i', \vec{Z}
\rangle$.

First, we prove a dispersion guarantee for $\uowr$. The full proof is in
Appendix~\ref{app:applications}, where we also demonstrate the theorem's implications for
our optimization settings (Theorems~\ref{thm:owr_DP}, \ref{thm:owr_full_info},
\ref{thm:owr_full_info_DP}, and \ref{thm:owr_bandit}).

\begin{restatable}{thm}{owrDispersed}\label{thm:owr_dispersed}
  For any matrix $A$ and vector $\vec{Z}$, $\uowr\left(A, \vec{Z}, \cdot\right)$ is piecewise $0$-Lipschitz. With probability
  $1-\zeta$ over $\vec{Z}^{\left(1\right)}, \dots, \vec{Z}^{\left(\numfunctions\right)} \sim
  \mathcal{N}_{2n}$, for any $\qp^{\left(1\right)}, \dots, \qp^{\left(\numfunctions\right)} \in \R^{n \times n}$ and any
  $\alpha \geq 1/2$, $\uowr$ is \[\left(\numfunctions^{\alpha -1}, O\left(n
  \numfunctions^{\alpha}\sqrt{\log\frac{n}{\zeta}}\right)\right)\text{-dispersed}\] with
  respect to $\sample = \left\{\left(\qp^{\left(t\right)}, \vec{Z}^{\left(t\right)}\right)\right\}_{t = 1}^\numfunctions$.
\end{restatable}
\begin{proof}[Proof sketch]
  The discontinuities of $u_{\text{owr}}\left(\qp, \vec{Z}, \gamma\right)$ occur whenever
  $\langle \vec{u}_i'  , \vec{Z}\rangle$ shifts from positive to negative for
  some $i \in [n]$. Between discontinuities, the function is constant. By
  definition of $ \vec{u}_i'$, this happens when $\gamma = \tan^{-1} \left(- \langle
  \vec{u}_i  , \vec{Z}[1, \dots, n]\rangle/Z[n+i]\right)$, which comes from a
  $1/\pi$-bounded distribution. The next challenge is that the discontinuities
  are not independent: the $n$ discontinuities from instance $t$ depend on the
  same vector $\vec{Z}^{\left(t\right)}$. To overcome this, we let $\cB_i$ denote the set
  of discontinuities contributed by vector $\vec{u}_i$ across all instances.
The buckets $\cB_i$ partition the set of discontinuities into
  $P = n$ sets, each containing at most $\numfunctions$ discontinuities. We then
  apply Lemma~\ref{lem:dispersion} with $P$ and $M = \numfunctions$ to prove the
  claim.
\end{proof}

Next, we prove the following guarantee for $\uslin$. The full proof is in
Appendix~\ref{app:applications}, where we also demonstrate the theorem's implications for
our optimization settings (Theorems~\ref{thm:slinear_DP}, \ref{thm:slinear_full_info}, and
\ref{thm:slinear_full_info_DP}).

\begin{restatable}{thm}{slinDispersed}\label{thm:slin_dispersed}
  With probability $1-\zeta$ over $\vec{Z}^{\left(1\right)}, \dots, \vec{Z}^{\left(\numfunctions\right)}
  \sim \mathcal{N}_n$, for any matrices $\qp^{\left(1\right)}, \dots, \qp^{\left(\numfunctions\right)}$ and any
  $\alpha \geq 1/2$, the functions $\uslin\left(\vec{Z}^{\left(1\right)}, \qp^{\left(1\right)}, \cdot\right), \dots, \uslin\left(\vec{Z}^{\left(\numfunctions\right)}, \qp^{\left(\numfunctions\right)}, \cdot\right)$ are piecewise $L$-Lipschitz
  with $L = \tilde{O}\left(M\numfunctions^3n^5/\zeta^3\right)$, where $M = \max_{i,j \in [n], t \in [\numfunctions]}|a_{ij}^{\left(t\right)}|$, and $\uslin$ is \[\left(\numfunctions^{\alpha -1},
  O\left(n \numfunctions^\alpha \sqrt{\log\frac{n}{\zeta}}\right)\right)\text{-dispersed}\] with respect to $\sample =
  \left\{\left(\qp^{\left(t\right)}, \vec{Z}^{\left(t\right)}\right)\right\}_{t = 1}^\numfunctions$.
\end{restatable}

\begin{proof}[Proof sketch] We show that over the randomness of $\vec{Z}^{\left(1\right)}, \dots, \vec{Z}^{\left(\numfunctions\right)}$,
$u_{\text{slin}}$ is $\left(w,k\right)$-dispersed. By definition of $\phi_s$, the
discontinuities of $u_{\text{slin}}\left(\qp^{\left(t\right)}, \vec{Z}^{\left(t\right)}, \cdot\right)$ have the
form $s = |\langle \vec{u}_i^{\left(t\right)}, \vec{Z}^{\left(t\right)} \rangle|$, where
$\vec{u}_i^{\left(t\right)}$ is the $i^{\rm th}$ vector in the solution to SDP-relaxation
of $A^{\left(t\right)}$. These random variables have density bounded by $\sqrt{2/\pi}$.
Let $\cB_i$ be the set of discontinuities contributed by
$\vec{u}_i^{\left(1\right)}, \dots, \vec{u}_i^{\left(\numfunctions\right)}$. The points within each $\cB_i$ are
independent. We apply Lemma~\ref{lem:dispersion} with $P = n$ and $M =
\numfunctions$ and arrive at our dispersion guarantee.

Proving that the
piecewise portions of $u_{\text{slin}}$ are Lipschitz is complicated by the fact
that they are quadratic in $1/s$, so the slope may go to $\pm \infty$ as $s$
goes to 0. However, if $s$ is smaller than the smallest boundary $s_0$, $\uslin\left(\vec{Z}^{\left(t\right)}, \qp^{\left(t\right)}, \cdot\right)$ is constant because
$\phi_s$ deterministically maps the variables to $-1$ or 1, as in the GW algorithm. We prove that $s_0$
is not too small using anti-concentration bounds. The Lipschitz constant is then roughly bounded by $n^2/s_0^3$, since we take the derivative of the sum of $n^2$ inverse quadratic functions.
\end{proof}

\section{Dispersion in pricing problems and auction design}\label{sec:auctions}
In this section, we study $n$-bidder, $m$-item posted price mechanisms and second price auctions.
We denote all $n$ buyers' valuations for all $2^m$ bundles $b_1, \dots, b_{2^m} \subseteq [m]$ by \[\vec{v} = (v_1(b_1), \dots, v_1(b_{2^m}), \dots, v_n(b_1), \dots, v_n(b_{2^m})).\]
We study buyers with additive valuations
$\left(v_j(b) = \sum_{i \in b} v_j(\{i\})\right)$ and unit-demand valuations
($v_j(b) = \max_{i \in b} v_j(\{i\})$). We also study buyers with general
valuations, where there is no assumption on $v_j$ beyond the fact that it is
nonnegative, monotone, and $v_j(\emptyset) = 0$.

\emph{Posted price mechanisms} are defined by $m$ prices $\rho_1, \dots,
\rho_m$ and a fixed ordering over the buyers. In order, each buyer has the
option of buying her utility-maximizing bundle among the remaining items. In
other words, suppose it is buyer $j$'s turn in the ordering and let $I$ be the set of items that buyers before her in the ordering did not
buy. Then she will buy the bundle $b \subseteq I$ that maximizes $v_j(b) -
\sum_{i \in b} \rho_i$.

\emph{Second price item auctions with anonymous reserve prices} are defined by
$m$ \emph{reserve prices} $\rho_1, \dots, \rho_m$. The bidders
submit bids for each of the items. For each item $i$, the highest bidder wins
the item if her bid is above $\rho_i$ and she pays the maximum of the second
highest bid for item $i$ and $\rho_i$. These auctions are only \emph{strategy proof}
for additive bidders,
which means that buyers have no incentive to misreport their values.
Therefore, we restrict our attention to this setting and assume the bids equal the values.

In this setting, $\Pi$ is a set of valuation vectors $\vec{v}$ and as in Section~\ref{sec:applications}, each $\dist^{(t)}$ is a distribution over $\Pi$.
The following results hold whenever the utility function corresponds to \emph{revenue} (the sum of the payments) or \emph{social surplus} (the sum of the buyers' values for their allocations). The full proof is in Appendix~\ref{app:auctions}.

\begin{restatable}{thm}{itemPricing}
\label{thm:item_pricing}
Suppose that $u(\vec{v}, \vec{\rho})$ is the social welfare (respectively, revenue) of the posted price mechanism with prices $\vec{\rho}$ and buyers' values $\vec{v}$. In this case, $L=0$ (respectively, $L=1$). The following are each true
with probability at least $1-\zeta$ over the draw $\sample \sim \dist^{(1)} \times \cdots
\times \dist^{(\numfunctions)}$ for any $\alpha \geq 1/2$:
\begin{enumerate}
\item Suppose the buyers have additive valuations and for each distribution $\dist^{(t)}$,
the item values have $\kappa$-bounded marginal distributions. Then $u$ is
\[\left(\frac{1}{2\kappa \numfunctions^{1-\alpha}}, O\left(nm\numfunctions^\alpha\sqrt{\ln\frac{nm}{\zeta}}\right)\right)\text{-dispersed}\] with respect to $\sample$.
\item Suppose the buyers are unit-demand with $v_j(\{i\}) \in [0,W]$ for each buyer $j \in [n]$ and item $i \in [m]$. Also, suppose that for each distribution
$\dist^{(t)}$, each buyer $j$, and every pair of items $i$ and $i'$, $v_j(\{i\})$ and $v_j(\{i'\})$ have a $\kappa$-bounded joint distribution.
Then $u$ is \[\left(\frac{1}{2W\kappa \numfunctions^{1-\alpha}}, O\left(nm^2\numfunctions^\alpha\sqrt{\ln\frac{nm}{\zeta}}\right)\right)\text{-dispersed}\] with respect to $\sample$.
\item Suppose the buyers have general valuations in $[0,W]$. Also, suppose that for each 
$\dist^{(t)}$, each buyer $j$, and every pair of bundles $b$ and $b'$, $v_j(b)$ and $v_j(b')$ have a $\kappa$-bounded joint distribution.
Then $u$ is \[\left(\frac{1}{2W\kappa \numfunctions^{1-\alpha}},
O\left(n2^{2m}\numfunctions^\alpha\sqrt{\ln\frac{n2^m}{\zeta}}\right)\right)\text{-dispersed}\] with respect to $\sample$.
\end{enumerate}
\end{restatable}

\begin{proof}[Proof sketch]
  We sketch the proof for additive buyers. Given a valuation vector $\vec{v}$,
  let $\partition_{\vec{v}}$ be the partition of $\configs$ over which
  $u(\vec{v}, \cdot)$ is Lipschitz. We prove that the boundaries of
  $\partition_{\vec{v}}$ correspond to a set of hyperplanes. Since the buyers
  are additive, these hyperplanes are axis-aligned: buyer $j$ will be willing to
  buy item $i$ at a price $\rho_i$ if and only if $v_j(\{i\}) \geq \rho_i$.
  Next, consider a set $\sample = \left\{\vec{v}^{(1)}, \dots,
  \vec{v}^{(\numfunctions)}\right\}$ of buyers' valuations and the hyperplanes
  corresponding to each partition $\partition_{\vec{v}^{(i)}}$. The key insight
  is that these hyperplanes can be partitioned into $P = nm$
  buckets consisting of parallel hyperplanes with offsets independently drawn
  from $\kappa$-bounded distributions. For additive buyers, these sets of
  hyperplanes have the form $\{v^{(1)}_j(\{i\}) = \rho_i, \dots,
  v^{(\numfunctions)}_j(\{i\}) = \rho_i\}$ for every item $i$ and every buyer
  $j$. Using Lemma~\ref{lem:dispersion}, we show that within each bucket, the
  offsets are $(w,k)$-dispersed, for $w=O(1/(\kappa \numfunctions^{1-\alpha}))$
  and $k = \tilde{O}(nm\numfunctions^{\alpha})$. Since the hyperplanes within
  each set are parallel, and since their offsets are dispersed, for any ball
  $\mathcal{B}$ of radius $w$ in $\configs$, at most $k$ hyperplanes from each
  set intersect $\mathcal{B}$. By a union bound, this implies that the $u$ is
  $(w,nmk)$-dispersed with respect to $\sample$. \end{proof}

We use a similar technique to analyze second-price item auctions. The full proof is in Appendix~\ref{app:auctions}, where we also show that Theorem~\ref{thm:item_pricing} and the following theorem imply optimization guarantees in our settings.

\begin{restatable}{thm}{secondPrice}
  \label{thm:2nd_price}
Suppose that $u(\vec{v}, \vec{\rho})$ is the social welfare (respectively, revenue) of the second-price auction with reserves $\vec{\rho}$ and bids $\vec{v}$. In this case, $L=0$ (respectively, $L=1$). Also, for each $\dist^{(t)}$ and each item $i$, suppose the
  distribution over $\max_{j \in [n]} v_j(\{i\})$ is $\kappa$-bounded. For any
  $\alpha \geq 1/2$, with probability $1-\zeta$ over the draw of $\sample \sim
  \bigtimes_{t=1}^\numfunctions \dist^{(t)}$, $u$ is \[\left(\frac{1}{2\kappa \numfunctions^{1-\alpha}},
  O\left(m\numfunctions^\alpha\sqrt{\ln\frac{m}{\zeta}}\right)\right)\text{-dispersed}\] with respect to $\sample$.
\end{restatable}

\section{Generalization guarantees for distributional learning}\label{sec:generalization}
It is known that regret bounds imply generalization guarantees for various
online-to-batch conversion algorithms~\citep{Cesa-Bianchi02:Generalization}, but
we also show that dispersion can be used to explicitly provide \emph{uniform
convergence guarantees,} which bound the difference between any function's
average value on a set of samples drawn from a distribution and its expected value. Our
primary tool is \emph{empirical Rademacher
complexity}~\citep{Koltchinskii01:Rademacher, Bartlett02:Rademacher}, which is
defined as follows. Let $\cF = \{f_{\vec{\rho}} : \Pi \to [0,1] \,:\,
\vec{\rho} \in \configs\}$, where $\configs \subset \reals^d$ is a parameter
space and let  $\sample = \{x_1, \dots, x_\numfunctions\} \subseteq \Pi$. (We use this
notation for the sake of generality beyond algorithm selection, but mapping to
the notation from Section~\ref{sec:IntroSetting}, $f_{\vec{\rho}}(x) = u(x,
\vec{\rho})$.) The empirical Rademacher complexity of $\cF$ with respect to
$\sample$ is defined as $\hat R(\cF, \sample) = \expect_{\vec{\sigma}} \bigl[\sup_{f
\in \cF} \frac{1}{\numfunctions}\sum_{i=1}^\numfunctions \sigma_i f(x_i)\bigr]$, where $\sigma_i \sim
U(\{-1, 1\})$. Classic results from learning
theory~\citep{Koltchinskii01:Rademacher, Bartlett02:Rademacher} guarantee that
for any distribution $\dist$ over $\Pi$, with probability $1-\zeta$ over
$\sample = \{x_1, \dots, x_\numfunctions\} \sim \dist^\numfunctions$, for all $f_{\vec{\rho}} \in \cF$,
$\bigl|\frac{1}{\numfunctions}\sum_{i=1}^\numfunctions f_{\vec{\rho}}(x_i) - \expect_{x \sim
\dist}[f_{\vec{\rho}}(x)]\bigr| = O(\hat R(\cF,\sample) + \sqrt{\log(1/\zeta)/\numfunctions})$.
Our bounds depend on the the dispersion parameters of functions belonging to the
\emph{dual} class $\cG$. That is, let $\cG = \{ u_x : \configs \to \reals \,:\,
x \in \Pi \}$ be the set of functions $u_x(\vec{\rho}) = f_{\vec{\rho}}(x)$
where $x$ is fixed and $\vec{\rho}$ varies. We bound $\hat R(\cF, \sample)$ in
terms of the dispersion parameters satisfied by $u_{x_1}, \dots, u_{x_\numfunctions} \in
\cG$. Moreover, even if these functions are not well dispersed, we can always
upper bound $\hat R(\cF, \sample)$ in terms of the pseudo-dimension of $\cF$,
denoted by $\pdim(\cF)$ (we review the definition in
Appendix~\ref{app:generalization}). The full proof of
Theorem~\ref{thm:dispersionRademacher} is in Appendix~\ref{app:generalization}.

\begin{restatable}{thm}{thmDispersionRademacher}
\label{thm:dispersionRademacher}
  Let $\cF = \{ f_{\vec{\rho}} : \Pi \to [0,1] \,:\, \vec{\rho} \in \configs\}$
  be parameterized by $\configs \subset \reals^d$, where
  $\configs$ lies in a ball of radius $R$. For any set $\sample = \{x_1,
  \dots, x_\numfunctions\}$, suppose the functions $u_{x_i}(\vec{\rho}) =
  f_{\vec{\rho}}(x_i)$ for $i \in [\numfunctions]$ are piecewise $L$-Lipschitz and
  $(w,k)$-dispersed. Then \[\hat R(\cF, \sample) \leq O\left(\min\left\{\sqrt{\frac{d}{\numfunctions} \log\frac{R}{w}} + Lw + \frac{k}{\numfunctions},
  \sqrt{\frac{\pdim(\cF)}{\numfunctions}}\right\}\right).\]
\end{restatable}
\begin{proof}[Proof sketch]
  The key idea is that when the functions $u_{x_1}, \dots, u_{x_\numfunctions}$ are
  $(w,k)$-dispersed, any pair of parameters $\vec{\rho}$ and
  $\vec{\rho}'$ with $\norm{\vec{\rho} - \vec{\rho'}}_2 \leq w$ satisfy
  $|f_{\vec{\rho}}(x_i) - f_{\vec{\rho}'}(x_i)| = |u_{x_i}(\vec{\rho}) -
  u_{x_i}(\vec{\rho}')| \leq Lw$ for all but at most $k$ of the elements in
  $\sample$. Therefore, we can approximate the functions in $\cF$
  on the set $\sample$ with a finite subset $\hat \cF_w = \{
  f_{\hat{\vec{\rho}}} \,:\, \hat{\vec{\rho}} \in \hat \configs_w \}$, where
  $\hat \configs_w$ is a $w$-net for $\configs$. Since $\hat \cF_w$ is
  finite, its empirical Rademacher complexity is $O((\log |\hat
  \cF_w|/\numfunctions)^{1/2})$. We then argue that the empirical Rademacher complexity of
  $\cF$ is not much larger, since all functions in $\cF$ are
  approximated by some function in $\hat \cF_w$.
\end{proof}

\section{Conclusion}
We study online and private optimization for application-specific algorithm selection. We introduce a general condition, dispersion, that allows us
to provide strong guarantees for both of these settings. As
we demonstrate, many problems in algorithm and auction design reduce to
optimizing dispersed functions. In this way, we connect learning
theory, differential privacy, online learning, bandits, high dimensional
sampling, computational economics, and algorithm design. Our main motivation is algorithm
selection, but we expect that dispersion is even more widely applicable, opening up
an exciting research direction.

\subsection*{Acknowledgements}

The authors would like to thank Yishay Mansour for valuable feedback and discussion. This work was supported in part by NSF grants CCF-1422910, CCF-1535967, IIS-1618714, an Amazon Research Award, a Microsoft Research Faculty Fellowship, a Google Research Award, a NSF Graduate Research Fellowship, and a Microsoft Research Women's Fellowship.

\newpage
\bibliographystyle{plainnat}
\bibliography{../../Library}

\appendix

\section{Generic lemmas for dispersion}\label{app:randomCriticalPoints}
In this appendix we provide several general tools for demonstrating that a
collection of functions will be $(w,k)$-dispersed. The dispersion analyses for
each of our applications leverages the general tools presented here. We first recall the
definition of dispersion.

\dispersionDef*

We begin by proving the dispersion lemma from Section~\ref{sec:dispersion}.
\dispersionLem*
\begin{proof}
  We begin by proving part 1 of the statement. The expected number of samples
  that land in any interval $I$ of width $w$ is at most $w \kappa r$, since for
  each $i \in [r]$, the probability $\beta_i$ lands in $I$ is at most $w\kappa$.
  If the distributions $p_1, \dots, p_r$ were identical, then the $\beta_i$
  would be i.i.d. samples and we could apply standard uniform convergence
  results leveraging the fact that the VC-dimension of intervals is 2. It is
  folklore that these uniform convergence results also apply for independent but
  not identically distributed random variables. We provide a proof of this fact
  in Lemma~\ref{lem:nonidentical_rad} for completeness.  By
  Lemma~\ref{lem:nonidentical_rad}, we know that with probability at least
  $1-\zeta$ over the draw of the set $\mathcal{B}$, \[\sup_{a,b \in \R,
  a<b}\left( \sum_{i = 1}^r \textbf{1}_{\beta_i \in (a,b)} - \E_{\mathcal{B}'}
  \left[\sum_{i = 1}^r \textbf{1}_{\beta_i' \in (a,b)}\right]\right) \leq
  O\left(\sqrt{r \log \frac{1}{\zeta}}\right),\] where $\mathcal{B}' =
  \{\beta_1', \dots, \beta_r'\}$ is another sample drawn from $p_1, \dots, p_r$.
  This implies that with probability at least $1-\zeta$, every interval $I$ of
  width $w$ satisfies $|\mathcal{B} \cap I| \leq w\kappa r + O(\sqrt{r
  \log(1/\zeta)})$. For any $\alpha \geq 1/2$, setting $w = r^{\alpha-1}/\kappa$
  gives $|\mathcal{B} \cap I| = O(r^\alpha \sqrt{\log 1/\zeta})$ for all
  intervals of width $w$ with probability at least $1-\zeta$.

  Next we prove part 2. Applying the argument from part 1 to each bucket
  $\mathcal{B}_i$, we know that with probability at least $1-\zeta/P$, any
  interval of width $w$ contains at most $w\kappa M + O(\sqrt{M \log(P/\zeta)})$
  samples belonging to $\mathcal{B}_i$. Taking the union bound over the $P$
  buckets, it follows that with probability at least $1-\zeta$, every interval
  of width $w$ contains at most $P(w\kappa M + O(\sqrt{M \log(1/\zeta)}))$
  samples in total from all $P$ buckets. For any $\alpha \geq 1/2$, setting $w =
  M^{\alpha-1}/\kappa$ guarantees that the number of samples in any interval of
  width $w$ is at most $O(PM^\alpha \sqrt{\log(P/\zeta)})$.
\end{proof}

\begin{cor}
  \label{cor:uniformdispersed} Let $\mathcal{B} = \left\{\beta_1, \dots,
  \beta_r\right\}$ be a collection of samples where $\beta_i \sim \uniform([a_i,
  a_i + W])$ and $a_1, \dots, a_r, W$ are arbitrary parameters. For any $\zeta >
  0$ and $\alpha \geq 1/2$, with probability at least $1-\zeta$, every interval
  of width $w = \frac{W}{r^{1-\alpha}}$ contains at most $O\biggl(r^\alpha
  \sqrt{\log \frac{1}{\zeta}}\biggr)$ points. \end{cor}
  \begin{proof}
  The density function for a uniform random variable on an interval of width $W$
  is $1/W$. Therefore, the corollary follows from part 1 of
  Lemma~\ref{lem:dispersion}.
\end{proof}

Finally, for completeness, we include the following folklore lemma which allows us to
use uniform convergence for non-identical random variables, whereas typical
uniform convergence bounds are written in terms of identical random variables.
It follows by modifying the well-known proof for uniform convergence using
Rademacher complexity \citep{Bartlett02:Rademacher, Koltchinskii01:Rademacher,
Shalev14:Understanding}.

\begin{lem}\label{lem:nonidentical_rad}
Let $\cB = \{\beta_1, \dots, \beta_r\} \subset \reals$ be a set of random
  variables where $\beta_i \sim p_i$. For any $\zeta > 0$, with probability at
  least $1-\zeta$ over the draw of the set $\mathcal{B}$, \[\sup_{a,b \in \R, a<b}\left( \sum_{i = 1}^r \textbf{1}_{\beta_i \in (a,b)} - \E_{\mathcal{B}'} \left[\sum_{i = 1}^r \textbf{1}_{\beta_i' \in (a,b)}\right]\right) \leq O\left(\sqrt{r \ln \frac{1}{\zeta}}\right),\] where
  $\mathcal{B}' = \{\beta_1', \dots, \beta_r'\}$ is another sample drawn from $p_1, \dots, p_r$.
\end{lem}

\begin{proof}
Let $\vec{\sigma}$ be a vector of Rademacher random variables. Since the VC-dimension of intervals is 2, we know from work by \citet{Dudley67:Sizes} that \begin{equation}\E_{\vec{\sigma}}\left[\sup_{a,b \in \R, a<b} \sum_{i = 1}^r \sigma_i\textbf{1}_{\beta_i \in (a,b)}\right] \leq O\left(\sqrt r\right).
\label{eq:chaining}\end{equation}

Also, we have that \begin{align*}\sup_{a,b \in \R, a<b}\left( \sum_{i = 1}^r \textbf{1}_{\beta_i \in (a,b)} - \E_{\mathcal{B}'} \left[\sum_{i = 1}^r \textbf{1}_{\beta_i' \in (a,b)}\right]\right) &= \sup_{a,b \in \R, a<b} \E_{\mathcal{B}'} \left[\sum_{i = 1}^r \textbf{1}_{\beta_i \in (a,b)} - \sum_{i = 1}^r \textbf{1}_{\beta_i' \in (a,b)}\right]\\
 &\leq \E_{\mathcal{B}'} \left[\sup_{a,b \in \R, a<b}\left(\sum_{i = 1}^r \textbf{1}_{\beta_i \in (a,b)} -  \textbf{1}_{\beta_i' \in (a,b)}\right)\right].\end{align*} Taking the expectation over the draw of $\mathcal{B}$, we have that \[\E_{\mathcal{B}}\left[\sup_{a,b \in \R, a<b}\left( \sum_{i = 1}^r \textbf{1}_{\beta_i \in (a,b)} - \E_{\mathcal{B}'} \left[\sum_{i = 1}^r \textbf{1}_{\beta_i' \in (a,b)}\right]\right)\right] \leq \E_{\mathcal{B}, \mathcal{B}'} \left[\sup_{a,b \in \R, a<b}\left(\sum_{i = 1}^r \textbf{1}_{\beta_i \in (a,b)} -  \textbf{1}_{\beta_i' \in (a,b)}\right)\right].\] For each $i$, $\beta_i$ and $\beta_i'$ are independent and identically distributed. Therefore, we can switch them without replacing the expectation, as follows.

\[ \E_{\mathcal{B}, \mathcal{B}'} \left[\sup_{a,b \in \R, a<b}\left(\sum_{i = 1}^r \textbf{1}_{\beta_i \in (a,b)} -  \textbf{1}_{\beta_i' \in (a,b)}\right)\right] =  \E_{\mathcal{B}, \mathcal{B}'} \left[\sup_{a,b \in \R, a<b}\left(\sum_{i = 1}^r \textbf{1}_{\beta_i' \in (a,b)} -  \textbf{1}_{\beta_i \in (a,b)}\right)\right].\] Letting $\sigma_i$ be a Rademacher random variable, we have that \[ \E_{\mathcal{B}, \mathcal{B}'} \left[\sup_{a,b \in \R, a<b}\left(\sum_{i = 1}^r \textbf{1}_{\beta_i \in (a,b)} -  \textbf{1}_{\beta_i' \in (a,b)}\right)\right] =  \E_{\vec{\sigma}, \mathcal{B}, \mathcal{B}'} \left[\sup_{a,b \in \R, a<b}\left(\sum_{i = 1}^r \sigma_i\left(\textbf{1}_{\beta_i \in (a,b)} -  \textbf{1}_{\beta_i' \in (a,b)}\right)\right)\right].\] Since \[\sup_{a,b \in \R, a<b}\left(\sum_{i = 1}^r \sigma_i\left(\textbf{1}_{\beta_i \in (a,b)} -  \textbf{1}_{\beta_i' \in (a,b)}\right)\right) \leq \sup_{a,b \in \R, a<b}\sum_{i = 1}^r \sigma_i\textbf{1}_{\beta_i \in (a,b)} + \sup_{a,b \in \R, a<b}\sum_{i = 1}^r-\sigma_i \textbf{1}_{\beta_i' \in (a,b)},\] we have that \begin{align*}&\E_{\vec{\sigma}, \mathcal{B}, \mathcal{B}'} \left[\sup_{a,b \in \R, a<b}\left(\sum_{i = 1}^r \sigma_i\left(\textbf{1}_{\beta_i \in (a,b)} -  \textbf{1}_{\beta_i' \in (a,b)}\right)\right)\right]\\
\leq &\E_{\vec{\sigma}, \mathcal{B}} \left[\sup_{a,b \in \R, a<b}\sum_{i = 1}^r \sigma_i\textbf{1}_{\beta_i \in (a,b)}\right] + \E_{\vec{\sigma}, \mathcal{B}'} \left[\sup_{a,b \in \R, a<b}\sum_{i = 1}^r \sigma_i\textbf{1}_{\beta_i' \in (a,b)}\right]\\
= &2\E_{\vec{\sigma}, \mathcal{B}} \left[\sup_{a,b \in \R, a<b}\sum_{i = 1}^r \sigma_i\textbf{1}_{\beta_i \in (a,b)}\right].\end{align*} All in all, this means that \begin{equation}\sup_{a,b \in \R, a<b}\left( \sum_{i = 1}^r \textbf{1}_{\beta_i \in (a,b)} - \E_{\mathcal{B}'} \left[\sum_{i = 1}^r \textbf{1}_{\beta_i' \in (a,b)}\right]\right) \leq 2\E_{\vec{\sigma}, \mathcal{B}} \left[\sup_{a,b \in \R, a<b}\sum_{i = 1}^r \sigma_i\textbf{1}_{\beta_i \in (a,b)}\right].\label{eq:rad}\end{equation}

We now apply McDiarmid's Inequality to \begin{equation}\E_{\vec{\sigma} \sim \{-1,1\}^r}\left[\sup_{a,b \in \R, a<b} \sum_{i = 1}^r \sigma_i\textbf{1}_{\beta_i \in (a,b)}\right]\label{eq:erad}.\end{equation} Notice that if we switch $\beta_j$ with an arbitrary $\beta_j'$, Equation~\eqref{eq:erad} will change by at most 1. Therefore, with probability at least $1-\zeta$ over the draw of $\mathcal{B}$, \begin{equation}\left| \E_{\vec{\sigma}}\left[\sup_{a,b \in \R, a<b} \sum_{i = 1}^r \sigma_i\textbf{1}_{\beta_i \in (a,b)}\right] - \E_{\vec{\sigma}, \mathcal{B}}\left[\sup_{a,b \in \R, a<b} \sum_{i = 1}^r \sigma_i\textbf{1}_{\beta_i \in (a,b)}\right]\right| \leq \sqrt{\frac{r}{2}\ln\frac{2}{\zeta}}.\label{eq:mcd}\end{equation}
Combining Equations~\eqref{eq:chaining}, \eqref{eq:rad}, and \eqref{eq:mcd}, we have that with probability at least $1-\zeta$, \[\sup_{a,b \in \R, a<b}\left( \sum_{i = 1}^r \textbf{1}_{\beta_i \in (a,b)} - \E_{\mathcal{B}'} \left[\sum_{i = 1}^r \textbf{1}_{\beta_i' \in (a,b)}\right]\right) \leq O\left(\sqrt{r \ln \frac{1}{\zeta}}\right).\]
\end{proof}

\subsection{Properties of $\kappa$-bounded distributions}\label{app:prelim}
In order to prove dispersion for many of our applications, we start by assuming
there is some randomness present in the relevant problem parameters and show
that this implies that the resulting utility functions are $(w,k)$-dispersed
with meaningful parameters. The key step of these arguments is to show that the
discontinuity locations resulting from the randomness in the problem parameters
have $\kappa$-bounded density functions. The following lemmas are helpful for
reasoning about how transformations of a $\kappa$-bounded random variable affect
the density upper bound.

\begin{lem}\label{lem:bounded}
Suppose $X$ and $Y$ are independent, real-valued random variables drawn from $\kappa$-bounded distributions. Let $Z = |X-Y|$. Then $Z$ is drawn from a $2\kappa$-bounded distribution.
\end{lem}

\begin{proof}
  Let $f_X$ and $f_Y$ be the density functions of $X$ and $Y$. The cumulative density function for $Z$ is
  \begin{align*}
    F_{Z}(z)
    &= \Pr[Z \leq z] = \Pr[Y-X \leq z \text{ and } X-Y \leq z] = \Pr[Y-z \leq X \leq z + Y]\\
    &= \int_{-\infty}^{\infty}\int_{y-z}^{y+z} f_{X,Y}(x,y)\, dx dy = \int_{-\infty}^{\infty}\int_{y-z}^{y+z} f_X(x) f_Y(y)\, dx dy\\
    &=  \int_{-\infty}^{\infty} (F_X(y+z) - F_X(y-z)) f_Y(y)\, dy.
  \end{align*}
  Therefore, applying the fundamental theorem of calculus, the density function of $Z$ can be bounded as follows:
  \begin{align*}
    f_{Z}(z)
    &= \frac{d}{dz} F_{Z}(z)
    = \int_{-\infty}^{\infty} \frac{d}{dz}(F_X(y+z) - F_X(y-z)) f_Y(y)\, dy \\
    &=  \int_{-\infty}^{\infty} (f_X(y+z) + f_X(y-z)) f_Y(y)\, dy \leq 2\kappa\int_{-\infty}^{\infty} f_Y(y)\, dy
    = 2\kappa.
  \end{align*}
\end{proof}

Next, we show that even when $X$ and $Y$ are dependent random variables with a
$\kappa$-bounded joint distribution, $X-Y$ has a $W\kappa$-bounded distribution,
as long as the support set of $X$ and $Y$ are of width at most $W$.

\begin{lem}\label{lem:bounded_joint}
  Suppose $X$ and $Y$ are real-valued random variables taking values in $[a, a+W]$ and $[b,
  b+W]$ for some $a,b,W \in \R$ and suppose that their joint distribution is
  $\kappa$-bounded. Let $Z = X-Y$. Then $Z$ is drawn from a $W\kappa$-bounded
  distribution.
\end{lem}

\begin{proof}
The cumulative density function for $Z$ is \begin{align*}F_Z(z) &= \Pr[Z \leq z] = \Pr[X-Y \leq z] = \Pr[X \leq z + Y]\\
&= \int_b^{b+W} \int_{a}^{y+z} f_{X,Y}(x,y) \, dx dy.\end{align*}
The density function for $Z$ is \begin{align*}
f_{Z}(z) &= \frac{d}{dz} F_{Z}(z)\\
&= \frac{d}{dz} \int_b^{b+W} \int_{a}^{y+z} f_{X,Y}(x,y) \, dx dy\\
&= \int_b^{b+W} \frac{d}{dz} \int_{a}^{y+z} f_{X,Y}(x,y) \, dx dy\\
&= \int_b^{b+W} \left( \frac{d}{dz}\int_{a}^y f_{X,Y}(x,y)\, dx + \frac{d}{dz}\int_{0}^z f_{X,Y}(y+t,y)\, dt\right) dy\\
&= \int_b^{b+W} \left(0 + f_{X,Y}(y+z,y)\right) dy \leq W\kappa,
\end{align*} as claimed.
\end{proof}

Finally, we prove that if $X$ and $Y$ have support in $(0,1]$ and a
$\kappa$-bounded joint distribution, then $\ln(X)$ and $\ln(Y)$ have a
$\kappa$-bounded joint distribution as well. We will use this fact to show that
$\ln(X) - \ln(Y)$ is $\kappa/2$-bounded. These results are primarily useful for
the maximum weight independent set and knapsack algorithm selection dispersion
analyses.

\begin{lem}\label{lem:bounded_joint_ln}
  Suppose $X$ and $Y$ are random variables taking values in $(0,1]$ and suppose
  that their joint distribution is $\kappa$-bounded. Let $A = \ln X$ and $B =
  \ln Y$. Then $A$ and $B$ have a $\kappa$-bounded joint distribution.
\end{lem}

\begin{proof}
  We will perform a change of variables using the function $g(x,y) = (\ln x, \ln
  y)$. Let $g^{-1}(a,b) =h(a,b) = (e^a,  e^b)$. Then $f_{A,B}(a,b) =
  f_{X,Y}(a,b)|J_h(a,b)| \leq \kappa e^ae^b \leq \kappa$, where $J_h$ is the
  Jacobian matrix of $h$.
\end{proof}

\begin{lem}\label{lem:ln_difference_bounded}
Suppose $X$ and $Y$ are random variables taking values in $(0,1]$ and suppose that their joint distribution is $\kappa$-bounded. Then the distribution of $\ln(X) - \ln(Y)$ is $\kappa/2$ bounded.
\end{lem}

\begin{proof}
Let $Z = \ln(X) - \ln(Y)$. We will perform change of variables using the function $g(x,y) = (x, \ln(x) - \ln(y))$. Let $g^{-1}(x,z) = h(x,z) = (x, xe^{-z}).$ Then \[J_h(x,z) = \det \begin{pmatrix}
1 & e^{-z}\\
0 & -xe^{-z}
\end{pmatrix} = -xe^{-z}.\] Therefore, $f_{X,Z}(x,z) = xe^{-z}f_{X,Y}(x,xe^{-z}).$ This means that $f_Z(z) = \int_0^1 xe^{-z}f_{X,Y}(x,xe^{-z}) \, dx \leq \frac{\kappa}{2e^z}$, so when $z \geq 0$, $f_Z(z) \leq \kappa/2$.

Next, we will perform change of variables using the function $g(x,y) = (\ln(x) - \ln(y),y)$. Let $g^{-1}(z,y) = h(z,y) = (ye^{z},y).$ Then \[J_h(x,z) = \det \begin{pmatrix}
ye^z & 0\\
e^z & 1
\end{pmatrix} = ye^{z}.\] Therefore, $f_{Z,Y}(z,y) = ye^{z}f_{X,Y}(ye^{z},y).$ This means that $f_Z(z) = \int_0^1 ye^{z}f_{X,Y}(ye^{z},y) \, dy \leq \frac{\kappa e^z}{2}$, so when $z \leq 0$, $f_Z(z) \leq \kappa/2$.

Combining these two bounds, we see that $f_Z(z) \leq \kappa/2$.
\end{proof}

\begin{lem}\label{lem:ratio}
Suppose $X$ and $Y$ are two independent continuous random variables. Suppose that $Y$ has a $\kappa$-bounded density function and $-W \leq X \leq W$ with probability 1. Then $Y/X$ has a $\kappa W$-bounded density function.
\end{lem}

\begin{proof}
Let $Z = \frac{Y}{X}$ and let $f_Z$ be the probability density function of $Z$. We want to show that for all $z \in \R$, $f_Z(z) \leq \kappa W$.

It is well-known (e.g.,~\cite{Rohatgi15:Introduction}) that because $X$ and $Y$ are independent, \[f_Z(z) = \int_{-\infty}^{\infty} |x|f_{X}(x)f_Y(zx)\,dx.\] Since $Y$ has a $\kappa$-bounded density function and $-W \leq X \leq W$ with probability 1, this means that \begin{align*}f_Z(z) &= \int_{-\infty}^{\infty} |x|f_{X}(x)f_Y(zx)\, dx \leq \kappa \int_{-\infty}^{\infty} |x| f_{X}(x)\, dx = \kappa \int_{-W}^{W} |x|f_{X}(x)\, dx \leq \kappa W \int_{-W}^{W} f_{X}(x)\, dx\\
&= \kappa W.\end{align*} The first inequality follows because $Y$ has a $\kappa$-bounded density function, the second equality follows because $-W \leq X \leq W$ with probability 1, and the final equality follows because $f_X$ is a density function.
\end{proof}

\begin{lem}\label{lem:constant_mult}
Suppose $X$ is a random variable with $\kappa$-bounded distribution and suppose $c$ is a constant such that $|c| \in (0, W]$ for some $W \in \R$. Then $\frac{X}{c}$ has a $c\kappa$-bounded distribution.
\end{lem}

\begin{proof}
Let $f_X$ be the density function of the variable $X$. It is well-known~\citep{Tijms12:Understanding} that if the function $v(x)$ is strictly increasing or strictly decreasing, then the probability density of the random variable $Y = v(X)$ is given by $f_X(a(y))\left|a'(y)\right|$, where $a(y)$ is the inverse function of $v(x)$. In our setting $v(x) = \frac{x}{c}$, so $a(x) = cx$. Therefore, the probability density of $Y = \frac{X}{c}$ is $cf_X(cx)$. Since $X$ has a $\kappa$-bounded distribution, $\max cf_X(cx) \leq c\kappa$.
\end{proof}

\section{Efficient sampling}\label{app:efficient}
Both our differential privacy and online algorithms critically rely on our
ability to sample from a particular type of distribution. Specifically, let $g$
be a piecewise Lipschitz function mapping vectors in the set $\configs \subseteq
\R^d$ to $\R$. These applications require us to sample from a distribution $\mu$
with density proportional to $e^{g(\vec{\rho})}$. We use the notation
$f_{\mu}(\vec{\rho}) = e^{g(\vec{\rho})} / \int_C e^{g(\vec{\rho'})} \,
d\vec{\rho'}$ to denote the density function of $\mu$. In this section we
provide efficient algorithms for approximately sampling from $\mu$. Our utility
guarantees, privacy guarantees, and regret bounds in the following sections
include bounds that hold under approximate sampling procedures.

\subsection{Efficient implementation for 1-dimensional piecewise Lipschitz functions}

We begin with an efficient and exact algorithm for sampling from $\mu$ in
1-dimensional problems. Our algorithms for higher dimensional sampling have the
same basic structure. First, our algorithm requires that the parameter space
$\configs$ is an interval on the real line. Second, it requires that $f_{\mu}$
is piecewise defined with efficiently computable integrals on each piece of the
domain. More formally, suppose there are intervals
$\bigl\{[a_i,b_i)\bigr\}_{i=1}^K$ partitioning $\configs$ such that the
indefinite integral $F_i$ of $f_{\mu}$ restricted to $[a_i,b_i)$ is efficient to
compute. We propose a two-stage sampling algorithm. First, it randomly chooses
one of the intervals $[a_i,b_i)$ with probability proportional to
$\int_{a_i}^{b_i} f_{\mu}(\rho) \, d\rho = F_i(b_i) - F_i(a_i)$. Then, it
outputs a sample from the conditional distribution on that interval. By breaking
the problem into two stages, we take advantage of the fact that $f_{\mu}$ has a
simple form on each of its components. We thus circumvent the fact that
$f_{\mu}$ may be a complicated function globally. We provide the pseudocode in
Algorithm~\ref{alg:1dEfficient}.

\begin{algorithm}
\caption{One-dimensional sampling algorithm}\label{alg:1dEfficient}
\begin{algorithmic}[1]
\Require Function $g$, intervals
$\bigl\{[a_i,b_i)\bigr\}_{i=1}^K$ partitioning $\configs$
\State Define $h(\rho) = \exp\bigl(g(\rho)\bigr)$ and let $H_i$ be the indefinite integral of $h$ on $[a_i, b_i)$.
\State Let $Z_i = H_i(b_i) - H_i(a_i)$ and define $P_i(\rho) = \frac{1}{Z_i}\bigl(H_i(\rho) - H_i(a_i)\bigr)$.
\State Choose random interval index $I = i$ with probability $Z_i / \sum_j Z_j$.
\State Let $U$ be uniformly distributed in $[0,1]$ and set $\hat \rho = P_I^{-1}(U)$.
\Ensure $\hat{\rho}$
\end{algorithmic}
\end{algorithm}

The following lemma shows that Algorithm~\ref{alg:1dEfficient} exactly outputs a
sample from $f_\mu(\rho) \propto e^{g(\rho)}$.
\begin{restatable}{lem}{lemAlgOneDimDist}
  Algorithm~\ref{alg:1dEfficient} outputs samples from the distribution $\mu$
  with density $f_\mu(\rho) \propto e^{g(\rho)}$.
\end{restatable}

\begin{proof}
  Let $\mu$ be the target distribution. The density function for $\mu$ is given
  by $f_\mu(\rho) = h(\rho) / Z$, where $h(\rho) = e^{g(\rho)}$ and $Z =
  \int_\configs g(\rho) \, d\rho = \sum_{i=1}^K Z_i$. Let $\hat \rho$ be the
  output of Algorithm~\ref{alg:1dEfficient}. We need to show that $\prob(\hat
  \rho \leq \tau) = \int_{a_1}^\tau f_\mu(\rho) \, d\rho$ for all $\tau \in
  \configs$.

  Fix any $\tau \in \configs$ and let $T$ be the largest index $i$ such that
  $b_i \leq \tau$. Then we have
  \begin{align*}
  \prob(\hat \rho \leq \tau)
  &= \sum_{i=1}^T \prob(\hat \rho \in [a_i, b_i))
  + \prob(\hat \rho \in [a_{T+1}, \tau))
  = \frac{1}{Z}\sum_{i=1}^T Z_i + \frac{1}{Z}(H_{T+1}(\tau) - H_{T+1}(a_{T+1}))\\
  &= \frac{1}{Z}\sum_{i=1}^T \int_{a_i}^{b_i} h(\rho) \, d\rho + \frac{1}{Z} \int_{a_{T+1}}^\tau f(\rho) \,d\rho
  = \frac{1}{Z} \int_{a_1}^\tau h(\rho) \, d\rho
  = \int_{a_1}^\tau f_\mu(\rho) \, d\rho,
  \end{align*}
  as required.
\end{proof}

\subsection{Efficient approximate sampling in multiple dimensions}

In this section, we turn to the multi-dimensional setting. We present an
efficient algorithm for approximately sampling from $\mu$ with density
$f_\mu(\vec{\rho}) \propto e^{g(\vec{\rho})}$. It applies to the case where the
input function $g$ is piecewise concave and each piece of the domain is a convex
set. As in the single dimensional case, the algorithm first chooses one piece of
the domain with probability proportional to the integral of $f_\mu$ on that
piece, and then it outputs a sample from the conditional distribution on that
piece. See Algorithm~\ref{alg:efficient} for the pseudo-code. Our algorithm uses
techniques from high dimensional convex geometry. These tools allow us to
approximately integrate and sample efficiently. Bassily et
al.~\cite{Bassily14:ERM} used similar techniques for differentially private
convex optimization. Their algorithm also approximately samples from the
exponential mechanism's output distribution. We generalize these techniques to
apply to cases when the function $g$ is only piecewise concave.

We will frequently measure the distance between two probability measures in
terms of the relative (multiplicative) distance $\reldist$. This is defined as
$\reldist(\chi, \sigma) = \sup_{\vec{\rho}} \bigl| \log
\frac{d\chi}{d\sigma}(\vec{\rho}) \bigr|$, where $\frac{d\chi}{d\sigma}$ denotes the Radon-Nikodym
derivative. The following lemma characterizes the $\reldist$ metric in terms of
the probability mass of sets:
\begin{lem}
  \label{lem:reldistmult}
  For any probability measures $\chi$ and $\sigma$, we have that
  $\reldist(\chi, \sigma) \leq \beta$ if and only if for every
  set $S$ we have
  $e^{-\beta}\sigma(S) \leq \chi(S) \leq e^{\beta}
  \sigma(S)$.
\end{lem}
\begin{proof}
  First, suppose that $\reldist(\chi, \sigma) \leq \beta$. Then for every $\vec{\rho}$,
  we have that $-\beta \leq \log \frac{d\chi}{d\sigma}(\vec{\rho}) \leq \beta$.
  Exponentiating both sides gives $e^{-\beta} \leq \frac{d\chi}{d\sigma}(\vec{\rho})
  \leq e^\beta$. Now fix any set $A$. We have:
  \[
    \chi(A)
    = \int_A \frac{d\chi}{d\sigma}(\vec{\rho}) \, d\sigma(\vec{\rho})
    \leq e^\beta \int_A 1 \, d\sigma(s) = e^\beta \sigma(A).
  \]
  Similarly, $\chi(A) \geq e^{-\beta} \sigma(A)$.

  Now suppose that $e^{-\beta} \sigma(A) \leq \chi(A) \leq e^{\beta} \sigma(A)$ for all sets $A$ and let $\vec{\rho}$ be any point. Let $B_i = B(x, 1/i)$ be a sequence of decreasing balls converging to $\vec{\rho}$. The Lebesgue differentiation theorem gives that
  \[
  \frac{d\chi}{d\sigma}(\vec{\rho}) = \lim_{i \to 0} \frac{1}{\sigma(B_i)} \int_{B_i} \frac{d\chi}{d\sigma}(\vec{y}) \, d\sigma(\vec{y}) = \lim_{i \to 0} \frac{\chi(B_i)}{\sigma(B_i)}.
  \]
  Since $e^{-\beta} \leq \frac{\chi(B_i)}{\sigma(B_i)} \leq e^{\beta}$ for all
  $i$, it follows that $-\beta \leq \log \frac{d\chi}{d\sigma}(\vec{\rho}) \leq
  \beta$, as required.
\end{proof}

Our algorithm depends on two subroutines from high-dimensional convex
computational geometry. These subroutines use rapidly mixing random walks to
approximately integrate and sample from $\mu$. These procedures are efficient
when the function we would like to integrate or sample is logconcave. which
holds in our setting, since $f_\mu$ is piecewise logconcave when $g$ is
piecewise concave. Formally, we assume that we have access to two procedures,
$\intalg$ and $\samplealg$, with the following guarantees. Let $h : \R^d \to
\R_{\geq 0}$ be any logconcave function, we assume
\begin{enumerate}[leftmargin=*,itemsep=-1ex]
\item For any accuracy parameter $\alpha > 0$ and failure probability $\zeta >
0$, running $\intalg(h,\alpha,\zeta)$ outputs a number $\hat Z$ such that with
probability at least $1-\zeta$ we have $e^{-\alpha} \int h \leq \hat Z \leq
e^\alpha \int h$.
\item For any accuracy parameter $\beta > 0$ and failure probability $\zeta >
0$, running $\samplealg(h, \beta, \zeta)$ outputs a sample $\hat X$ drawn from
a distribution $\hat \mu_h$ such that with probability at least $1-\zeta$,
$\reldist(\hat \mu_h, \mu_h) \leq \beta$. Here, $\mu_h$ is the distribution with
density proportional to $h$.
\end{enumerate}
For example, the integration algorithm of Lov\'asz and Vempala~\cite{Vempala06:logconcave} satisfies
our assumptions on $\intalg$ and runs in time $\operatorname{poly}(d,
\frac{1}{\alpha}, \log \frac{1}{\zeta}, \log \frac{R}{r})$, where the domain of
$h$ is contained in a ball of radius $R$, and the level set of $h$ of
probability mass $1/8$ contains a ball of radius $r$. Similarly, Algorithm 6 of
Bassily et al.~\cite{Bassily14:ERM} satisfies our assumptions on $\samplealg$ with probability
1 and runs in time $\operatorname{poly}(d, L, \frac{1}{\beta}, \log
\frac{R}{r})$. When we refer to Algorithm~\ref{alg:efficient} in the rest of the
paper, we use these integration and sampling procedures.
\begin{algorithm}
\caption{Multi-dimensional sampling algorithm for piecewise concave functions}\label{alg:efficient}
\begin{algorithmic}[1]
\Require Piecewise concave function $g$, partition $\configs_1, \dots, \configs_K$ on which $g$ is concave, approximation parameter $\eta$, confidence parameter $\zeta$.
\State Define $\alpha = \beta = \eta/3$.
\State Let $h(\vec{\rho}) = \exp(g(\vec{\rho}))$ and  $h_i(\vec\rho) = \ind{\vec\rho\in\configs_i}h(\vec\rho)$ be $h$ restricted to  $\configs_i$.
\State For each $i \in [K]$, let $\hat Z_i = \intalg(h_i, \alpha, \zeta/(2K))$.
\State Choose random partition index $I = i$ with probability $\hat Z_i / \sum_j \hat
Z_j$.
\State Let $\hat{\vec{\rho}}$ be the sample output by
$\samplealg(h_I, \beta, \zeta/2)$.
\Ensure $\vec{\rho}$
\end{algorithmic}
\end{algorithm}

The main result in this section is that with high probability the output
distribution of Algorithm~\ref{alg:efficient} is close to $\mu$.

\begin{restatable}{lem}{lemEfficientApprox}
  \label{lem:efficientApprox}
  With probability at least $1-\zeta$ all the approximate integration and
  sampling operations performed by Algorithm~\ref{alg:efficient} succeed. Let
  $\hat \mu$ be the output distribution of Algorithm~\ref{alg:efficient}
  conditioned on all integration and sampling operations succeeding and let
  $\mu$ be the distribution with density $f_\mu(\vec{\rho}) \propto
  e^{g(\vec{\rho})}$. Then we have $\reldist\left(\hat \mu, \mu\right) \leq
  \eta$.
\end{restatable}

\begin{proof}
  First, with probability at least $1-\zeta$ every call to the subprocedures
  $\intalg$ and $\samplealg$ succeeds. Assume this high probability event occurs
  for the remainder of the proof.

  Let $\configs_1, \dots, \configs_K$, $f_\mu$, and $h_1, \dots, h_K$ be as
  defined in Algorithm~\ref{alg:efficient}. Let $E \subset \configs$ be any set
  of outcomes and let $\hat \mu_i$ denote the output distribution of
  $\samplealg(h_i,\beta, \delta'/(2K))$. We have
  \[
    \hat \mu(E)
    = \prob(\hat{\vec{\rho}} \in E)
    = \sum_{i=1}^K \prob(\hat{\vec{\rho}} \in E | \hat{\vec{\rho}} \in \configs_i) \prob(\hat{\vec{\rho}} \in \configs_i)
    = \sum_{i=1}^K \hat \mu_i(E) \cdot \frac{\hat Z_i}{\sum_j \hat Z_j}.
  \]
  Using the guarantees on $\intalg$ and $\samplealg$ and
  Lemma~\ref{lem:reldistmult}, it follows that
  \[
    \hat \mu(E) \leq \sum_{i=1}^K e^\beta \mu_i(E) e^{2\alpha} \frac{Z_i}{\sum_j Z_j} = e^{\eta} \mu(E),
  \]
  where $Z_i = \int_{\configs_i} f_\mu$ and $\mu_i$ is the distribution with density
  proportional to $\vec\rho \mapsto \ind{\vec\rho \in \configs_i}\cdot
  h(\vec\rho)$. Similarly, we have that $\hat \mu(E) \geq e^{-\eta} \mu(E)$. By Lemma~\ref{lem:reldistmult} it follows that $\reldist(\hat
  \mu, \mu) \leq \eta$.
\end{proof}

\section{Proofs for online learning (Section~\ref{sec:online})}\label{app:online}
In our regret bounds and utility guarantees for differentially private
optimization, we assume that the ball of radius $w$ centered at an optimal point
$\vec{\rho^*}$ is  contained in the parameter space $\configs$.
Lemma~\ref{lem:interiortransformation} shows that when $\configs$ is convex, we
can transform the problem so that this condition is satisfied, at the cost of
doubling the radius of $\configs$.

\begin{lem}\label{lem:interiortransformation}
	Let $\configs \subset \reals^d$ be a convex parameter space contained in a
	ball of radius $R$ and let $u_1, \dots, u_\numfunctions : \configs \to [0,H]$
	be any piecewise $L$-Lipschitz and $(w,k)$-dispersed utility functions. There
	exists an enlarged parameter space $\configs' \supset \configs$ contained in a
	ball of radius $2R$ and extended utility functions $q_1, \dots,
	q_\numfunctions : \configs' \to [0,H]$ such that:
	\begin{enumerate}
	\item Any maximizer of $\sum_t q_t$ can be transformed into a maximizer for
	$\sum_t u_t$ by projecting onto $\configs$.
	\item The functions $q_1, \dots, q_t$ are piecewise $L$-Lipschitz and
	$(w,k)$-dispersed.
	\item There exists an optimizer $\vec{\rho}^* \in \argmax_{\vec{\rho} \in
	\configs'} \sum_t q_t(\vec{\rho})$ such that $B(\vec{\rho}^*, R)\subset
	\configs'$.
	\end{enumerate}
\end{lem}
\begin{proof}
	For any $\vec{\rho} \in \reals^d$, let $\configs(\vec{\rho}) =
	\argmin_{\vec{\rho}' \in \configs} \norm{\vec{\rho} - \vec{\rho}'}_2$ denote
	the Euclidean projection of $\vec{\rho}$ onto $\configs$. Define $\configs' =
	\{\vec{\rho} \in \reals^d \,:\, \norm{\vec{\rho} - \configs(\vec{\rho})}_2
	\leq R\}$ to be the set of points within distance $R$ of $\configs$, and let
	$q_t : \configs' \to [0,H]$ be given by $q_t(\vec{\rho}) =
	u_t(\configs(\vec{\rho}))$ for $t \in [\numfunctions]$. Since $\configs$ is
	contained in a ball of radius $R$ and every point in $\configs'$ is within
	distance $R$ of $\configs$, it follows that $\configs'$ is contained in a ball
	of radius $2R$.

	\vspace{1em}\noindent\textit{Part 1.} Let $\vec{\rho}^* \in
	\argmax_{\vec{\rho} \in \configs'} \sum_{t=1}^\numfunctions q_t(\vec{\rho})$
	be any maximizer of $\sum_t q_t$. We need to show that
	$\configs(\vec{\rho}^*)$ is a maximizer of $\sum_t u_t$. First, since for any
	$\vec{\rho} \in \configs'$ we have $q_t(\vec{\rho}) =
	u_t(\configs(\vec{\rho}))$, it follows that $\max_{\vec{\rho} \in \configs'}
	\sum_{t=1}^T q_t(\vec{\rho}) = \max_{\vec{\rho} \in \configs} \sum_{t=1}^T
	u_t(\vec{\rho})$ (i.e., the maximum value attained by $\sum_t q_t$ over
	$\configs'$ is equal to the maximum value attained by $\sum_t u_t$ over
	$\configs$). Since $\vec{\rho}^*$ is a maximizer of $\sum_t q_t$, we have
	$\max_{\vec{\rho} \in \configs} \sum_{t=1}^\numfunctions u_t(\vec{\rho}) =
	\sum_{t=1}^\numfunctions q_t(\vec{\rho}^*) = \sum_{t=1}^\numfunctions
	u_t(\configs(\vec{\rho}^*))$ and it follows that $\configs(\vec{\rho}^*)$ is a
	maximizer for $\sum_t u_t$.

	\vspace{1em}\noindent\textit{Part 2.} Next, we show that each function $q_t$
	is piecewise $L$-Lipschitz. Let $\configs_1, \dots, \configs_N$ be the
	partition of $\configs$ such that $u_t$ is $L$-Lipschitz on each piece, and
	define $\configs'_1, \dots, \configs'_N$ by $\configs'_i = \{ \vec{\rho} \in
	\configs' \,:\, \configs(\vec{\rho}) \in \configs_i\}$ for each $i \in [N]$.
	We will show that $q_t$ is $L$-Lipschitz on each set $\configs'_i$. To see
	this, we use the fact that projections onto convex sets are contractions
	(i.e., $\norm{\vec{\rho} - \vec{\rho}'}_2 \geq \norm{\configs(\vec{\rho}) -
	\configs(\vec{\rho}')}_2$). From this it follows that for any $\vec{\rho},
	\vec{\rho}' \in \configs'_i$ we have
	\[
	|q_t(\vec{\rho}) - q_t(\vec{\rho}')|
	= |u_t(\configs(\vec{\rho})) - u_t(\configs(\vec{\rho}'))|
	\leq L \cdot \norm{\configs(\vec{\rho}) - \configs(\vec{\rho}')}_2
	\leq L \cdot \norm{\vec{\rho} - \vec{\rho}'}_2,
	\]
	where the first inequality follows from the fact that $\configs(\vec{\rho})$
	and $\configs(\vec{\rho}')$ belong to $\configs_i$ and $u_t$ is $L$-Lipschitz
	on $\configs_i$.

	Next, we show that $q_1, \dots, q_\numfunctions$ are $(w,k)$-dispersed. Fix
	any function index $t$, let $B = B(\vec{\rho}, w)$ be any ball of radius $w$
	and suppose that $B$ is split by the partition $\configs'_1, \dots,
	\configs'_N$ of $\configs'$ defined above for which $q_t$ is piecewise
	Lipschitz. This implies that we can find two points $\vec{\rho}_1$ and
	$\vec{\rho}_2$ in $B$ such that (after possibly renaming the partitions) we
	have $\vec{\rho}_1 \in \configs'_1$ and $\vec{\rho}_2 \in \configs'_2$. By
	definition of the sets $\configs'_i$, it follows that $\configs(\vec{\rho}_1)
	\in \configs_1$ and $\configs(\vec{\rho}_2) \in \configs_2$. Moreover, since
	projections onto convex sets are contractions, we have that
	$\configs(\vec{\rho}_1)$ and $\configs(\vec{\rho}_2)$ are both contained in
	$B(\configs(\vec{\rho}), w)$. Therefore, the ball $B(\configs(\vec{\rho}),w)$
	is split by the partition $\configs_1, \dots, \configs_\numfunctions$ of
	$\configs$ on which $u_t$ is piecewise $L$-Lipschitz. It follows that if no
	ball of radius $w$ is split by more than $k$ of the piecewise Lipschitz
	partitions for the functions $u_1, \dots, u_\numfunctions$, then the same is
	true for $q_1, \dots, q_\numfunctions$.

	\vspace{1em}\noindent\textit{Part 3.} Finally, let $\vec{\rho}^* \in
	\argmax_{\vec{\rho} \in \configs} \sum_t u_t(\vec{\rho})$. This point is also
	a maximizer for $\sum_t q_t$, and is contained in the $R$-interior of
	$\configs'$.
\end{proof}

We now turn to proving our main result for online piecewise Lipschitz
optimization in the full information setting.

\begin{algorithm}
\caption{Online learning algorithm for single-dimensional piecewise functions}\label{alg:1_d_online}
\begin{algorithmic}[1]
\Require $\lambda\in (0, 1/H]$
\State Set $u_0(\cdot) = 0$ to be the constant 0 function over $\configs$.
\For{$t = 1, 2, \dots, T$}
	\State Obtain a point $\rho_t$ using Algorithm~\ref{alg:1dEfficient} with $g=\lambda \sum_{s = 0}^{t-1} u_s$. (The point $\rho_t$ is sampled with probability proportional to $e^{g(\rho_t)}$.) \label{step:1_d_sample}
	\State Observe the the function $u_t(\cdot)$ and receive payoff $u_t(\rho_t)$.
\EndFor
\end{algorithmic}
\end{algorithm}
\begin{algorithm}
\caption{Online learning algorithm for multi-dimensional piecewise concave functions}\label{alg:multi_d_online}
\begin{algorithmic}[1]
\Require $\lambda \in (0, 1/H]$, $\eta, \zeta \in (0,1)$.
\State Set $u_0(\cdot) = 0$ to be the constant 0 function over $\configs$.
\For{$t = 1, 2, \dots, T$}
	\State Obtain a vector $\vec{\rho}_t$ using Algorithm~\ref{alg:efficient} with $g=\lambda \sum_{s = 0}^{t-1} u_s$, approximation parameter $\eta/4$, and confidence parameter $\zeta/T$. (The vector $\vec{\rho}_t$ is sampled with probability that is approximately proportional to $e^{g(\vec{\rho}_t)}$.) \label{step:multi_d_sample}
	\State Observe the function $u_t(\cdot)$ and receive payoff $u_t(\vec{\rho}_t)$.
\EndFor
\end{algorithmic}
\end{algorithm}

\OnedOnline*

\begin{proof}
Define $u_0(\vec{\rho}) = 0$ and $U_t(\vec{\rho}) = \sum_{s = 0}^{t-1}
u_s(\vec{\rho})$ for each $t \in [\numfunctions]$. Let $W_t = \int_\configs
\exp(\lambda U_t(\vec{\rho})) \, d\vec{\rho}$ be the normalizing constant at
round $t$ and let $P_t = \E_{\vec{\rho} \sim p_t}[u_t(\vec{\rho})]$ denote the
expected payoff achieved by the algorithm in round $t$, where the expectation is
only with respect to sampling $\vec{\rho}_t$ from $p_t$. Also, let $P(\alg) =
\sum_{i = 1}^T P_t$ be the expected payoff of the algorithm (with respect to its
random choices). We begin by upper bounding $W_{t+1}/W_t$ by
$\exp\left(\left(e^\lambda - 1\right)P_t\right)$.

\begin{align*}
\frac{W_{t+1}}{W_t} &= \frac{\int_\configs \exp(\lambda U_{t+1}(\vec{\rho})) \, d\vec{\rho}}{\int_\configs \exp(\lambda U_t(\vec{\rho})) \, d\vec{\rho}}\\
&= \frac{\int_\configs \exp(\lambda U_{t}(\vec{\rho}))\cdot  \exp(\lambda u_{t+1}(\vec{\rho})) \, d\vec{\rho}}{\int_\configs \exp(\lambda U_t(\vec{\rho})) \, d\vec{\rho}} &(U_{t+1} = U_t + u_t)\\
&= \int_\configs p_t(\vec{\rho}) \exp(\lambda u_{t+1}(\vec{\rho})) \, d\vec{\rho} &(\text{By definition of }p_t)\\
&\leq \int_\configs p_t(\vec{\rho}) \left(1 + (e^{H\lambda} - 1)\frac{u_t(x)}{H}\right) \, d\vec{\rho} &(\text{For }z \in [0,1], e^{\lambda z} \leq 1 + (e^\lambda - 1)z)\\
&\leq 1 + (e^{H\lambda} - 1)\frac{P_t}{H} \leq \exp\left((e^{H\lambda} - 1)\frac{P_t}{H}\right) &(1 + z \leq e^z).
\end{align*}

Therefore, \begin{equation}\frac{W_{T+1}}{W_1} \leq \exp\left(\frac{e^{H\lambda} - 1}{H}\sum_{i = 1}^T P_t\right) = \exp\left(\frac{P(\mathcal{A})\left(e^{H\lambda} - 1\right)}{H}\right).\label{eq:online_upper}\end{equation}

We now lower bound $W_{T+1}/W_1$. To do this, let $\vec{\rho}^*$ be the optimal parameter and let $\OPT = U_{T+1}(\vec{\rho}^*)$. Also, let $\mathcal{B}^*$ be the ball of radius $w$ around $\vec{\rho}^*$. From $(w,k)$-dispersion, we know that for all $\vec{\rho} \in \mathcal{B}^*$, $U_{T+1}(\vec{\rho}) \geq \OPT - Hk - LTw$. Therefore, \begin{align*}
W_{T+1} &= \int_\configs \exp(\lambda U_{T+1}(\vec{\rho})) \, d\vec{\rho}\\
&\geq \int_{\mathcal{B}^*} \exp(\lambda U_{T+1}(\vec{\rho})) \, d\vec{\rho}\\
&\geq \int_{\mathcal{B}^*} \exp(\lambda (\OPT - Hk - LTw))d\vec{\rho}\\
&\geq \vol(B(\vec{\rho}^*, w))\exp(\lambda (\OPT - Hk - LTw)).
\end{align*}
Moreover, $W_1 = \int_{\configs} \exp(\lambda U_1(\vec{\rho})) \, d\vec{\rho} \leq \vol(B(\vec{0}, R))$. Therefore, \[\frac{W_{T+1}}{W_1} \geq \frac{\vol(B(\vec{\rho}^*, w))}{\vol(B(\vec{0}, R))} \exp(\lambda (\OPT - Hk - LTw)).\] The volume ratio is equal to
  $(w/R)^d$, since the volume of a ball of radius $r$ in $\R^d$ is proportional
  to $r^d$. Therefore, \begin{equation}\frac{W_{T+1}}{W_1} \geq \left(\frac{w}{R}\right)^d \exp(\lambda (\OPT - Hk - LTw)).\label{eq:online_lower}\end{equation}

  Combining Equations~\ref{eq:online_upper} and \ref{eq:online_lower}, taking the log, and rearranging terms, we have that
  \[
  \OPT \leq \frac{P(\alg)(e^{H\lambda} - 1)}{H\lambda} + \frac{d \ln (R/w)}{\lambda} + Hk + LTw.
  \]
  We subtract $P(\alg)$ on either side have that
  \[\OPT - P(\alg) \leq \frac{P(\alg)(e^{H\lambda} - 1 - H\lambda)}{H\lambda} + \frac{d \ln (R/w)}{\lambda} + Hk + LTw.
  \]
  We use the fact that for $z \in [0,1]$, $e^z \leq 1 + z + (e-2) z^2$ and the
  that $P(\alg) \leq HT$ to conclude that
  \[
  \OPT - P(\alg) \leq H^2T\lambda + \frac{d \ln (R/w)}{\lambda} + Hk + LTw.
  \]
	The analysis of the efficient multi-dimensional algorithm that uses approximate
	sampling is given in Theorem~\ref{thm:multi_d_online}.
\end{proof}

Next, we argue that the dependence on the Lipschitz constant can be made
logarithmic by tuning the parameter $w$ exploiting the fact that any functions
that are $(w,k)$-dispersed are also $(w',k)$-dispersed for $w' \leq w$.

\begin{cor}\label{cor:1_d_online}
  Let $u_1, \dots, u_T$ be the functions observed by
  Algorithm~\ref{alg:1_d_online} and suppose they satisfy the conditions of
  Theorem~\ref{thm:1_d_online}. Suppose $T \geq 1/(Lw)$. Setting $\lambda =
  1/(H\sqrt{T})$, the regret of Algorithm~\ref{alg:1_d_online} is bounded by
  $H\sqrt{T}\left(1 + d \ln (RNL) \right) + Hk + 1.$
\end{cor}

 \begin{proof}
   This bound follows from applying Theorem~\ref{thm:1_d_online}
   using the $(w',k)$-disperse critical boundaries condition with
   $w' = 1/(LT)$. The lower bound on requirement
   on $T$ ensures that $w' \leq w$.
 \end{proof}

 Lemma~\ref{lem:smooth2disperse} shows that when the sequence of functions
 $u_1$, \dots, $u_\numfunctions$ are chosen by a smoothed adversary in the sense
 of \citet{Cohen-Addad17:Online} then the set of functions is $(w,k)$-dispersed
 with non-trivial parameters.

 \begin{lem} \label{lem:smooth2disperse}
   Let $u_1, \dots, u_\numfunctions$ be a sequence of functions chosen by  a
   $\kappa$-smoothed adversary. That is, each function $u_t$ has at most $\tau$
   discontinuities, each drawn independently from a potentially different
   $\kappa$-bounded density. For any $\alpha \geq 1/2$, with probability at
   least $1-\zeta$ the functions $u_1, \dots, u_\numfunctions$ are
   $(w,k)$-dispersed with  $w = \frac{1}{\kappa (\numfunctions
   \tau)^{1-\alpha}}$ and $k = O( (\numfunctions\tau)^\alpha \sqrt{\log
   1/\zeta})$.
 \end{lem}
 \begin{proof}
   There are a total of $\numfunctions\tau$ discontinuities from the
   $\numfunctions$ functions, each independently drawn from a $\kappa$-bounded
   density. Applying the first part of Lemma~\ref{lem:dispersion} guarantees
   that with high probability, every interval of width $w$ contains at most
   $O(\numfunctions\tau\kappa w + \sqrt{\numfunctions \tau \log(1/\zeta)})$
   discontinuities. Setting $w = \frac{1}{\kappa (\numfunctions
   \tau)^{1-\alpha}}$ completes the proof.
 \end{proof}

\subsection{Bandit Online Optimization}

Our algorithm for online learning under bandit feedback requires that we
construct a $w$-net for the parameter space $\configs$. The following Lemma
shows that there exists a $w$-net for any set contained in a ball of radius $R$
of size $(3R/w)^d$. This is a standard result, but we include the proof for
completeness.
\begin{lem}\label{lem:netSize}
	Let $\configs \subset \reals^d$ be contained in a ball of radius $R$. Then
	there exists a subset $\hat \configs_w \subset \configs$ such that $|\hat
	\configs_w| \leq (3R/w)^d$ and for every $\rho \in \configs$ there exists
	$\hat \rho \in \hat \configs_w$ such that $\norm{\rho - \hat \rho}_2 \leq w$.
\end{lem}
\begin{proof}
	Consider the following greedy procedure for constructing $\hat \configs_w$:
	while there exists any point in $\configs$ further than distance $w$ from
	$\hat \configs_w$, pick any such point and it to the $\hat \configs_w$.
	Suppose this greedy procedure has added points $rho_1, \dots, \rho_n$ to the
	covering so far. We will argue that the algorithm must terminate with $n \leq
	(3R/w)^d$.

	By construction, we know that the distance between any $\rho_i$ and $\rho_j$
	is at least $w$, which implies that the balls $B(\rho_1, w/2)$, \dots,
	$B(\rho_n, w/2)$ are all disjoint. Moreover, since their centers are contained
	$\configs$ which is contained in a ball of radius $R$, we are guaranteed that
	the balls of radius $w/2$ centered on $\rho_1$, \dots, $\rho_n$ are also
	contained in a ball of radius $R+w/2$. Therefore, we have
	$\vol(\bigcup_{i=1}^n B(\rho_i, w/2)) \leq \vol(B(0,R+w/2))$. Since the balls
	$B(\rho_i, w/2)$ are disjoint, we have $\vol(\bigcup_{i=1}^n B(\rho_i, w/2)) =
	\sum_{i=1}^n \vol(B(\rho_i, w/2)) = n (w/2)^d v_d$, where $v_d$ is the volume
	of the unit ball in $d$ dimensions. Similarly, $\vol(B(0,R+w/2)) = (R+w/2)^d
	v_d$. Therefore, we have $n \leq \bigl(\frac{2(R+w/2)}{w}\bigr)^d \leq
	\bigl(\frac{3R}{w}\bigr)^d$, where the last inequality follows from the fact
	that $w < R$.
\end{proof}

\subsection{Approximate sampling for online learning}

\begin{restatable}{thm}{multiDOnline}\label{thm:multi_d_online} Let $u_1, \dots,
u_T : \configs \to [0,H]$ be the sequence of functions observed by
Algorithm~\ref{alg:multi_d_online}. Suppose that each $u_t$ is piecewise
$L$-Lipschitz and concave on convex pieces. Moreover, suppose that $u_1, \dots,
u_T$ are $(w,k)$-disperse, $\configs \subset \reals^d$ is convex and contained
in a ball of radius $R$, and that for some $\vec{\rho^*} \in \argmax_{\vec{\rho}
\in \configs} \sum_{t=1}^T u_t(\vec{\rho})$ we have $B(\vec{\rho^*}, w) \subset
\configs$. Then for any $\eta, \zeta \in (0,1)$, the expected regret of
Algorithm~\ref{alg:multi_d_online} with $\lambda = \sqrt{d\ln(R/w)/\numfunctions}/H$ is bounded by
\[
O(H(\sqrt{Td \ln(R/w)} + k) + TLw + \eta HT + \zeta HT).
\]
Moreover, suppose there are $K$ intervals partitioning $\configs$ so that
$\sum_{t=1}^T u_t$ is piecewise $L$-Lipschitz on each region. Also, suppose that
we use the integration algorithm of \citet{Vempala06:logconcave} and the
sampling algorithm of \citet{Bassily14:ERM} to implement
Algorithm~\ref{alg:efficient}. The running time of
Algorithm~\ref{alg:multi_d_online} is
\[
T\left(K\cdot\operatorname{poly}\left(d,
\frac{1}{\eta}, \log \frac{TK}{\zeta}, \log \frac{R}{r}\right) +
\operatorname{poly}\left(d, L, \frac{1}{\eta}, \log \frac{R}{r}\right)\right).
\]
\end{restatable}
\begin{proof}
On each round we use Algorithm~\ref{alg:efficient} to approximately sample a
point from the distribution proportional to $g_t(\vec{\rho}) = \exp(\lambda
\sum_{t=1}^\numfunctions u_t(\vec{\rho}))$. Each invocation of
Algorithm~\ref{alg:efficient} has failure probability $\zeta' = \zeta/T$, which
implies that with probability at least $1-\zeta$ the sampler succeeds on every
round. Assume this high probability event holds for the remainder of the proof.
In this case, Lemma~\ref{lem:efficientApprox} guarantees that if $\hat \mu_t$ is
the output distribution of Algorithm~\ref{alg:efficient} oun round $t$ and
$\mu_t$ is the distribution with density proportional to $g_t$, then we have
$\reldist(\hat \mu_t, \mu_t) \leq \eta$.

Next, we show that the expected utility per round of the approximate sampler is
at most a $(1-\eta)$ factor smaller than the expected utility per round of the
exact sampler. Let $\hat{\vec{\rho}}_t \sim \hat \mu_t$ and $\vec{\rho}_t \sim
\mu_t$ be samples drawn from the approximate and exact samplers at round $t$,
respectively. Then we have
\[
  \expect[u_t(\hat{\vec{\rho}}_t)]
  = \int_0^\infty \prob(u_t(\hat{\vec{\rho}}_t) \geq \tau) \, d\tau
  \geq e^{-\eta} \int_0^\infty \prob(u_t(\vec{\rho}_t) \geq \tau) \, d\tau
  = e^{-\eta} \cdot \expect[u_t(\vec{\rho}_t)]
  \geq (1-\eta) \cdot \expect[u_t(\vec{\rho}_t)].
\]
where the first inequality follows from Lemma~\ref{lem:reldistmult} (i.e., since
$\reldist(\hat \mu_t, \mu_t) \leq \eta$, we know that the probability mass of
any event under $\hat{\mu}_t$ is at least $e^{-\eta}$ of its mass under
$\mu_t$). Using this, we can bound the excess regret suffered by the approximate
sampler compared to the exact sampling algorithm:
\begin{align*}
\expect\biggl[
  \sum_{t=1}^\numfunctions u_t(\vec{\rho}_t) - u_t(\hat{\vec{\rho}}_t)
\biggr]
\leq
\expect\biggl[
  \sum_{t=1}^\numfunctions u_t(\vec{\rho}_t)
\biggr]
-
(1-\eta)
\cdot
\expect\biggl[
  \sum_{t=1}^\numfunctions u_t(\vec{\rho}_t)
\biggr]
=
\eta
\cdot
\expect\biggl[
  \sum_{t=1}^\numfunctions u_t(\vec{\rho}_t)
\biggr]
\leq \eta H\numfunctions.
\end{align*}
Combining this with the regret bound for the exact sampling algorithm gives a
regret bound of
\[
  H^2 T \lambda + \frac{d \ln(R/W)}{\lambda} + Hk + \numfunctions Lw + \eta H \numfunctions + \zeta H\numfunctions.
\]
where the $\zeta H\numfunctions$ term comes from the $\zeta$-probability event
that at least one invocation of the approximate sampler fails, in which case the
maximum possible regret is $H\numfunctions$. Setting $\eta = \zeta =
1/\sqrt{\numfunctions}$ and $\lambda$ as in Theorem~\ref{thm:1_d_online} gives a
regret bound of
\[
  O(H(\sqrt{\numfunctions d\log(R/w)} + k) + \numfunctions Lw).
\]
\end{proof}

\subsection{Lower bound for single-dimensional parameter spaces}\label{app:online_lb_single}
We will use the following adversarial construction to prove our lower bound.

\begin{lem}[\citet{Weed16:Online}]\label{lem:Weed}
Define the two functions $u^{(0)}: [0,1] \to [0,1]$ and $u^{(1)}: [0,1] \to [0,1]$ such that \[u^{(0)}(\rho) = \begin{cases}
\frac{1}{2} &\text{if } \rho < \frac{1}{2}\\
0 &\text{if } \rho \geq \frac{1}{2}\end{cases} \text{ and } u^{(1)}(\rho) = \begin{cases}
\frac{1}{2} &\text{if } \rho < \frac{1}{2}\\
1 &\text{if } \rho \geq \frac{1}{2}.\end{cases}\] There exists a pair of adversaries $U$ and $L$ defining two distributions $\mu_U$ and $\mu_L$ over $\left\{u^{(0)}, u^{(1)}\right\}$ such that for any learning algorithm, \[\max_{A \in \{U,L\}} \max_{\rho \in [0,1]} \E \left[ \sum_{t = 1}^T u_t(\rho) - \sum_{t = 1}^T u_t(\rho_t)\right] \geq \frac{1}{32} \sqrt{T},\] where the expectation is over $u_1, \dots, u_{\numfunctions} \sim \mu_A$ and the random choices $\rho_1, \dots, \rho_{\numfunctions}$ of the algorithm. Moreover, under adversary $U$, any parameter $\rho \geq \frac{1}{2}$ is optimal and under adversary $L$, any parameter $\rho < \frac{1}{2}$ is optimal.
\end{lem}

Specifically, the adversary $U$ defined by \citet{Weed16:Online} selects the function $u^{(0)}$ with probability $\frac{1}{2} - \frac{1}{8\sqrt{T}}$ and $u^{(1)}$ with probability $\frac{1}{2} + \frac{1}{8\sqrt{T}}$. Meanwhile, the adversary $L$ selects the function $u^{(0)}$ with probability $\frac{1}{2} + \frac{1}{8\sqrt{T}}$ and $u^{(1)}$ with probability $\frac{1}{2} - \frac{1}{8\sqrt{T}}$.
The theorem's proof follows from standard information
theoretic techniques for lower bounds (e.g., \citet{Tsybakov09:Introduction}).

Weed et al.~\cite{Weed16:Online} study the specific problem of learning to bid in an online setting. A single item is sold at each round. The learner is a potential buyer, and he does not know his value for the item at any given round. The seller sells each item in a second-price auction. The other buyers' values may be adversarially selected. If the buyer wins the item, he learns his value, but if he does not win the item, he learns nothing about his value at that round. Thus, the buyer must learn to bid without knowing his value. Weed et al.~\cite{Weed16:Online} prove that the buyer's optimization problem amounts to the online optimization of threshold functions with a specific structure. They do not develop a general theory of dispersion, but we can map their analysis into our setting. In essence, they prove that if these threshold functions are $(w, 0)$-dispersed at the maximizer, then the adversary's regret is bounded by $O\left(\sqrt{T \log \frac{1}{w}}\right)$. They use Lemma~\ref{lem:Weed} to prove a matching lower bound.

\begin{theorem}\label{thm:online_LB_single_dim}
For any learning algorithm and $T \geq 3$, there is a sequence $u_1, \dots, u_T$ of piecewise constant functions mapping $[0,1]$ to $[0,1]$ such that if \[D = \{(w,k) : \{u_1, \dots, u_T\} \text{ is } (w,k)\text{-dispersed at the maximizer}\},\] then \[\max_{\rho \in [0,1]}\E\left[\sum_{t = 1}^T u_t(\rho) - u_t\left(\rho_t\right)\right] = \Omega\left(\inf_{(w,k) \in D}\left\{ \sqrt{T\log \frac{1}{w}} + k \right\}\right).\]
\end{theorem}

\begin{proof}
We begin with an outline of the proof. For the first $\numfunctions - \sqrt{\numfunctions}$ rounds, our adversary behaves exactly like the worse of the two adversaries defined in Lemma~\ref{lem:Weed}, playing threshold functions at each round. Each threshold function has a discontinuity at $\rho = \frac{1}{2}$. Since these functions are piecewise constant, either $\frac{1}{4}$ or $\frac{3}{4}$ maximizes the sum $\sum_{t = 1}^{\numfunctions - \sqrt{\numfunctions}} u_t$. Denoting this maximizer as $\rho^*$, our adversary then plays $\sqrt{T}$ copies of the indicator function corresponding to the interval $\left[\rho^* - 2^{-T}, \rho^* + 2^{-T}\right]$. At the end of all $T$ rounds, $\rho^*$ maximizes the sum $\sum_{t = 1}^{\numfunctions} u_t$. We prove that the expected regret incurred by this adversary is at least $\frac{\sqrt{T}}{64}$, which follows from Lemma~\ref{lem:Weed}.  In order to prove the theorem, we need to show that $\frac{\sqrt{T}}{64} = \Omega\left(\inf_{(w,k) \in D}\left\{ \sqrt{T\log \frac{1}{w}} + k \right\}\right)$. Therefore, we need to show that the set of functions played by the adversary is $(w,k)$-dispersed at the maximizer $\rho^*$ for $w = \Theta(1)$ and $k = O\left(\sqrt{\numfunctions}\right).$ The reason this is true is that the only functions with discontinuities in the interval $\left[\rho^* - \frac{1}{8}, \rho^* + \frac{1}{8}\right]$ are the final $\sqrt{T}$ functions played by the adversary. Thus, the theorem statement holds.

\bigskip\noindent\emph{Regret lower bound.}
Fix the learning algorithm. We begin be demonstrating the existence of a sequence of functions inducing a regret lower bound of $\Omega\left(\sqrt{T}\right)$.

\begin{claim}\label{claim:reg_lower_bnd_1d}
Let $T' = \left\lfloor T - \sqrt{T} \right\rfloor$. There is a sequence $u_1, \dots, u_{T'}$ of piecewise constant functions mapping $[0,1]$ to $[0,1]$ such that:
\begin{enumerate}
\item The expected regret is lower bounded as follows: $\max_{\rho \in [0,1]}\E\left[\sum_{t = 1}^{T'} u_t(\rho) - u_t\left(\rho_t\right)\right] \geq \frac{\sqrt{T}}{64},$ where the expectation is over the random choices $\rho_1, \dots, \rho_{T'}$ of the learner.
\item Each function $u_t$ is a threshold function with a discontinuity at $\frac{1}{2}$.
\item Either $\left[0, \frac{1}{2}\right] = \argmax_{\rho \in [0,1]} \sum_{t = 1}^{T'} u_t(\rho)$ or $\left(\frac{1}{2}, 1\right] = \argmax_{\rho \in [0,1]} \sum_{t = 1}^{T'} u_t(\rho)$.
\end{enumerate}
\end{claim}

\begin{proof}[Proof of Claim~\ref{claim:reg_lower_bnd_1d}]
By Lemma~\ref{lem:Weed}, there exists a randomized adversary such that \[\max_{\rho \in [0,1]} \E\left[\sum_{t = 1}^{T'} u_t(\rho) - u_t\left(\rho_t\right)\right] \geq \frac{1}{32} \sqrt{T'} = \frac{1}{32} \sqrt{\left\lfloor T - \sqrt{T} \right\rfloor} \geq \frac{1}{32} \sqrt{\frac{T - \sqrt{T}}{2}} \geq \frac{\sqrt{T}}{64},\] where the expectation is over the random sequence $u_1, \dots, u_{T'}$ of functions chosen by the adversary and the random choices $\rho_1, \dots, \rho_{T'}$ of the learner. Since this inequality holds in expectation over the adversary's choices, there must be a sequence $u_1, \dots, u_{T'}$ of functions such that \[\max_{\rho \in [0,1]} \E\left[\sum_{t = 1}^{T'} u_t(\rho) - u_t\left(\rho_t\right)\right] \geq \frac{\sqrt{T}}{64},\] where the expectation is only over the random choices $\rho_1, \dots, \rho_{T'}$ of the learner. Therefore, the first part of the claim holds. By Lemma~\ref{lem:Weed}, we know that each function is piecewise constant with a discontinuity at $\frac{1}{2}$, so the second part of the claim holds. Finally, Lemma~\ref{lem:Weed} guarantees that either every parameter in $[0, 1/2]$ is optimal, or every parameter in $(1/2, 1]$ is optimal, so the third part of the claim holds.
\end{proof}

\bigskip\noindent\emph{Construction of the final $\sqrt{T}$ functions.}
From the previous claim, we know that either $\left[0, \frac{1}{2}\right] = \argmax_{\rho \in [0,1]} \sum_{t = 1}^{T'} u_t(\rho)$ or $\left(\frac{1}{2}, 1\right] = \argmax_{\rho \in [0,1]} \sum_{t = 1}^{T'} u_t(\rho)$. We define the parameter $\rho^* \in \left\{\frac{1}{4}, \frac{3}{4}\right\}$ such that $\rho^* = \frac{1}{4}$ in the former case, and $\rho^* = \frac{3}{4}$ in the latter case. Under this definition, $\rho^*$ maximizes the sum $\sum_{t = 1}^{T'} u_t$.
We now define the functions $u_{T' + 1}, \dots, u_{T}$ to all be equal to the function $\rho \mapsto \textbf{1}_{\left\{\rho \in \left[\rho^* - 2^{-T}, \rho^* + 2^{-T}\right]\right\}}.$
Under this definition, the parameter $\rho^*$ remains a maximizer of the sum $\sum_{t = 1}^{T} u_t$.

In our final regret bound, we will use the following property of the functions $u_{T' + 1}, \dots, u_{T}$.
\begin{claim}\label{claim:regret_nonzero_1d}
For any parameters $\rho_{T'+1}, \dots, \rho_T$,
$\sum_{t = T'+1}^T u_t\left(\rho^*\right) - u_t\left(\rho_t\right) \geq 0.$
\end{claim}
\begin{proof}[Proof of Claim~\ref{claim:regret_nonzero_1d}]
By definition, $\sum_{t = T'+1}^T u_t\left(\rho^*\right) = T - T' + 1$. Since the range of each function $u_t$ is contained in $[0,1]$, for any parameters $\rho_{T'+1}, \dots, \rho_T$, $\sum_{t = T'+1}^T u_t\left(\rho_t\right) \leq T - T' + 1$. Therefore, the claim holds.
\end{proof}

\bigskip\noindent\emph{Dispersion parameters.}
We now prove that the only functions with discontinuities in the interval $\left[\rho^* - \frac{1}{8}, \rho^* + \frac{1}{8}\right]$ are the functions $u_{T' + 1}, \dots, u_{T}$.
Since $T \geq 3$, if $\rho^* = \frac{1}{4}$, then $\left[\rho^* - 2^{-T}, \rho^* + 2^{-T}\right] \subseteq \left[\rho^* - \frac{1}{8}, \rho^* + \frac{1}{8}\right] \subset \left[0, \frac{1}{2}\right]$ and $\rho^* = \frac{3}{4}$, then $\left[\rho^* - 2^{-T}, \rho^* + 2^{-T}\right] \subseteq \left[\rho^* - \frac{1}{8}, \rho^* + \frac{1}{8}\right] \subset \left(\frac{1}{2}, 1\right]$. Since the discontinuities of the functions $u_1, \dots, u_{T'}$ only fall at $\frac{1}{2}$, this means that the interval $\left[\rho^* - \frac{1}{8}, \rho^* + \frac{1}{8}\right]$ only contains the discontinuities of the functions $u_{T' + 1}, \dots, u_{T}$. Since $T - T' = T - \left\lfloor T - \sqrt{T} \right\rfloor \leq T - \left(T - \sqrt{T}-1\right) = \sqrt{T} +1$, the set $\left\{u_1, \dots, u_T\right\}$ is $\left(\frac{1}{8}, \sqrt{T} + 1\right)$-dispersed at the maximizer $\rho^*$. Therefore, \begin{align*}&\inf_{(w,k) \in D} \left\{\sqrt{T \log \frac{1}{w}} + k \right\}\\
&\leq \sqrt{T \log 8} + \sqrt{T} + 1\\
&\leq 4\sqrt{T} + 0\\
&\leq 256\max_{\rho \in [0,1]} \E\left[\sum_{t = 1}^{T'} u_t(\rho) - u_t\left(\rho_t\right)\right] + \E\left[\sum_{t = T'+1}^{T} u_t\left(\rho^*\right) - u_t\left(\rho_t\right)\right] &\text{(Claims~\ref{claim:reg_lower_bnd_1d} and \ref{claim:regret_nonzero_1d})}\\
&= 256\E\left[\sum_{t = 1}^{T'} u_t\left(\rho^*\right) - u_t\left(\rho_t\right)\right] + \E\left[\sum_{t = T'+1}^{T} u_t\left(\rho^*\right) - u_t\left(\rho_t\right)\right] &\left(\rho^* \in \argmax_{\rho \in [0,1]}\sum_{t = 1}^{T'} u_t(\rho)\right)\\
&\leq 256\E\left[\sum_{t = 1}^{T} u_t\left(\rho^*\right) - u_t\left(\rho_t\right)\right]\\
&\leq 256\max_{\rho \in [0,1]}\E\left[\sum_{t = 1}^{T} u_t\left(\rho\right) - u_t\left(\rho_t\right)\right].\\\end{align*} Therefore, \[\max_{\rho \in [0,1]} \E\left[\sum_{t = 1}^{T} u_t(\rho) - u_t\left(\rho_t\right)\right] = \Omega\left(\inf_{(w,k) \in D} \left\{\sqrt{T \log \frac{1}{w}} + k \right\}\right),\] as claimed.
\end{proof}

\begin{remark}
As we describe in Section~\ref{sec:related}, Cohen-Addad and Kanade~\cite{Cohen-Addad17:Online} show that if the functions their full-information, online optimization algorithm sees are piecewise constant, map from $[0,1]$ to $[0,1]$, and are $(w,0)$-dispersed at the maximizer, then their algorithm's regret is bounded by $O\left(\sqrt{T\ln(1/w)}\right)$.
The worst-case, piecewise constant functions $u_1, \dots, u_T$ from Theorem~\ref{thm:online_LB_single_dim} map from $[0,1]$ to $[0,1]$ and are $\left(\frac{1}{8}, \sqrt{T} + 1\right)$-dispersed at the maximizer, which means that our regret upper bound (Theorem~\ref{thm:1_d_online}) is $O\left(\sqrt{T\log (1/w)} + k\right) = O\left(\sqrt{T}\right)$. However, these functions are not $(w,0)$-dispersed at the maximizer for any $w \geq 2^{-T}$, so the regret bound by Cohen-Addad and Kanade~\cite{Cohen-Addad17:Online} is trivial, since $\sqrt{T\log (1/w)}$ with $w = 2^{-T}$ equals $T$.
\end{remark}

\subsection{Lower bound for multi-dimensional parameter spaces}\label{app:online_lb_multi}

We begin with the following corollary of Lemma~\ref{lem:Weed} by \citet{Weed16:Online} which simply generalizes the adversarial functions from single-dimensional thresholds to multi-dimensional thresholds (i.e., axis-aligned hyperplanes).

\begin{cor}[Corollary of Lemma~\ref{lem:Weed}]\label{cor:Weed}
For any $i \in [d]$, define the two functions $u^{(0)}: [0,1]^d \to [0,1]$ and $u^{(1)}: [0,1]^d \to [0,1]$ such that \[u^{(0)}(\vec{\rho}) = \begin{cases}
\frac{1}{2} &\text{if } \rho[i] < \frac{1}{2}\\
0 &\text{if } \rho[i] \geq \frac{1}{2}\end{cases} \text{ and } u^{(1)}(\vec{\rho}) = \begin{cases}
\frac{1}{2} &\text{if } \rho[i] < \frac{1}{2}\\
1 &\text{if } \rho[i] \geq \frac{1}{2}.\end{cases}\] There exists a pair of adversaries $U$ and $L$ defining two distributions $\mu_U$ and $\mu_L$ over $\left\{u^{(0)}, u^{(1)}\right\}$ such that for any learning algorithm, \[\max_{A \in \{U,L\}} \max_{\vec{\rho} \in [0,1]^d} \E \left[ \sum_{t = 1}^T u_t(\vec{\rho}) - \sum_{t = 1}^T u_t(\vec{\rho}_t)\right] \geq \frac{1}{32} \sqrt{T},\] where the expectation is over $u_1, \dots, u_{\numfunctions} \sim \mu_A$ and the random choices $\vec{\rho}_1, \dots, \vec{\rho}_{\numfunctions}$ of the algorithm. Moreover, under adversary $U$, any parameter vector $\vec{\rho}$ such that $\rho[i] > \frac{1}{2}$ is optimal and under adversary $L$, any parameter vector $\vec{\rho}$ such that $\rho[i] \leq \frac{1}{2}$ is optimal.
\end{cor}

\onlineLB*

\begin{proof}
The proof of this theorem is a straightforward generalization of Theorem~\ref{thm:online_LB_single_dim}.
We begin with an outline of the proof. For each dimension $i \in [d]$, the adversary plays $\left\lfloor\frac{\numfunctions - \sqrt{\numfunctions}}{d}\right\rfloor$ thresholds aligned with the $i^{th}$ axis, behaving exactly like the worse of the two adversaries defined in Corollary~\ref{cor:Weed}. Each threshold function has a discontinuity along the hyperplane $\left\{\vec{\rho} \in [0,1]^d : \rho[i] = \frac{1}{2}\right\}$. Since these functions are piecewise constant, either $\left\{\vec{\rho}\in [0,1]^d : \rho[i] \leq \frac{1}{2}\right\}$ is the set of points maximizing the sum of these $\left\lfloor\frac{\numfunctions - \sqrt{\numfunctions}}{d}\right\rfloor$ thresholds or $\left\{\vec{\rho}\in [0,1]^d : \rho[i] > \frac{1}{2}\right\}$. Denoting this maximizing set as $\cP_i^*$, let $\cP^* = \bigcap_{i = 1}^d \cP_i^*$ be the set of points maximizing all $\left\lfloor\frac{\numfunctions - \sqrt{\numfunctions}}{d}\right\rfloor$ thresholds over all $d$ dimensions. By definition of the sets $\cP_i^*$, this set is a hypercube with side-length $\frac{1}{2}$. Let $\vec{\rho}^*$ be the center of the hypercube $\cP^*$. Our adversary then plays $\numfunctions - d\left\lfloor\frac{\numfunctions - \sqrt{\numfunctions}}{d}\right\rfloor \leq \sqrt{T} + d$ copies of the indicator function corresponding to the ball $\left\{\vec{\rho} : ||\vec{\rho}^* - \vec{\rho}|| \leq 2^{-T}\right\}$. At the end of all $T$ rounds, $\vec{\rho}^*$ maximizes the sum $\sum_{t = 1}^{\numfunctions} u_t$. We prove that the expected regret incurred by this adversary is at least $\frac{\sqrt{Td}}{64}$, which follows from Corollary~\ref{cor:Weed}.  In order to prove the theorem, we need to show that $\frac{\sqrt{Td}}{64} = \Omega\left(\inf_{(w,k) \in D}\left\{ \sqrt{Td\log \frac{1}{w}} + k \right\}\right)$. Therefore, we need to show that the set of functions played by the adversary is $(w,k)$-dispersed at the maximizer $\vec{\rho}^*$ for $w = \Theta(1)$ and $k = O\left(\sqrt{\numfunctions d}\right).$ The reason this is true is that the only functions with discontinuities in the ball $\left\{\vec{\rho} : ||\vec{\rho}^* - \vec{\rho}|| \leq \frac{1}{8}\right\}$ are the final $\sqrt{T} + d$ functions played by the adversary. Thus, the theorem statement holds.

\bigskip\noindent\emph{Regret lower bound.}
Fix the learning algorithm. We begin be demonstrating the existence of a sequence of functions inducing a regret lower bound of $\Omega\left(\sqrt{Td}\right)$.

\begin{claim}\label{claim:reg_lower_bnd_multid}
Let $T' = \left\lfloor \frac{T - \sqrt{T}}{d} \right\rfloor$. There is a sequence $u_1, \dots, u_{T'd}$ of piecewise constant functions mapping $[0,1]^d$ to $[0,1]$ such that:
\begin{enumerate}
\item The expected regret is lower bounded as follows: $\max_{\vec{\rho} \in [0,1]}\E\left[\sum_{t = 1}^{T'd} u_t(\vec{\rho}) - u_t\left(\vec{\rho}_t\right)\right] \geq \frac{\sqrt{Td}}{64},$ where the expectation is over the random choices $\vec{\rho}_1, \dots, \vec{\rho}_{T'd}$ of the learner.
\item The set of points maximizing $\sum_{t = 1}^{\numfunctions' d} u_t$ is a hypercube of side length $\frac{1}{2}$.
\end{enumerate}
\end{claim}

\begin{proof}[Proof of Claim~\ref{claim:reg_lower_bnd_multid}]
Corollary~\ref{cor:Weed} with $i = 1$ tells us there exists a randomized adversary such that \[\max_{\vec{\rho} \in [0,1]^d} \E\left[\sum_{t = 1}^{T'} u_t^{(1)}(\vec{\rho}) - u_t^{(1)}\left(\vec{\rho}_t\right)\right] \geq \frac{1}{32} \sqrt{T'},\] where the expectation is over the random sequence $u_1^{(1)}, \dots, u_{T'}^{(1)}$ of functions chosen by the adversary and the random choices $\vec{\rho}_1, \dots, \vec{\rho}_{T'}$ of the learner. Next, for each $i \in \{2, \dots, d\}$, we apply Corollary~\ref{cor:Weed} to get $T'$ random functions $u_1^{(i)}, \dots, u_{T'}^{(i)}$ such that \begin{equation}\max_{\vec{\rho} \in [0,1]^d} \E\left[\sum_{t = 1}^{T'} u_{t}^{(i)}(\vec{\rho}) - u_{t}^{(i)}\left(\vec{\rho}_{(i-1)T' + t}\right)\right] \geq \frac{1}{32} \sqrt{T'},\label{eq:Weed_multid}\end{equation} where the expectation is over the random sequence $u_1^{(i)}, \dots, u_{T'}^{(i)}$ of functions chosen by the adversary and the random choices $\vec{\rho}_{(i-1)T' + 1}, \dots, \vec{\rho}_{iT'}$ of the learner.
Since for each $i \in [d]$, Equation~\eqref{eq:Weed_multid} holds in expectation over the adversary's choices, there must be a sequence $u_1^{(i)}, \dots, u_{T'}^{(i)}$ of functions such that \[\max_{\vec{\rho} \in [0,1]^d} \E\left[\sum_{t = 1}^{T'} u_{t}^{(i)}(\vec{\rho}) - u_{t}^{(i)}\left(\vec{\rho}_{(i-1)T' + t}\right)\right] \geq \frac{1}{32} \sqrt{T'},\] where the expectation is only over the random choices $\vec{\rho}_{(i-1)T' + 1}, \dots, \vec{\rho}_{iT'}$ of the learner.

From Corollary~\ref{cor:Weed}, we know that
either \[\left\{\vec{\rho} \in [0,1]^d : \rho[i] \leq \frac{1}{2}\right\} = \argmax_{\vec{\rho} \in [0,1]^d} \left\{\E\left[\sum_{t = 1}^{T'} u_{t}^{(i)}(\vec{\rho}) - u_{t}^{(i)}\left(\vec{\rho}_{(i-1)T' + t}\right)\right]\right\}\] or \[\left\{\vec{\rho} \in [0,1]^d : \rho[i] > \frac{1}{2}\right\} = \argmax_{\vec{\rho} \in [0,1]^d} \left\{\E\left[\sum_{t = 1}^{T'} u_{t}^{(i)}(\vec{\rho}) - u_{t}^{(i)}\left(\vec{\rho}_{(i-1)T' + t}\right)\right]\right\}.\] Call this set of maximizing points $\cP^*_i$.
 Note that the intersection $\cP^* = \bigcap_{i = 1}^d \cP_i^*$ of these $d$ sets is a hypercube with side length $\frac{1}{2}$. Therefore, for any $\vec{\rho} \in \cP^*$, \[\E\left[\sum_{i = 1}^d\sum_{t = 1}^{T'} u_{t}^{(i)}(\vec{\rho}) - u_{t}^{(i)}\left(\vec{\rho}_{(i-1)T' + t}\right)\right] \geq \frac{d}{32}\sqrt{T'} = \frac{d}{32} \sqrt{\left\lfloor \frac{T - \sqrt{T}}{d} \right\rfloor} \geq \frac{d}{32} \sqrt{\frac{T}{4d}} = \frac{\sqrt{Td}}{64}.\] For ease of notation, we relabel the functions $u_1^{(1)}, \dots, u_{T'}^{(1)}, \dots, u_1^{(d)}, \dots, u_{T'}^{(d)}$ as $u_1, \dots, u_{T'd}$.
\end{proof}

\bigskip\noindent\emph{Construction of the final $T - T'd$ functions.}
Let $\vec{\rho}^*$ be the center of the hypercube $\cP^*$.
We now define the functions $u_{T'd + 1}, \dots, u_{T}$ to all be equal to the function $\vec{\rho} \mapsto \textbf{1}_{\left\{||\vec{\rho} - \vec{\rho}^*|| \leq 2^{-T}\right\}}.$
Under this definition, the parameter $\vec{\rho}^*$ remains a maximizer of the sum $\sum_{t = 1}^{T} u_t$.

In our final regret bound, we will use the following property of the functions $u_{T'd + 1}, \dots, u_{T}$.
\begin{claim}\label{claim:regret_nonzero_multid}
For any parameters $\vec{\rho}_{T'd+1}, \dots, \vec{\rho}_T$,
$\sum_{t = T'd+1}^T u_t\left(\vec{\rho}^*\right) - u_t\left(\vec{\rho}_t\right) \geq 0.$
\end{claim}
\begin{proof}[Proof of Claim~\ref{claim:regret_nonzero_multid}]
By definition, $\sum_{t = T'd+1}^T u_t\left(\vec{\rho}^*\right) = T - T'd + 1$. Since the range of each function $u_t$ is contained in $[0,1]$, for any parameters $\vec{\rho}_{T'd+1}, \dots, \vec{\rho}_T$, $\sum_{t = T'd+1}^T u_t\left(\vec{\rho}_t\right) \leq T - T'd + 1$. Therefore, the claim holds.
\end{proof}

\bigskip\noindent\emph{Dispersion parameters.}
We now prove that the only functions with discontinuities in the ball $\left\{\vec{\rho} : ||\vec{\rho}^* - \vec{\rho}|| \leq \frac{1}{8}\right\}$ are the functions $u_{T'd + 1}, \dots, u_{T}$. Since $\cP^*$ is a hypercube with side length $\frac{1}{2}$ and $\vec{\rho}^*$ is the center of that hypercube, $\left\{\vec{\rho} : ||\vec{\rho}^* - \vec{\rho}|| \leq \frac{1}{8}\right\} \subset \cP^*.$ Therefore, the ball $\left\{\vec{\rho} : ||\vec{\rho}^* - \vec{\rho}|| \leq \frac{1}{8}\right\}$ only contains the discontinuities of the functions $u_{T'd + 1}, \dots, u_{T}$. Since $T - T'd = T - d\left\lfloor \frac{T - \sqrt{T}}{d} \right\rfloor \leq T - d\left(\frac{T - \sqrt{T}}{d}-1\right) = \sqrt{T} +d$, the set $\left\{u_1, \dots, u_T\right\}$ is $\left(\frac{1}{8}, \sqrt{T} + d\right)$-dispersed at the maximizer $\vec{\rho}^*$. Therefore, \begin{align*}&\inf_{(w,k) \in D} \left\{\sqrt{Td \log \frac{1}{w}} + k \right\}\\
&\leq \sqrt{T \log 8} + \sqrt{T} + d\\
&\leq 4\sqrt{Td} + 0\\
&\leq 256\max_{\vec{\rho} \in [0,1]^d} \E\left[\sum_{t = 1}^{T'd} u_t(\vec{\rho}) - u_t\left(\vec{\rho}_t\right)\right] + \E\left[\sum_{t = T'd+1}^{T} u_t\left(\vec{\rho}^*\right) - u_t\left(\vec{\rho}_t\right)\right] &\text{(Claims~\ref{claim:reg_lower_bnd_multid} and \ref{claim:regret_nonzero_multid})}\\
&= 256\E\left[\sum_{t = 1}^{T'd} u_t\left(\vec{\rho}^*\right) - u_t\left(\vec{\rho}_t\right)\right] + \E\left[\sum_{t = T'd+1}^{T} u_t\left(\vec{\rho}^*\right) - u_t\left(\vec{\rho}_t\right)\right] &\left(\vec{\rho}^* \in \argmax_{\vec{\rho} \in [0,1]^d}\sum_{t = 1}^{T'd} u_t(\vec{\rho})\right)\\
&\leq 256\E\left[\sum_{t = 1}^{T} u_t\left(\vec{\rho}^*\right) - u_t\left(\vec{\rho}_t\right)\right]\\
&\leq 256\max_{\vec{\rho} \in [0,1]^d}\E\left[\sum_{t = 1}^{T} u_t\left(\vec{\rho}\right) - u_t\left(\vec{\rho}_t\right)\right].\\\end{align*} Therefore, \[\max_{\vec{\rho} \in [0,1]^d} \E\left[\sum_{t = 1}^{T} u_t(\vec{\rho}) - u_t\left(\vec{\rho}_t\right)\right] = \Omega\left(\inf_{(w,k) \in D} \left\{\sqrt{Td \log \frac{1}{w}} + k \right\}\right),\] as claimed.
\end{proof}

\subsection{Differentially Private Online Learning}

\begin{lem}[\citet{Dwork10:Boosting}]\label{lem:adv_comp}
Given target privacy parameters $\epsilon \in (0,1)$ and $\delta > 0$, to ensure $(\epsilon, \tau\delta' + \delta)$ cumulative privacy loss over $\tau$ mechanisms, it suffices that each mechanism is $(\epsilon', \delta')$-differentially private, where \[\epsilon' = \frac{\epsilon}{2\sqrt{2\tau \ln (1/\delta)}}.\]
\end{lem}

\begin{theorem}\label{thm:online_1d_private}
  Let $u_1, \dots, u_T$ be the sequence of functions observed by
  Algorithm~\ref{alg:1_d_online} and suppose they satisfy the conditions of
  Theorem~\ref{thm:1_d_online}. Let $\epsilon \in (0,1)$ and $\delta > 0$ be
  privacy parameters. If $\lambda = \frac{\epsilon}{4H\sqrt{2T\ln(1/\delta)}}$,
  then Algorithm~\ref{alg:1_d_online} is ($\epsilon, \delta$)-differentially
  private. Its regret is bounded by
  \[H\sqrt{T}\left(\frac{\epsilon}{4\sqrt{2\ln(1/\delta)}} +
  \frac{4\ln(R/w)\sqrt{2\ln(1/\delta)}}{\epsilon}\right) + Hk + LTw.\] Moreover,
  suppose there are $K$ intervals partitioning $\configs$ so that $\sum_{t = 1}^T = u_t$ is piecewise $L$-Lipschitz on each interval. Then the running time of
  Algorithm~\ref{alg:1_d_online} is $T\cdot poly(K)$.
\end{theorem}

\begin{proof}
For all $t \in [T]$, the sensitivity of the function $\sum_{i = 0}^{t-1} u(x_t, \cdot)$ is bounded by $H$. Therefore, at each time step $t$, Algorithm~\ref{alg:1_d_online} samples from the exponential mechanism with privacy parameters $\epsilon' = \frac{\epsilon}{2\sqrt{2T\ln(1/\delta)}}$ and $\delta = 0$. The privacy guarantee therefore follows from Lemma~\ref{lem:adv_comp}. The regret bound follows from Theorem~\ref{thm:1_d_online}. The running time follows from the running time of Algorithm~\ref{alg:1dEfficient}.
\end{proof}

\begin{cor}\label{cor:online_1d_private}
  Let $u_1, \dots, u_T$ be the sequence of functions observed by
  Algorithm~\ref{alg:1_d_online} and suppose they satisfy the conditions of
  Theorem~\ref{thm:1_d_online}. Let $\epsilon \in (0,1)$ and $\delta > 0$ be
  privacy parameters. Suppose $T \geq 1/(Lw)$. If $\lambda =
  \frac{\epsilon}{4H\sqrt{2T\ln(1/\delta)}}$, then
  Algorithm~\ref{alg:1_d_online} is ($\epsilon, \delta$)-differentially private.
  Its regret is bounded by
  \[
  H\sqrt{T}\left(\frac{\epsilon}{4\sqrt{2\ln(1/\delta)}} +
  \frac{4\ln(RLT)\sqrt{2\ln(1/\delta)}}{\epsilon}\right) + Hk + 1.
  \]
\end{cor}

 \begin{proof}
   This bound follows from applying Theorem~\ref{thm:online_1d_private}
   using the $(w',k)$-disperse critical boundaries condition with
   $w' = 1/(LT)$. The lower bound on requirement
   on $T$ ensures that $w' \leq w$.
 \end{proof}

For multi-dimensional parameter spaces, we prove a similar theorem with respect to Algorithm~\ref{alg:multi_d_online}.

\begin{theorem}\label{thm:online_multiD_private}
  Let $u_1, \dots, u_T$ be the sequence of functions observed by
  Algorithm~\ref{alg:multi_d_online} and suppose they satisfy the conditions of
  Theorem~\ref{thm:1_d_online}. Moreover, suppose $\sum_{t=1}^T u_t$ is
  piecewise concave on convex pieces. Let $\epsilon \in (0,1)$ and $\delta > 0$
  be privacy parameters. Also, let $\epsilon' =
  \epsilon/\left(2\sqrt{2T\ln(2/\delta)}\right)$, $\lambda = \epsilon'/(6H)$,
  $\eta = \epsilon'/3$, and $\zeta = \delta/\left(2T\left(1 + e^{\epsilon'}\right)\right)$.
  Algorithm~\ref{alg:multi_d_online} with input $\lambda$, $\eta$, and $\zeta$
  is ($\epsilon, \delta$)-differentially private. Moreover, its regret is
  bounded by \[\frac{H\epsilon}{12}\sqrt{\frac{T}{2\ln(1/\delta)}} + \frac{12Hd
  \ln (R/w)\sqrt{2T\ln(1/\delta)}}{\epsilon} +
  H\left(k+2+\frac{\delta}{2}\right) + LTw.\] Moreover, suppose there are $K$
  intervals partitioning $\configs$ so that $U(\sample, \cdot)$ is piecewise
  $L$-Lipschitz on each interval. Then the running time of
  Algorithm~\ref{alg:multi_d_online} is $TK\cdot poly\left(d, H, T, K,
  \frac{1}{\epsilon}, \log \frac{1}{\delta}, \log \frac{R}{r}\right)$.
\end{theorem}

\begin{proof}
For all $t \in [T]$, the sensitivity of the function $\sum_{i = 0}^{t-1} u(x_t, \cdot)$ is bounded by $H$. By Lemma~\ref{lem:efficientPrivacy}, at each time step $t$, Algorithm~\ref{alg:multi_d_online} samples from a distribution that is $(\epsilon', \delta/(2T))$-differentially private. By Lemma~\ref{lem:adv_comp}, this means that Algorithm~\ref{alg:multi_d_online} is $(\epsilon, 2\delta)$-differentially private. The regret and running time bounds follow from Theorem~\ref{thm:multi_d_online}.
\end{proof}

\section{Proofs for differential privacy (Section~\ref{sec:algorithm})}\label{app:algorithm}
\cexpmutility*
\begin{proof}
  The proof follows the same outline as the utility guarantee for the
  exponential mechanism given by \citet{Dwork14:PrivacyBook} when the set of
  outcomes is finite. The main additional challenge is lower bounding the
  normalizing constant for $\expmf$, which is the key place where we use
  dispersion.

  Let $\expmf(\vec{\rho}) = \exp\bigl(\frac{\epsilon \numfunctions}{2H} \cdot
  \frac{1}{\numfunctions}\sum_{t=1}^\numfunctions u_t(\vec{\rho})\bigr)$ be the
  unnormalized density sampled by the exponential mechanism. For any utility
  threshold $c$, let $E = \{\vec{\rho} \in \configs \,:\,
  \frac{1}{\numfunctions}\sum_{t=1}^\numfunctions u_t(\vec{\rho}) \leq c\}$ be
  the set of output points with average utility at most $c$. We can write the
  probability that a sample drawn from $\expmf$ lands in $E$ as $F / Z$, where
  $F = \int_E \expmf$ and $Z = \int_\configs \expmf$. We bound $F$ and $Z$
  independently.

  First, we have
  \[
  F
  = \int_E \expmf(\vec{\rho}) \, d\vec\rho
  \leq \int_E \exp\biggl( \frac{\epsilon \numfunctions c}{2H} \biggr) \, d\vec{\rho}
  = \exp\biggl( \frac{\epsilon \numfunctions c}{2H} \biggr) \cdot \vol(E)
  \leq \exp\biggl( \frac{\epsilon \numfunctions c}{2H} \biggr) \cdot \vol(\configs).
  \]

  To lower bound $Z$, we use the fact that at most $k$ of the functions $u_1,
  \dots, u_\numfunctions$ have discontinuities in the ball $B(\vec{\rho^*},w)$
  and the rest are $L$-Lipschitz. This implies that every $\vec{\rho} \in
  B(\vec{\rho}^*,w)$ satisfies $\frac{1}{\numfunctions}
  \sum_{t=1}^\numfunctions u_t(\rho) \geq OPT - Lw - Hk/\numfunctions$, where
  $OPT = \frac{1}{\numfunctions}\sum_{t=1}^\numfunctions u_t(\vec{\rho}^*)$.
  Therefore, we have
  \[
  Z
  = \int_\configs \expmf(\vec{\rho}) \, d\vec\rho
  \geq \int_{B(\vec{\rho}^*,w)} \expmf(\vec{\rho}) \, d\vec\rho
  \geq \exp\biggl(\frac{\epsilon \numfunctions}{2H}(OPT - Lw - Hk/\numfunctions)\biggr) \cdot \vol(B(\vec{\rho}^*,w)).
  \]

  Combining these bounds gives
  \[
  \frac{F}{Z}
  \leq
  \exp\biggl(
    \frac{\epsilon \numfunctions}{2H}
    (c - OPT + Lw + Hk/\numfunctions)
  \biggr)
  \frac{\vol(\configs)}{\vol(B(\vec{\rho}^*,w))}
  \leq
  \exp\biggl(
    \frac{\epsilon \numfunctions}{2H}
    (c - OPT + Lw + Hk/\numfunctions)
  \biggr)
  \biggl(\frac{R}{w}\biggr)^d,
  \]
  where the second inequality follows from the fact that $\configs$ is contained
  in a ball of radius $R$, and the volume of a ball of radius $r$ is
  proportional to $r^d$. Choosing $c$ so that this bound on the probability of
  outputting a point with average utility at most $c$ is at most $\zeta$
  completes the proof.

  Our efficient sampling algorithm is given in Algorithm~\ref{alg:efficient}.
  Given target privacy parameters $\epsilon > 0$ and $\delta > 0$, we use
  Algorithm~\ref{alg:efficient} to approximately sample from the unnormalized
  density $g(\vec{\rho}) = \frac{\epsilon' T}{2 H}
  \frac{1}{\numfunctions}\sum_{t=1}^\numfunctions u_t(\vec{\rho})$ with
  parameters $\epsilon' = \eta = \epsilon / 3$ and $\zeta = \delta /
  (1+e^\epsilon)$. In Lemma~\ref{lem:efficientPrivacy} we show that for these
  parameter settings, the algorithm preserves $(\epsilon, \delta)$-differential
  privacy and still has high utility.
\end{proof}

Next, as in the full-information online learning setting, we show that the
utility dependence on the Lipschitz constant $L$ can be made logarithmic. The
main idea is that whenever functions are $(w,k)$-dispersed, they are also
$(w',k)$-dispersed for any $w' \leq w$. By choosing $w'$ sufficiently small, we
are able to balance the $Lw$ and $\frac{dH}{\numfunctions\epsilon}\log
\frac{R}{w}$ terms.

\begin{cor}
  \label{cor:expmutilityoptimized}
  Suppose the functions $u_1, \dots, u_\numfunctions$ satisfy the conditions of
  Theorem~\ref{thm:cexpmutility} and $\numfunctions \geq \frac{2Hd}{w\epsilon L}$.
  Then with probability at least $1-\zeta$ the output $\hat{\vec{\rho}}$ sampled from $\expmf$ satisfies:
  \[
    \frac{1}{\numfunctions} \sum_{t=1}^\numfunctions u_t(\hat{\vec{\rho}})
    \geq
    \frac{1}{\numfunctions} \sum_{t=1}^\numfunctions u_t(\vec{\rho}^*)
    -
    O\biggl(
    \frac{H}{\numfunctions \epsilon}
    \biggl(
    d
    \log \frac{L\epsilon R\numfunctions}{2Hd} \biggr(
    +
    \log \frac{1}{\zeta}
    \biggr)
    + \frac{Hk}{\numfunctions}
    \biggr)
  \]
\end{cor}

\begin{proof}
  If the functions $u_1, \dots, u_\numfunctions$ are $(w,k)$-dispersed, then
  they are also $(w',k)$-dispersed for any $w' \leq w$. This bound follows from
  applying Theorem~\ref{thm:cexpmutility} using the $(w',k)$-dispersion with $w'
  = \frac{2Hd}{\epsilon L \numfunctions}$. The bound on $\numfunctions$ ensures
  that $w' \leq w$.
\end{proof}

In all of our applications we show $(w,k)$-dispersion for $w \approx
1/\sqrt{\numfunctions}$ and $k \approx \sqrt{\numfunctions}$ (ignoring
problem-specific parameters). In this case, the requirement on $\numfunctions$
becomes $\numfunctions^{3/2} \geq \frac{2H}{\epsilon L}$, which will be
satisfied for sufficiently large $T$.

\subsection{Approximate sampling for differential privacy}
\begin{restatable}{lem}{corEfficientPrivacyUtility}
  \label{lem:efficientPrivacy}
  Let $u_1, \dots, u_\numfunctions$ be piecewise $L$-Lipschitz and
  $(w,k)$-dispersed at a maximizer $\vec{\rho}^* \in \configs$, and suppose that
  $\configs \subset \reals^d$ is convex, contained in a ball of radius $R$, and
  $B(\vec{\rho^*},w) \subset \configs$. For any privacy parameters $\epsilon >
  0$ and $\delta > 0$, let $\hat{\vec{\rho}}$ be the output of running
  Algorithm~\ref{alg:efficient} to sample from $g(\vec{\rho}) =
  \frac{T\epsilon'}{2 H} \cdot \frac{1}{T} \sum_{t=1}^\numfunctions
  u_t(\vec{\rho})$ with parameters $\eta = \epsilon' = \epsilon / 3$ and $\zeta
  = \delta / (1+e^\epsilon)$. This procedure preserves $(\epsilon,
  \delta)$-differential privacy and with probability at least $1-\delta$ we have
  \[
  \frac{1}{\numfunctions} \sum_{t=1}^\numfunctions u_t(\hat{\vec{\rho}})
  \geq
  \frac{1}{\numfunctions} \sum_{t=1}^\numfunctions u_t(\vec{\rho}^*)
  -
  O\biggl(
  \frac{H}{\numfunctions\epsilon} \biggl(d \log \frac{R}{w} + \log \frac{1}{\delta} \biggr) - Lw - \frac{Hk}{\numfunctions}
  \biggr).
  \]
  \end{restatable}

  \begin{proof}
  Let $u_1, \dots, u_\numfunctions$ and $u'_1, \dots, u'_\numfunctions$ be two
  neighboring sets of functions (that is, they differ on at most one function)
  and let $g(\vec{\rho}) = \frac{T\epsilon'}{2 H} \cdot \frac{1}{T}
  \sum_{t=1}^\numfunctions u_t(\vec{\rho})$ and $g'(\vec{\rho}) =
  \frac{T\epsilon'}{2 H} \cdot \frac{1}{T} \sum_{t=1}^\numfunctions
  u'_t(\vec{\rho})$. Let $\mu$ and $\mu'$ be the distributions with densities
  proportional to $g$ and $g'$, respectively. The distributions $\mu$ is the
  output distribution of the exponential mechanism when maximizing
  $\frac{1}{\numfunctions} \sum_{t=1}^\numfunctions u_t$ (and similarly for
  $\mu'$). We know that exactly sampling from $\mu$ preserves $(\epsilon,
  0)$-differential privacy and has strong utility guarantees. When we run
  Algorithm~\ref{alg:efficient}, we get approximate samples from $\mu$ and
  $\mu'$. We need to show that the approximate sampling procedure still
  preserves $(\epsilon,\delta)$-differential privacy and has good utility.

  Let $\hat{\vec{\rho}}$ and $\hat{\vec{\rho}}'$ be samples produced by
  Algorithm~\ref{alg:efficient} when run on $g$ and $g'$, respectively. From
  Lemma~\ref{lem:efficientApprox}, we know that all approximate integration and
  sampling operations of Algorithm~\ref{alg:efficient} succeed with probability
  at least $1-\zeta$. Let $\hat \mu$ be the output distribution of
  Algorithm~\ref{alg:efficient} when run on $g$ conditioned on success for all
  integration and sampling operations (and similarly let $\hat \mu'$ be the
  distribution when run on $g'$ without failures). Also by
  Lemma~\ref{lem:efficientApprox}, we know that $\reldist(\hat \mu, \mu) \leq
  \eta$ and $\reldist(\hat \mu', \mu') \leq \eta$. With this, for any set $E
  \subset \configs$ of outcomes, we have
  \begin{align*}
  \prob(\hat{\vec{\rho}} \in E)
  &\leq \hat \mu(E) + \zeta & \hbox{(Failure probability of Algorithm~\ref{alg:efficient})}\\
  &\leq e^\eta \mu(E) + \zeta & \hbox{($\reldist(\hat \mu, \mu) \leq \eta$)} \\
  &\leq e^{2\eta} \mu'(E) + \zeta & \hbox{(The exp. mech. preserves $\eta$-differential privacy)} \\
  &\leq e^{3\eta} \hat \mu'(E) + \zeta & \hbox{($\reldist(\hat \mu', \mu') \leq \eta$)}\\
  &\leq e^{3\eta}( \prob(\hat{\vec{\rho}}' \in E) + \zeta) + \zeta & \hbox{(Failure probability of Algorithm~\ref{alg:efficient})}\\
  &= e^\epsilon \prob(\hat{\vec{\rho}}' \in E) + \delta.
  \end{align*}
  It follows that the approximate sampling procedure preserves $(\epsilon,
  \delta)$-differential privacy.

  Next we turn to proving the utility guarantee. Let
  \[
    E = \left\{ \vec{\rho} \in \configs \,:\,
    \frac{1}{\numfunctions} \sum_{t=1}^\numfunctions u_t(\vec{\rho})
    <
    \frac{1}{\numfunctions} \sum_{t=1}^\numfunctions u_t(\vec{\rho}^*)
    -
    \frac{2H}{\numfunctions\eta} \biggl(d \log \frac{R}{w} + \log \frac{1}{\zeta} \biggr) - Lw - \frac{Hk}{|\sample|}\right\},
  \]
  be the set of parameter vectors with high suboptimality. By
  Theorem~\ref{thm:cexpmutility} we know that $\mu(E) \leq \zeta$. Applying
  Lemma~\ref{lem:efficientApprox}, we have
  \[
    \prob(\hat{\vec{\rho}} \in E)
    \leq \hat\mu(E) + \zeta
    \leq e^{\eta}\mu(E) + \zeta
    \leq (1+e^{\eta})\zeta = \delta,
  \]
  and the claim follows.
\end{proof}

\subsection{Lower bound for differential privacy} \label{sec:privacyLowerBound}

Our privacy lower bounds follow a similar packing construction as the bounds
given by \citet{De12:Lower}. We will make use of the following simple Lemma
arguing that we can pack many balls of radius $r$ into the unit ball in $d$
dimensions.

\begin{lem} \label{lem:BallPacking}
  For any dimension $d$ and any radius $0 < r \leq 1/2$, there exist $t =
  (4r)^{-d}$ disjoint balls $B_1, \dots, B_t$ of radius $r$ contained in
  $B(0,1)$.
\end{lem}
\begin{proof}
  Let $\vec{\rho}_1, \dots, \vec{\rho}_t \in B(0,1/2)$ be any maximal set of points
  satisfying $\norm{\vec{\rho}_i - \vec{\rho}_j}_2 \geq 2r$ for any $i \neq j$. First, we
  argue that $B(0,1)$ is contained in $\bigcup_{i=1}^t B(\vec{\rho}_i, 2r)$. For
  contradiction, suppose there is some point $\vec{\rho} \in B(0,1/2)$ that is not
  contained in $\bigcup_{i=1}^t B(\vec{\rho}_i, 2r)$. Then we must have that
  $\norm{\vec{\rho} - \vec{\rho}_i}_2 \geq 2r$ for all $r$, which implies that it could be
  added to the list $\vec{\rho}_1, \dots, \vec{\rho}_t$, contradicting maximality. From
  this, it follows that $\vol(B(0,1/2)) \leq \vol(\bigcup_i B(\vec{\rho}_i, 2r))$.
  Using the fact that $\vol(B(\cdot, r)) = r^d v_d$ and $\vol(\bigcup_i
  B(\vec{\rho}_i, 2r)) \leq \sum_i \vol(B(\vec{\rho}_i, 2r))$, this implies that $(1/2)^d
  v_d \leq t (2r)^d v_d$. Rearranging gives $t \geq (4r)^{-d}$.

  Now consider the set of balls given by $B_i = B(\vec{\rho}_i, r)$. We know that
  $B_i \subset B(0,1)$, since $\vec{\rho}_i \in B(0,1/2)$ and $r \leq 1/2$. Moreover,
  since $\norm{\vec{\rho}_i - \vec{\rho}_j}_2 \geq 2r$ for all $i \neq j$, we have that
  $B_i \cap B_j = \emptyset$ for all $i \neq j$. It follows that the set of
  balls $B_1, \dots, B_t$ are disjoint and contained in $B(0,1)$.
\end{proof}

With this, we are ready to prove our differential privacy lower bound.

\thmPrivacyLowerBound*
\begin{proof}
  We will construct $M=2^d$ multisets $\sample_1, \dots, \sample_M$ of piecewise
  constant functions all satisfying the same $(w,k)$-dispersion parameters. We
  argue that for every $\epsilon$-differentially private optimizer $\cA$, there
  is at least one $\sample_i$ such that $\cA(\sample_i)$ outputs a relatively
  suboptimal point with high probability. Next, we tune the parameters of the
  construction so that this suboptimality bound can be expressed in terms of the
  dispersion parameters $w$ and $k$.

  \newcommand \uall {u_{\rm all}}
  \vspace{1em}\noindent\textit{Set Construction.}
  Let $\vec{\rho}_1$, \dots, $\vec{\rho}_M$ be a collection of $M = 2^d$ points such that
  the balls $B(\vec{\rho}_i,1/8)$ for $i = 1, \dots, M$ are disjoint and contained in
  $B(0,1)$ (Lemma~\ref{lem:BallPacking} ensures that such a collection exists).
  Now define $\uall(\vec{\rho}) = \ind{\vec{\rho} \in \bigcup_{i=1}^M B(\vec{\rho}_i, 1/8)}$
  and $u_i(\vec{\rho}) = \ind{\vec{\rho} \in B(\vec{\rho}_i, r)}$ for each $i = 1, \dots, M$,
  where $r$ is a parameter we will set later. Finally, for each index $i$, let
  $\sample_i$ be the multiset of functions that contains $N$ copies of $u_i$ and
  $\numfunctions - N$ copies of $\uall$, where $N$ is a second parameter of the
  construction that we will set later.

  \vspace{1em}\noindent\textit{Dispersion Parameters.}
  For each set $\sample_i$, we can exactly characterize the $(w,k)$-dispersion
  parameters at the maximizer. First, for $\sample_i$, the point $\vec{\rho}_i$ is a
  maximizer with total utility $\numfunctions$. On the other hand, any point
  outside $B(\vec{\rho}_i,r)$ has utility at most $\numfunctions-N < \numfunctions$.
  For any $w \leq r$, the ball $B(\vec{\rho}_i, w)$ is not split by any of the
  discontinuities of functions in $\sample_i$, so the functions are
  $(w,0)$-dispersed at the maximizer. For $r < w \leq 1/8$, the ball $B(\vec{\rho}_i,
  w)$ is split by the discontinuities of the $N$ copies of $u_i$, and so the
  functions are $(w,N)$-dispersed at the maximizer. Finally, for any $w > 1/8$,
  the functions are $(w,\numfunctions)$-dispersed at the maximizer, since every
  function's discontinuity splits the ball. To summarize, the functions are
  $(w,k)$-dispersed at the maximizer for any $w$ with
  \[
  k = \begin{cases}
    0 & \hbox{if $w < r$} \\
    N & \hbox{if $r \leq w < 1/8$} \\
    \numfunctions & \hbox{if $w \geq 1/8$}.
  \end{cases}
  \]

  \vspace{1em}\noindent\textit{Suboptimality.}
  Let $\cA$ be any $\epsilon$-differentially private optimizer for collections
  of piecewise constant functions. We first argue that running $\cA$ on
  $\sample_1$ must output a point with low utility for at least one of the other
  sets of functions $\sample_i$ with high probability. Since the balls
  $B(\vec{\rho}_i, 1/8)$ are disjoint, we also know that the balls $B(\vec{\rho}_i, r)$
  are also. Therefore, we have that $\sum_{i=1}^M \prob(\cA(\sample_1) \in
  B(\vec{\rho}_i, r)) \leq 1$. But this implies that there exists some $i$ such that
  $\prob(\cA(\sample_1) \in B(\vec{\rho}_i, r)) \leq 1/M = 2^{-d}$. Given that any
  point outside of $B(\vec{\rho}_i, r)$ has suboptimality at least $N$ for the set
  $\sample_i$, it follows that $\cA(\sample_1)$ has suboptimality at least $N$
  for the functions in $\sample_i$ with probability at least $1-2^{-d}$. Next,
  we show that this implies that $\cA$ has low utility when run on $\sample_i$
  itself. Since $\cA$ is $\epsilon$-differentially private and the sets of
  functions $\sample_1$ and $\sample_i$ differ only $2N$ functions (the $N$
  copies of $u_1$ in $\sample_1$ and the $N$ copies of $u_i$ in $\sample_i$),
  we have
  \[
  \prob(\cA(\sample_i) \in B(\vec{\rho}_i, r))
  \leq e^{2\epsilon N} \prob(\cA(\sample_1) \in B(\vec{\rho}_i, r))
  \leq e^{2\epsilon N}/M
  \]
  Therefore, with probability at least $1-e^{2\epsilon N}/M$, the point
  $\cA(\sample_i)$ is $N$-suboptimal for $\sample_i$.

  \vspace{1em}\noindent\textit{Parameter Setting.}
  There are two parameters in the above construction that we can set: $r$, the
  radius of the small optimal balls, and $N$, the number of copies of the
  indicator function for those small balls in each set of functions. Inuitively,
  we will set $r$ to be small enough so that the dispersion parameters giving
  the best bound are $w = 1/8$ and $k = N$. Tuning the value of $N$ is more
  involved.

  Let $r$ be small enough that $\frac{d}{\epsilon} \log \frac{1}{r} \geq
  \frac{d}{\epsilon} \log \frac{1}{8} + N$. For this value of $r$ we have that
  \[
    \inf_{w,k} \frac{d}{\epsilon} \log \frac{1}{w} + k = \frac{d}{\epsilon} \log \frac{1}{8} + N.
  \]
  We also know that with probability at least $1-e^{2\epsilon N}/M$, the
  suboptimality of algorithm $\cA$ when run on $\sample_i$ is at least $N$.
  Choosing the value of $N$ trades between two competing effects: first, as we
  increase $N$, the suboptimality of $\cA$ in the bad event that it outputs a
  point outside of $B(\vec{\rho}_i, r)$ get worse (formally, our suboptimality lower
  bound scales with $N$). Second, as we increase $N$, the datasets $\sample_1,
  \dots, \sample_M$ become more different, and the probability of the bad event
  required by $\epsilon$-differential privacy drops (formally, $1-e^{2\epsilon
  N}/M$ gets smaller as $N$ grows). We will have proved the theorem if we can
  find a value of $N$ such that the probability $e^{2\epsilon N}/M \leq \zeta$
  and $N = \Omega(\inf_{(w,k)}\frac{d}{\epsilon}(\log \frac{1}{w} - \log
  \frac{1}{\zeta}) + k)$. We will have $e^{2\epsilon N}/M \leq \zeta$ whenever
  $N \leq \frac{d}{\epsilon}(\frac{\ln 2}{2} - \ln \frac{1}{\zeta})$. Therefore,
  setting $N = \frac{d}{\epsilon}(\frac{\ln 2}{2} - \ln \frac{1}{\zeta})$
  achieves the probability requirement. Finally, for this setting we have that
  $N = \Omega(N + N) = \Omega(\inf_{(w,k)}\frac{d}{\epsilon}(\log \frac{1}{w} -
  \ln \frac{1}{\zeta}) + k)$. For this setting to be justified, we must have
  $\numfunctions \geq N = \frac{d}{\epsilon}(\frac{\ln 2}{2} - \ln
  \frac{1}{\zeta}))$.

  Finally, this bound was on the total suboptimality. Dividing by
  $\numfunctions$ proves the theorem.
\end{proof}

Next, we show that the above lower bound can be instantiated by maximum weight
independent set instances, showing that these lower bounds bind for algorithm
configuration problems. In this case, the dimension of the problem is $d=1$. To
show this, we only need to construct MWIS instances for which the utility
function of our greedy algorithm as a function of its parameter behaves like the
indicator set for some subinterval of $[0,1]$. The following Lemma shows that
this can be achieved. For a graph $x$, let $u(x, \rho)$ be the total weight of
the independent set returned by the algorithm parameterized by $\rho$.
\begin{lem}[\citet{Gupta16:PAC}]\label{lem:Gupta_WC}
  For any constants $0 < r < s< 1$ and any $t \geq 2$, there exists a MWIS
  instance $x$ on $t^3 + 2t^2 + t - 2$ vertices such that $u(x, \rho) = 1$ when
  $\rho \in (r,s)$ and $u(x, \rho) = \frac{t^r(t^2 - 2) + t^{-s}(t^2 + t +
  1)}{t^3 - 1}$ when $\rho \in [0,1] \setminus (r,s)$.
\end{lem}

\begin{cor}\label{cor:Gupta_WC}
  For any constants $\frac{1}{10} < r < s< \frac{3}{20}$, there exists a MWIS
  instance $x$ on 178 vertices such that $u(x, \rho) = 1$ when $\rho \in (r,s)$
  and $\frac{2}{5} \leq u(x, \rho) \leq \frac{1}{2}$ when $\rho \in [0,1]
  \setminus (r,s)$.
\end{cor}

While the Corollary~\ref{cor:Gupta_WC} does not show that the constructed
instance behave exactly as indicator functions for subintervals, it demonstrates
that for any interval $[r,s] \subset [\frac{2}{20}, \frac{3}{20}]$, we can
construct a graph $x$ so that the utility for any $\rho \in [r,s]$ is 1, and the
utility for any $\rho \not \in [r,s]$ is at most $1/2$. This additive gap is
enough to instantiate Theorem~\ref{thm:lower_bound} (after rescaling
appropriately so that the construction is performed in the interval
$[\frac{2}{20}, \frac{3}{20}]$).

\section{Proofs for algorithm configuration (Section~\ref{sec:applications})}\label{app:applications}
\subsection{MWIS algorithm configuration}

\MWIS*
\begin{proof}
 Given a set of samples $\sample = \left\{\left(\vec{w}^{\left(1\right)}, \vec{e}^{\left(1\right)}\right), \dots, \left(\vec{w}^{\left(\numfunctions\right)}, \vec{e}^{\left(\numfunctions\right)}\right)\right\}$, \citet{Gupta17:PAC} prove that the $\sum_{t = 1}^{\numfunctions} u\left(\vec{w}^{\left(t\right)}, \vec{e}^{\left(t\right)}, \cdot\right)$ is piecewise constant and the boundaries between the constant pieces have the
  form \[\frac{\ln\left(w_i^{\left(t\right)}\right) - \ln\left(w_j^{\left(t\right)}\right)}{\ln\left(d_1\right) - \ln\left(d_2\right)}\] for all $t \in
  [\numfunctions]$ and $i,j,d_1, d_2 \in [n]$, where $w_j^{\left(t\right)}$ is the weight of the
  $j^{th}$ vertex of the $t^{th}$ sample. For each unordered pair $\left(i,j\right) \in
  {[n] \choose 2}$ and degrees $d_1, d_2 \in [n]$, let \[\cB_{i,j,d_1,d_2} =
  \left\{\frac{\ln\left(w_i^{(t)}\right) - \ln\left(w_j^{(t)}\right)}{\ln\left(d_1\right) - \ln\left(d_2\right)} \ : \ t \in [\numfunctions]\right\}.\]
  The points in each set $\cB_{i,j,d_1,d_2}$ are independent since they are
  determined by different problem instances. Since the vertex weights are
  supported on $(0,1]$ and have pairwise $\kappa$-bounded joint densities,
  Lemma~\ref{lem:ln_difference_bounded} tells us that $\ln\left(w_i^{(t)}\right) -
  \ln\left(w_j^{(t)}\right)$ has a $\kappa/2$-bounded distribution for all $i,j\in[n]$ and
  $t \in [\numfunctions]$. Also, since $|\ln\left(d_1\right) - \ln\left(d_2\right)| \leq \ln n$, Lemma~\ref{lem:constant_mult} allows
  us to conclude that the elements of each set $\cB_{i,j,d_1,d_2}$ come from
  $\frac{\kappa \ln n}{2}$-bounded distributions. The theorem statement follows from
  Lemma~\ref{lem:dispersion} with $M = \max
  \left|\cB_{i,j,d_1,d_2}\right| = \numfunctions$ and $P = n^4/2$.
\end{proof}

\begin{theorem}[Differential privacy]\label{thm:MWIS_DP}
  Given a set of samples $\sample = \left\{\left(\vec{w}^{\left(1\right)}, \vec{e}^{\left(1\right)}\right), \dots, \left(\vec{w}^{\left(\numfunctions\right)}, \vec{e}^{\left(\numfunctions\right)}\right)\right\} \sim \dist^T$, suppose Algorithm~\ref{alg:1dEfficient} takes as input the function $\sum_{t = 1}^T u\left(\vec{w}^{\left(t\right)}, \vec{e}^{\left(t\right)}, \cdot\right)$ and the set of intervals over which this function is piecewise constant. Suppose all vertex weights are in $(0,1]$
  and every pair of vertex weights has a $\kappa$-bounded joint distribution.
  Algorithm~\ref{alg:1dEfficient} returns a parameter $\hat{\rho}$ such that
  with probability at least $1-\zeta$ over the draw of $\sample$,
  \[
    \E_{(\vec{w}, \vec{e}) \sim \dist}\left[u\left(\vec{w}, \vec{e}, \hat{\rho}\right)\right]
    \geq
    \max_{\rho \in [0,B]}\E_{(\vec{w}, \vec{e}) \sim \dist}[u\left(\vec{w}, \vec{e}, \rho\right)] - O\left(\frac{H}{\numfunctions\epsilon} \log \frac{B\numfunctions\kappa\ln n}{\zeta} + Hn^4\sqrt{\frac{\log\left(n/\zeta\right)}{\numfunctions}}\right).
  \]
\end{theorem}
\begin{proof}
  The theorem statement follows from Theorems~\ref{thm:cexpmutility} and
  \ref{thm:MWIS_upper} and Lemma~\ref{lem:MWIS_sample}.
\end{proof}

\begin{theorem}[Full information online optimization]\label{thm:MWIS_full_info}
  Let $u\left(\vec{w}^{(1)}, \vec{e}^{(1)}, \cdot \right), \dots, u\left(\vec{w}^{(\numfunctions)}, \vec{e}^{(\numfunctions)}, \cdot \right)$ be the set of functions observed by Algorithm~\ref{alg:1_d_online}, where each instance $\left(\vec{w}^{(t)}, \vec{e}^{(t)}\right)$ is drawn from a distribution $\dist^{(t)}$. Suppose all
  vertex weights are in $(0,1]$ and every pair of vertex weights has a
  $\kappa$-bounded joint distribution. Algorithm~\ref{alg:1_d_online} with input
  parameter $\lambda = \frac{1}{H}\sqrt{\frac{\ln \left(B \sqrt{\numfunctions} \kappa \ln n\right)}{\numfunctions}}$ has regret bounded by $\tilde{O}\left(n^4 H
  \sqrt{T}\right).$
\end{theorem}
\begin{proof}
  In Theorem~\ref{thm:MWIS_upper}, we show that with probability $1-\zeta$ over $\sample \sim \bigtimes_{t = 1}^\numfunctions \dist^{(t)}$, $u$ is \[\left(\frac{1}{\sqrt{\numfunctions}\kappa \ln n}, O\left(n^4
  \sqrt{T\ln (n/\zeta)}\right)\right)\text{-dispersed}\] with respect to $\sample$. Therefore, by
  Theorem~\ref{thm:1_d_online}, with probability at least $1-\zeta$, the
  expected regret of Algorithm~\ref{alg:1_d_online} is at most $\tilde{O}\left(H
  n^4\sqrt{T}\right).$ If this regret bound does not hold, then the regret is at
  most $HT$, but this only happens with probability $\zeta$. Setting $\zeta =
  1/\sqrt{T}$ gives the result.
\end{proof}

\begin{theorem}[Differentially private online optimization in the full information setting]\label{thm:MWIS_full_info_DP}
  Let \[u\left(\vec{w}^{(1)}, \vec{e}^{(1)}, \cdot \right), \dots, u\left(\vec{w}^{(\numfunctions)}, \vec{e}^{(\numfunctions)}, \cdot \right)\] be the set of functions observed by Algorithm~\ref{alg:1_d_online}, where each instance $\left(\vec{w}^{(t)}, \vec{e}^{(t)}\right)$ is drawn from a distribution $\dist^{(t)}$. Suppose all
  vertex weights are in $(0,1]$ and every pair of vertex weights has a
  $\kappa$-bounded joint distribution. Algorithm~\ref{alg:1_d_online} with input
  parameter $\lambda = \frac{\epsilon}{4H\sqrt{2T \ln \left(1/\delta\right)}}$ is
  $\left(\epsilon, \delta\right)$-differentially private and has regret bounded by $\tilde
  O\left(H\sqrt{T}\left(1/\epsilon + n^4\right)\right)$.
\end{theorem}
\begin{proof}
  The proof is exactly the same as the proof of
  Theorem~\ref{thm:MWIS_full_info}, except we  rely on
  Theorem~\ref{thm:online_1d_private} instead of Theorem~\ref{thm:1_d_online} to
  obtain the regret bound.
\end{proof}

\begin{theorem}[Bandit feedback]\label{thm:MWIS_bandit}
  Let $u\left(\vec{w}^{(1)}, \vec{e}^{(1)}, \cdot \right), \dots, u\left(\vec{w}^{(\numfunctions)}, \vec{e}^{(\numfunctions)}, \cdot \right)$ be a sequence of functions where each instance $\left(\vec{w}^{(t)}, \vec{e}^{(t)}\right)$ is drawn from a distribution $\dist^{(t)}$. Suppose all vertex weights are in $(0,1]$ and every pair of
  vertex weights has a $\kappa$-bounded joint distribution. There is a
  bandit-feedback online optimization algorithm with regret bounded by $\tilde O\left(H
  T^{2/3}\left(\sqrt{B} + n^4\right)\right)$.
\end{theorem}
\begin{proof}
In Theorem~\ref{thm:MWIS_upper} with $\alpha = 2/3$, we show that with probability $1-\zeta$ over $\sample \sim \bigtimes_{t = 1}^\numfunctions \dist^{(t)}$, $u$ is \[\left(\frac{1}{\numfunctions^{1/3}\kappa \ln n}, O\left(n^4
  \numfunctions^{2/3} \sqrt{\ln (n/\zeta)}\right)\right)\text{-dispersed}\] with respect to $\sample$. Therefore, by
  Theorem~\ref{thm:banditRegret} with $R = B$, with probability at least $1-\zeta$, there is a bandit-feedback algorithm with expected regret at most $\tilde O\left(H
  T^{2/3}\left(\sqrt{B} + n^4\right)\right)$. If this regret bound does not hold, then the regret is at
  most $HT$, but this only happens with probability $\zeta$. Setting $\zeta = 1/T^{1/3}$ gives the result.
\end{proof}

\begin{lem}[\citep{Gupta17:PAC}]\label{lem:MWIS_sample}
  Let $\left\{\left(\vec{w}^{\left(1\right)}, \vec{e}^{\left(1\right)}\right), \dots, \left(\vec{w}^{\left(\numfunctions\right)}, \vec{e}^{\left(\numfunctions\right)}\right)\right\} \sim \dist^\numfunctions$ be a set of samples. Then with probability at least
  $1-\zeta$, for all $\rho > 0$, \[\left|\frac{1}{\numfunctions} \sum_{t = 1}^T u\left(\vec{w}^{(t)}, \vec{e}^{(t)}, \rho\right) - \E_{\left(\vec{w}, \vec{e}\right) \sim
  \dist}\left[u\left(\vec{w}, \vec{e}, \rho\right)\right]\right| = O\left(H\sqrt{\frac{1}{\numfunctions}\log
  \frac{n}{\zeta}}\right).\]
\end{lem}

\subsection{Knapsack algorithm configuration}\label{app:knapsack}

In the knapsack problem, the input is a knapsack capacity $C$ and a set of $n$ items
$i$ each with a value $v_i$ and a size $s_i$. The goal is to determine a set $I
\subseteq \left\{1, \dots, n\right\}$ with maximium total value $\sum_{i \in I} v_i$ such
that $\sum_{i \in I} s_i \leq C$. We assume that $v_i \in (0,1]$ for all $i \in [n]$.
\citet{Gupta17:PAC} suggest the family of algorithms parameterized by $\rho \in
[0, \infty)$ where each algorithm returns the better of the following two
solutions:
\begin{itemize}
\item Greedily pack items in order of nonincreasing value $v_i$ subject to feasibility.
\item Greedily pack items in order of $v_i/s_i^{\rho}$ subject to feasibility.
\end{itemize}
It is well-known that the algorithm with $\rho= 1$ achieves a 2-approximation.
We consider the family of algorithms where we restrict the parameter $\rho$ to
lie in the interval $\configs = [0,B]$ for some $B \in \R$.
We model the distribution $\dist$ over knapsack problem instances as a
distribution over value-size-capacity tuples $\left(\vec{v}, \vec{s}, C\right)
\in (0,1]^n \times \R^n\times \R$. For a sample of knapsack problem instances
$\sample = \left\{\left(\vec{v}^{\left(t\right)}, \vec{s}^{\left(t\right)},
C^{\left(t\right)}\right)\right\}_{t=1}^\numfunctions$, we denote the value and size of item $i$ under
instance $\left(\vec{v}^{\left(t\right)}, \vec{s}^{\left(t\right)}, C^{(t)}\right)$ as $v_i^{\left(t\right)}$ and
$s_i^{\left(t\right)}$. We use the notation $u\left(\vec{v}, \vec{s}, C, \rho\right)$ to denote the
total value of the items returned by the algorithm parameterized by $\rho$ given input $(\vec{v}, \vec{s}, C)$.

\citet{Gupta17:PAC} prove the following fact about the function $u$.
\begin{lem}[\citep{Gupta17:PAC}]
  Given a set of samples $\left\{\left(\vec{v}^{\left(t\right)}, \vec{s}^{\left(t\right)},
  C^{\left(t\right)}\right)\right\}_{t=1}^\numfunctions$, the function \[\sum_{t = 1}^T u\left(\vec{v}^{(t)}, \vec{s}^{(t)}, C^{(t)}, \cdot\right)\] is
  piecewise constant. It has at most $\numfunctions n^2$ constant pieces
  and the boundaries between constant pieces have the form \[\frac{\ln\left(v^{\left(t\right)}_i\right) -
  \ln\left(v^{\left(t\right)}_j\right)}{\ln\left(s^{\left(t\right)}_i\right) -
  \ln\left(s^{\left(t\right)}_j\right)}\] for all $t \in [\numfunctions]$ and $i,j \in [n]$.
\end{lem}

We now prove that dispersion holds under natural conditions.

\begin{theorem}\label{thm:knapsack_upper}
  Suppose that every pair of item values
  has a $\kappa$-bounded joint distribution, every item size is in $[1,W]$, and the item values are independent from the item sizes. For any tuple $\left(\vec{v}, \vec{s}, C\right)$, $u(\vec{v}, \vec{s}, C, \cdot)$ is piecewise 0-Lipschitz. With probability at least $1-\zeta$ over $\sample \sim \times_{t = 1}^T\dist^{(t)}$, for any $\alpha \geq 1/2$, $u$ is
  $\left(\frac{1}{\numfunctions^{1-\alpha}\kappa\ln W}, O\left(n^2\numfunctions^\alpha \sqrt{\ln
  \frac{n}{\zeta}}\right)\right)$-dispersed with respect to $\sample$.
\end{theorem}

\begin{proof}
  Consider the following partitioning of the boundaries:
  \[
  \cB_{i,j} = \left\{\frac{\ln\left(v^{\left(t\right)}_i\right) - \ln\left(v^{\left(t\right)}_j\right)}{\ln\left(s^{\left(t\right)}_i\right) - \ln\left(s^{\left(t\right)}_j\right)} \ : \ t \in [\numfunctions]\right\}
  \]
  for all $\left(i,j\right) \in {[n] \choose 2}$. The points making up each
  $\cB_{i,j}$ are all independent since they come from different samples. Since the values are supported on $(0,1]$ and have pairwise
  $\kappa$-bounded joint densities, Lemma~\ref{lem:ln_difference_bounded} tells
  us that $\ln\left(v^{\left(t\right)}_i\right) - \ln\left(v^{\left(t\right)}_j\right)$ has a
  $\kappa/2$-bounded distribution for all $i,j\in[n]$ and $t \in [\numfunctions]$. Also,
  since $\left|\ln\left(s^{\left(t\right)}_i\right) -
  \ln\left(s^{\left(t\right)}_j\right)\right| \leq \ln W$ and the numerator of each element in $\cB_{i,j}$ is independent from its denominator, Lemma~\ref{lem:ratio} implies that the
  elements of each $\cB_{i,j}$ come from $\frac{\kappa \ln W}{2}$-bounded
  distributions. Applying Lemma~\ref{lem:dispersion} with $M=T$ and $P \leq n^2$ gives the result, since
  each bin $\cB_{i,j}$ contains $\numfunctions$ elements and there are at most $n^2$ bins.
\end{proof}

\begin{theorem}[Differential privacy]
  Given a set of samples \[\sample = \left\{\left(\vec{v}^{\left(1\right)}, \vec{s}^{\left(1\right)}, C^{(1)}\right), \dots, \left(\vec{v}^{\left(\numfunctions\right)}, \vec{s}^{\left(\numfunctions\right)}, C^{(\numfunctions)}\right)\right\} \sim \dist^T,\] suppose Algorithm~\ref{alg:1dEfficient} takes as input the function $\sum_{t = 1}^T u\left(\vec{v}^{\left(t\right)}, \vec{s}^{\left(t\right)}, C^{(t)}, \cdot\right)$ and the set of intervals over which this function is piecewise constant. Suppose that every pair of item values
  has a $\kappa$-bounded joint value distribution, every item size is in $[1,W]$, and the item values are independent from the item sizes. Algorithm~\ref{alg:1dEfficient} returns a parameter
  $\hat{\rho}$ such that with probability at least $1-\zeta$ over the draw of
  $\sample$,
  \[
    \E[u\left(\vec{v}, \vec{s}, C, \hat{\rho}\right)]
    \geq
    \max_{\rho \in [0,B]}\E[u\left(\vec{v}, \vec{s}, C, \rho\right)] - O\left(\frac{H}{\numfunctions\epsilon} \log \frac{B\numfunctions\kappa\ln W}{\zeta} + Hn^2\sqrt{\frac{\log\left(n/\zeta\right)}{\numfunctions}}\right).
  \]
\end{theorem}
\begin{proof}
  The theorem statement follows from Theorems~\ref{thm:cexpmutility} and
  \ref{thm:knapsack_upper} and Lemma~\ref{lem:knapsack_sample}.
\end{proof}

\begin{theorem}[Full information online optimization]\label{thm:knapsack_full_info}
  Let \[u\left(\vec{v}^{(1)}, \vec{s}^{(1)}, C^{(1)}, \cdot \right), \dots, u\left(\vec{v}^{(\numfunctions)}, \vec{s}^{(\numfunctions)}, C^{(\numfunctions)}, \cdot \right)\] be the set of functions observed by Algorithm~\ref{alg:1_d_online}, where each instance $\left(\vec{v}^{(t)}, \vec{s}^{(t)}, C^{(t)}\right)$ is drawn from a distribution $\dist^{(t)}$. Suppose that every pair of item values
  has a $\kappa$-bounded joint distribution, every item size is in $[1,W]$, and the item values are independent from the item sizes.
  Algorithm~\ref{alg:1_d_online} with input parameter $\lambda = \frac{1}{H} \sqrt{\frac{\ln \left(B \sqrt{\numfunctions} \kappa \ln W\right)}{\numfunctions}}$
  has regret bounded by $\tilde{O}\left(H n^2 \sqrt{T}\right).$
\end{theorem}
\begin{proof}
  In Theorem~\ref{thm:knapsack_upper}, we show that with probability $1-\zeta$ over $\sample \sim \times_{t = 1}^T\dist^{(t)}$, $u$ is
  \[\left(\frac{1}{\sqrt{\numfunctions}\kappa\ln W}, O\left(n^2\sqrt{T\ln
  \frac{n}{\zeta}}\right)\right)\text{-dispersed}\] with respect to $\sample$. Therefore, by
  Theorem~\ref{thm:1_d_online}, with probability at least $1-\zeta$, the
  expected regret of Algorithm~\ref{alg:1_d_online} is at most $\tilde{O}\left(H
  n^2\sqrt{T}\right).$ If this regret bound does not hold, then the regret is at
  most $HT$, but this only happens with probability $\zeta$. Setting $\zeta =
  1/\sqrt{T}$ gives the result.
\end{proof}

\begin{theorem}[Differentially private online optimization in the full information setting]\label{thm:knapsack_full_info_DP}
  Let \[u\left(\vec{v}^{(1)}, \vec{s}^{(1)}, C^{(1)}, \cdot \right), \dots, u\left(\vec{v}^{(\numfunctions)}, \vec{s}^{(\numfunctions)}, C^{(\numfunctions)}, \cdot \right)\] be the set of functions observed by Algorithm~\ref{alg:1_d_online}, where each instance $\left(\vec{v}^{(t)}, \vec{s}^{(t)}, C^{(t)}\right)$ is drawn from a distribution $\dist^{(t)}$. Suppose that every pair of item values
  has a $\kappa$-bounded joint distribution, every item size is in $[1,W]$, and the item values are independent from the item sizes.
  Algorithm~\ref{alg:1_d_online} with input parameter $\lambda =
  \frac{\epsilon}{4H\sqrt{2T \ln \left(1/\delta\right)}}$ is $\left(\epsilon,
  \delta\right)$-differentially private and has regret bounded by $\tilde{O}\left( H
  \sqrt{T} \left(1/\epsilon + n^2\right)\right).$
\end{theorem}
\begin{proof}
  The proof is exactly the same as the proof of
  Theorem~\ref{thm:knapsack_full_info}, except we  rely on
  Theorem~\ref{thm:online_1d_private} instead of Theorem~\ref{thm:1_d_online} to
  obtain the regret bound.
\end{proof}

\begin{theorem}[Bandit feedback]\label{thm:knapsack_full_info_DP}
  Let $u\left(\vec{v}^{(1)}, \vec{s}^{(1)}, C^{(1)}, \cdot \right), \dots, u\left(\vec{v}^{(\numfunctions)}, \vec{s}^{(\numfunctions)}, C^{(\numfunctions)}, \cdot \right)$ be a sequence of functions where each instance $\left(\vec{v}^{(t)}, \vec{s}^{(t)}, C^{(t)}\right)$ is drawn from a distribution $\dist^{(t)}$. Suppose that every pair of item values
  has a $\kappa$-bounded joint distribution, every item size is in $[1,W]$, and the item values are independent from the item sizes. There is a bandit-feedback online optimization algorithm with regret bounded
  by $\tilde{O}\left(HT^{2/3}\left(\sqrt{B} + n^2\right)\right).$
\end{theorem}

\begin{proof}
In Theorem~\ref{thm:knapsack_upper} with $\alpha = 2/3$, we show that with probability $1-\zeta$ over $\sample \sim \bigtimes_{t = 1}^\numfunctions \dist^{(t)}$, $u$ is \[\left(\frac{1}{\numfunctions^{1/3}\kappa \ln W}, O\left(n^2
  \numfunctions^{2/3} \sqrt{\ln (n/\zeta)}\right)\right)\text{-dispersed}\] with respect to $\sample$. Therefore, by
  Theorem~\ref{thm:banditRegret} with $R = B$, with probability at least $1-\zeta$, there is a bandit-feedback algorithm with expected regret at most $\tilde O\left(H
  T^{2/3}\left(\sqrt{B} + n^2\right)\right)$. If this regret bound does not hold, then the regret is at
  most $HT$, but this only happens with probability $\zeta$. Setting $\zeta = 1/T^{1/3}$ gives the result.
\end{proof}

\begin{lem}[\citep{Gupta17:PAC}]\label{lem:knapsack_sample}
  Let $\left\{\left(\vec{v}^{\left(t\right)}, \vec{s}^{\left(t\right)},
  C^{\left(t\right)}\right)\right\}_{t=1}^\numfunctions$ be $\numfunctions$ knapsack problem instances sampled from
  $\dist$. Then with probability at least $1-\zeta$, for all $\rho\geq 0$,
  \[
  \left|\frac{1}{T}\sum_{t = 1}^T u\left(\vec{v}^{(t)}, \vec{s}^{(t)}, C^{(t)}, \rho\right) - \E_{\left(\vec{v}, \vec{s}, C\right) \sim \dist}\left[u\left(\vec{v},
  \vec{s}, C, \rho\right)\right]\right| = O\left(H\sqrt{\frac{\log \left(n/\zeta\right)}{\numfunctions}}\right).
  \]
\end{lem}

\subsection{Outward rotation rounding algorithms}\label{app:OWR}

\begin{algorithm}
\caption{SDP rounding algorithm with rounding function $r: \R \to [-1,1]$}\label{alg:GW}
\begin{algorithmic}[1]
\Require Matrix $A \in \R^{n \times n}$.
\State Solve the SDP
\[\text{maximize } \sum_{i,j \in [n]} a_{ij}\left\langle \vec{u}_i , \vec{u}_j\right\rangle \qquad \text{subject to }\vec{u}_i \in S^{n-1}\]
for the optimal embedding $U = \left\{\vec{u}_1, \dots, \vec{u}_n\right\}$.
\State Draw $\vec{Z} \sim \mathcal{N}_n$.\label{step:draw}
\State For all $i \in [n]$, with probability $\left(1 + r\left(\left\langle \vec{Z}, \vec{u}_i\right\rangle\right)\right)/2$, set $z_i = 1$ and with probability $\left(1 - r\left(\left\langle \vec{Z}, \vec{u}_i\right\rangle\right)\right)/2$, set $z_i = -1$.
 \label{step:GWthreshold}
\Ensure $z_1, \dots, z_n$.
\end{algorithmic}
\end{algorithm}

\begin{algorithm}
\caption{SDP rounding algorithm using $\gamma$-outward rotation}\label{alg:owr}
\begin{algorithmic}[1]
\Require Matrix $\qp \in \mathbb{R}^{n \times n}$
\State Solve the SDP
\[\text{maximize } \sum_{i,j \in [n]} a_{ij}\left\langle \vec{u}_i , \vec{u}_j\right\rangle \qquad \text{subject to }\vec{u}_i \in S^{n-1}\]
 to obtain the optimal embedding $U = \left\{\vec{u}_1, \dots, \vec{u}_n\right\}$.
\State Define a new embedding $\vec{u}_i'$ in $\mathbb{R}^{2n}$ as follows. The first $n$ co-ordinates correspond to $\vec{u}_i \cos \gamma$ and the following $n$ co-ordinates are set to $0$ except the $\left(n+i\right)$th co-ordinate which is set to $\sin \gamma$.
\State Choose a random vector $\vec{Z} \in \R^{2n}$ according to the $2n$-dimensional Gaussian distribution.
\State For each decision variable $z_i$, assign $z_i=\sign \left(\left\langle \vec{u}_i', \vec{Z} \right\rangle\right).$
\Ensure $z_1, \hdots, z_n$.
\end{algorithmic}
\end{algorithm}

\owrDispersed*
\begin{proof}
  \citet{Balcan17:Learning} prove that
  the function $\sum_{t = 1}^T \uowr\left(A^{(t)},\vec{Z}^{(t)},\cdot \right)$ consists of $n\numfunctions+1$ piecewise constant
  components. The discontinuities are of the form \[\tan^{-1} \left(- \frac{\left\langle \vec{u}_i^{\left(j\right)} , \vec{Z}^{\left(j\right)}[1, \dots,n]\right\rangle}{Z^{\left(j\right)}[n+i] }\right)\] for each $\vec{u}_i^{\left(j\right)}$ in
  the optimal SDP embedding of each $\qp^{\left(j\right)}$. We show that the critical
  points are uniform random variables and thus are dispersed.

  For an IQP instance $\qp$ and its SDP embedding $\left\{\vec{u}_1, \dots,
  \vec{u}_n\right\}$, since each $\vec{u}_i$ is a unit vector, we know that
  $-\left\langle \vec{u}_i  , \vec{Z}[1, \dots, n]\right\rangle$ is a standard normal
  random variable. Therefore, $- \frac{\left\langle \vec{u}_i  , \vec{Z}[1, \dots,
  n]\right\rangle}{Z[n+i] }$ is a Cauchy random variable and $\tan^{-1} \left(-
  \frac{\left\langle \vec{u}_i  , \vec{Z}[1, \dots, n]\right\rangle}{Z[n+i] }\right)$
  is a uniform random variable in the range $\left[-\frac{\pi}{2},
  \frac{\pi}{2}\right]$ \citep{Tijms12:Understanding, Biagini16:Elements}.

  Define
  \[
    \gamma_i^{(j)} = \tan^{-1} \left(- \frac{\left\langle \vec{u}_i^{\left(j\right)} , \vec{Z}^{\left(j\right)}[1, \dots,n]\right\rangle}{Z^{\left(j\right)}[n+i] }\right).
  \]
  For any two vectors $\vec{u}_i^{\left(j\right)}$ and $\vec{u}_i^{\left(k\right)}$ from different SDP
  embeddings, the random variables $\gamma_i^{\left(j\right)}$ and $\gamma_i^{\left(k\right)}$ are
  independent uniform random variables in $\left[-\pi/2, \pi/2\right]$. Therefore, we
  define the sets $\cB_1, \dots, \cB_n$ such that $\cB_i = \left\{\gamma_i^{\left(1\right)},
  \dots, \gamma_i^{\left(\numfunctions\right)}\right\}$. Within each $\cB_i$, the variables are
  independent. Therefore, by Lemma~\ref{lem:dispersion} with $P = n$, $M = \max \left|\cB_i\right|  = \numfunctions$, and $\kappa = \pi$, the theorem statement holds.
\end{proof}

\begin{theorem}[Differential privacy]\label{thm:owr_DP}
  Given a set of samples $\sample = \left\{\left(A^{\left(1\right)}, \vec{Z}^{\left(1\right)}\right), \dots, \left(A^{\left(\numfunctions\right)}, \vec{Z}^{\left(\numfunctions\right)}\right)\right\} \sim \left(\mathcal{D}\times \mathcal{N}_{2n}\right)^\numfunctions$, suppose Algorithm~\ref{alg:1dEfficient} takes as input the function $\sum_{t = 1}^T \uowr\left(A^{\left(t\right)}, \vec{Z}^{\left(t\right)}, \cdot\right)$ and the set of intervals over which this function is piecewise constant.  
  Algorithm~\ref{alg:1dEfficient} returns a parameter $\hat{\gamma}$ such that
  with probability at least $1-\zeta$ over the draw of $\sample$,
  \[
    \E_{\qp, \vec{Z} \sim \dist \times \mathcal{N}_{2n}}[\uowr\left(\qp, \vec{Z}, \hat{\gamma}\right)]
    \geq
    \max_{\gamma \in \left[-\frac{\pi}{2}, \frac{\pi}{2}\right]}\E_{\qp, \vec{Z} \sim \dist \times \mathcal{N}_{2n}}[\uowr\left(\qp, \vec{Z}, \gamma\right)] - O\left(\frac{H}{\numfunctions\epsilon} \log \frac{\numfunctions}{\zeta} + Hn\sqrt{\frac{1}{\numfunctions}\log\frac{n}{\zeta}}\right).
  \]
\end{theorem}
\begin{proof}
  The theorem statement follows from Theorems~\ref{thm:cexpmutility} and
  \ref{thm:owr_dispersed} and Lemma~\ref{lem:owr_generalization}.
\end{proof}

\begin{theorem}[Full information online optimization]\label{thm:owr_full_info}
Let $\uowr\left(A^{(1)}, \vec{Z}^{(1)}, \cdot \right), \dots,
  \uowr\left(A^{(\numfunctions)}, \vec{Z}^{(\numfunctions)}, \cdot \right)$
  be the set of functions observed by Algorithm~\ref{alg:1_d_online}, where each
  vector $\vec{Z}^{(t)}$ is drawn from $\cN_{2n}$.  
  Algorithm~\ref{alg:1_d_online}
  with input parameter $\lambda = \frac{1}{H}\sqrt{\frac{\ln \left(\pi \sqrt{\numfunctions}\right)}{\numfunctions}}$ has regret bounded by
  $\tilde{O}\left(H n \sqrt{T} \right).$
\end{theorem}
\begin{proof}
   In Theorem~\ref{thm:MWIS_upper}, we show that with probability $1-\zeta$ over $\vec{Z}^{(1)}, \dots, \vec{Z}^{(\numfunctions)} \sim \cN_{2n}$, $\uowr$ is $\left(\frac{1}{\sqrt{T}}, O\left(n
  \sqrt{T\log(n/\zeta)}\right)\right)$-dispersed with
  respect to $\sample = \{(\qp^{(t)}, \vec{Z}^{(t)})\}_{t = 1}^\numfunctions$. Therefore, by Theorem~\ref{thm:1_d_online}, with
  probability at least $1-\zeta$, the expected regret of
  Algorithm~\ref{alg:1_d_online} is at most $\tilde{O}\left(H n\sqrt{T}\right).$
  If this regret bound does not hold, then the regret is at most $HT$, but this
  only happens with probability $\zeta$. Setting $\zeta = 1/\sqrt{T}$ gives the
  result.
\end{proof}

\begin{theorem}[Differentially private online optimization in the full information setting]\label{thm:owr_full_info_DP} $\,$
  Let $\uowr\left(A^{(1)}, \vec{Z}^{(1)}, \cdot \right), \dots,
  \uowr\left(A^{(\numfunctions)}, \vec{Z}^{(\numfunctions)}, \cdot \right)$
  be the set of functions observed by Algorithm~\ref{alg:1_d_online}, where each
  vector $\vec{Z}^{(t)}$ is drawn from $\cN_{2n}$. Algorithm~\ref{alg:1_d_online}
  with input parameter $\lambda = \frac{\epsilon}{4H\sqrt{2T \ln \left(1/\delta\right)}}$
  is $\left(\epsilon, \delta\right)$-differentially private and has regret bounded by
  $\tilde{O}\left(H\sqrt{T}\left(1/\epsilon + n\right)\right).$
\end{theorem}
\begin{proof}
  The proof is exactly the same as the proof of Theorem~\ref{thm:owr_full_info},
  except we  rely on Theorem~\ref{thm:online_1d_private} instead of
  Theorem~\ref{thm:1_d_online} to obtain the regret bound.
\end{proof}

\begin{theorem}[Bandit feedback]\label{thm:owr_bandit}
Let $\uowr\left(A^{(1)}, \vec{Z}^{(1)}, \cdot \right), \dots, 
\uowr\left(A^{(\numfunctions)},
  \vec{Z}^{(\numfunctions)}, \cdot \right)$ be a sequence of functions where each
  vector $\vec{Z}^{(t)}$ is drawn from $\cN_{2n}$.
  There is a bandit-feedback online optimization algorithm with regret bounded by
  $\tilde{O}\left(HnT^{2/3}\right).$
\end{theorem}
\begin{proof}
  The proof is exactly the same as the proof of Theorem~\ref{thm:owr_full_info},
  except we  rely on Theorem~\ref{thm:banditRegret} instead of
  Theorem~\ref{thm:1_d_online} to obtain the regret bound. In this case,
  $\configs = [0,\pi/2]$ and we take $\zeta = 1/T^{1/3}$.

  In Theorem~\ref{thm:owr_dispersed} with $\alpha = 2/3$, over $\vec{Z}^{(1)}, \dots, \vec{Z}^{(\numfunctions)} \sim
  \mathcal{N}_{2n}$, for any $\qp^{(1)}, \dots, \qp^{(\numfunctions)} \in \R^{n \times n}$, $\uowr$ is $\left(\frac{1}{\numfunctions^{1/3}}, O\left(n
  \numfunctions^{1/3}\sqrt{\log(n/\zeta)}\right)\right)$-dispersed with
  respect to $\sample = \{(\qp^{(t)}, \vec{Z}^{(t)})\}_{t = 1}^\numfunctions$. Therefore, by
  Theorem~\ref{thm:banditRegret} with $R = \pi/2$, with probability at least $1-\zeta$, there is a bandit-feedback algorithm with expected regret at most $\tilde{O}\left(HnT^{2/3}\right)$. If this regret bound does not hold, then the regret is at
  most $HT$, but this only happens with probability $\zeta$. Setting $\zeta = 1/T^{1/3}$ gives the result.
\end{proof}

\begin{lem}\label{lem:owr_generalization}[\cite{Balcan17:Learning}]
  Let $\sample = \left\{\left(\qp^{\left(1\right)}, \vec{Z}^{\left(1\right)}\right), \dots,
  \left(\qp^{\left(\numfunctions\right)}, \vec{Z}^{\left(\numfunctions\right)}\right)\right\}$ be $\numfunctions$ tuples sampled from
  $\dist \times \mathcal{N}_{2n}$. With probability at least $1-\zeta$, for
  all $ \gamma \in [-\pi/2, \pi/2]$,
  \[
  \left|\frac{1}{T} \sum_{t = 1}^\numfunctions \uowr\left(\qp^{(t)}, \vec{Z}^{(t)}, \gamma\right) - \E_{\qp, \vec{Z} \sim \dist \times \mathcal{N}_{2n} }\left[\uowr\left(A, \vec{Z}, \gamma\right)\right]\right| < O\left(H\sqrt{\frac{\log \left(n/\zeta\right)}{\numfunctions}}\right).
  \]
\end{lem}

\subsection{$s$-linear rounding algorithms}
We make the following assumption, which is without loss of generality up to scaling, on the input matrices $\qp^{(1)}, \dots, \qp^{(T)}$.

\begin{assumption}\label{assumption:lower_bound} There exists a constant $H \in \R$ such that for any matrices $\qp^{\left(1\right)}, \dots, \qp^{\left(\numfunctions\right)}$ given as input to the algorithms in this paper, $\sum_{i,j} \left|a_{ij}^{(t)}\right| \in [1, H]$ for all $t \in [\numfunctions]$.\end{assumption}

\begin{figure}
\begin{center}
  \includegraphics[width=0.4\textwidth] {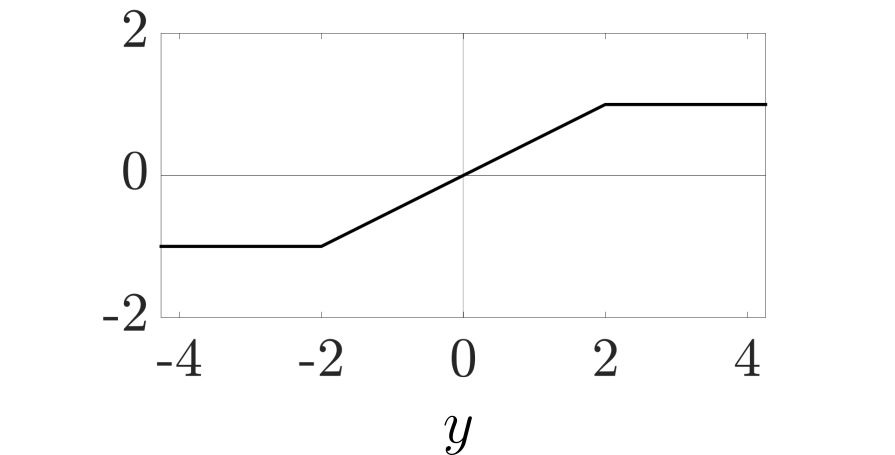}
  \caption{A graph of the 2-linear function $\phi_2$.}\label{fig:slin2}
  \end{center}
\end{figure}

\slinDispersed*
\begin{proof}
 \citet{Balcan17:Learning} proved that $\sum_{t = 1}^T \uslin\left(\qp^{(t)}, \vec{Z}^{(t)}, s\right)$ has the form \[\sum_{t = 1}^T \uslin\left(\qp^{(t)}, \vec{Z}^{(t)}, s\right) = \sum_{t = 1}^T \left(\sum_{i = 1}^n \left(a_{ii}^{(t)}\right)^2 + \sum_{i \not = j} a_{ij}^{(t)} \phi_s\left(\left\langle\vec{Z}^{(t)}, \vec{u}_i^{(t)}\right\rangle\right) \cdot \phi_s\left(\left\langle \vec{Z}^{(t)}, \vec{u}_j^{(t)} \right\rangle\right)\right)\] and the function $\sum_{t = 1}^T \uslin\left(\qp^{(t)}, \vec{Z}^{(t)}, \cdot\right)$ is made up of $\numfunctions n+1$ piecewise components of the form $\frac{a}{s^2} + \frac{b}{s} + c$ for some constants $a,b,c \in \R$.
  Let $\cB_1, \dots, \cB_n$ be $n$ sets of random variables such that $\cB_i =
  \left\{\left|\left\langle \vec{u}_i^{\left(t\right)} , \vec{Z}^{\left(t\right)}\right\rangle\right| : t \in
  [\numfunctions]\right\}.$  \citet{Balcan17:Learning} proved that $\bigcup_{t = 1}^n \cB_t$ are all of the boundaries dividing the domain of $\sum_{t = 1}^T \uslin\left(\qp^{(t)}, \vec{Z}^{(t)}, \cdot\right)$ into pieces over which the function is differentiable. Also, within each $\cB_i$, the variables
  are all absolute values of independent standard Gaussians, since for any unit vector
  $\vec{u}$ and any $\vec{Z} \sim \mathcal{N}_n$, $\langle\vec{u}, \vec{Z}\rangle$ is a
  standard Gaussian. When $Z$ is a Gaussian random variable, $|Z|$ is drawn from
  a $\left(4/5\right)$-bounded distribution. Therefore, the dispersion bound follows from
  Lemma~\ref{lem:dispersion} with $P = n$ and $M= \max\left|\cB_i\right| = \numfunctions$.
  
The main challenge in this proof is showing that for any $t \in [T]$, $\uslin\left(\qp^{(t)}, \vec{Z}^{(t)}, \cdot\right)$ is Lipschitz even when $s$ approaches zero. We show that with probability at least $1-\zeta$, for all $t \in [T]$, $\uslin\left(\qp^{(t)}, \vec{Z}^{(t)}, \cdot\right)$ is constant on the interval $\left(0, 16M\numfunctions^3n^5/\zeta^3\right)$. This way, we know that the derivative of $\uslin\left(\qp^{(t)}, \vec{Z}^{(t)}, \cdot\right)$ is zero as $s$ goes to zero, not infinity.

Let $s_0$ be the smallest boundary between piecewise components of any function  $\uslin\left(\qp^{(t)}, \vec{Z}^{(t)}, \cdot\right)$. In other words, for all $t \in [T]$, when $s \in \left(0, s_0\right)$, $\uslin\left(\qp^{(t)}, \vec{Z}^{(t)}, s\right)$ is differentiable and $\uslin\left(\qp^{(t)}, \vec{Z}^{(t)}, \cdot\right)$ is not differentiable at $s_0$. For all $s \in \left(0, s_0\right)$, all $i\in [n]$, and all $t \in [\numfunctions]$, $\left|\left\langle \vec{u}_i^{\left(t\right)}, \vec{Z}^{\left(t\right)}\right\rangle \right| > s$. This means that $\phi_s\left(\left\langle \vec{u}_i^{\left(t\right)}, \vec{Z}^{\left(t\right)}\right\rangle\right) = \pm 1.$ Therefore, for any $t \in [T]$, the derivative of $\uslin\left(\qp^{(t)}, \vec{Z}^{(t)}, \cdot\right)$ is zero on the interval $\left(0, s_0\right)$. In Lemma~\ref{lem:s0_bound}, we prove that with probability $1-\zeta/2$, $s_0 \geq \frac{\zeta}{4n\numfunctions}$.

We now bound the maximum absolute value of the derivative of any $\uslin\left(\qp^{(t)}, \vec{Z}^{(t)}, \cdot\right)$ for any $s > s_0$ where $\uslin\left(\qp^{(t)}, \vec{Z}^{(t)}, \cdot\right)$ is differentiable.
We know that \begin{align*}\frac{d}{ds}\uslin\left(\qp^{(t)}, \vec{Z}^{(t)}, s\right) &= \frac{d}{ds}\left(\sum_{i = 1}^n \left(a_{ii}^{(t)}\right)^2 + \sum_{i \not = j} a_{ij}^{(t)} \phi_s\left(\left\langle\vec{Z}^{(t)}, \vec{u}_i^{(t)}\right\rangle\right) \cdot \phi_s\left(\left\langle \vec{Z}^{(t)}, \vec{u}_j^{(t)} \right\rangle\right)\right)\\
&=  \sum_{i \not = j} a_{ij}^{(t)} \frac{d}{ds}\left( \phi_s\left(\left\langle\vec{Z}^{(t)}, \vec{u}_i^{(t)}\right\rangle\right) \cdot \phi_s\left(\left\langle \vec{Z}^{(t)}, \vec{u}_j^{(t)} \right\rangle\right)\right).\end{align*}
Therefore, we only need to bound $\left|\frac{d}{ds}\left( \phi_s\left(\left\langle\vec{Z}^{(t)}, \vec{u}_i^{(t)}\right\rangle\right) \cdot \phi_s\left(\left\langle \vec{Z}^{(t)}, \vec{u}_j^{(t)} \right\rangle\right)\right)\right|$ for all $i,j \in [n]$ and $t \in [\numfunctions]$. We assume that
\[
\max \left\{\left|\left\langle \vec{Z}^{\left(t\right)}, \vec{u}_i^{\left(t\right)}\right\rangle \right| : i \in [n], t \in [\numfunctions]\right\} \leq \sqrt{2 \ln \left(\sqrt{\frac{8}{\pi}}\frac{2n\numfunctions}{\zeta}\right)},
\]
which we know from Lemma~\ref{lemma:bar_s_UB} happens with probability at least $1-\zeta/2$. We also assume that $s_0 \geq \frac{\zeta}{4n\numfunctions}$, which we know from Lemma~\ref{lem:s0_bound} also happens with probability at least $1-\zeta/2$.

To this end, there are only three possible cases:
\begin{itemize}\item \textbf{Case 1:} $\phi_s\left(\left\langle\vec{Z}^{(t)}, \vec{u}_i^{(t)}\right\rangle\right) \cdot \phi_s\left(\left\langle \vec{Z}^{(t)}, \vec{u}_j^{(t)} \right\rangle\right) = \pm 1$
\item \textbf{Case 2:} $\phi_s\left(\left\langle\vec{Z}^{(t)}, \vec{u}_i^{(t)}\right\rangle\right) \cdot \phi_s\left(\left\langle \vec{Z}^{(t)}, \vec{u}_j^{(t)} \right\rangle\right) = \frac{\left\langle\vec{Z}^{(t)}, \vec{u}_i^{(t)}\right\rangle}{s}$
\item \textbf{Case 3:} $\phi_s\left(\left\langle\vec{Z}^{(t)}, \vec{u}_i^{(t)}\right\rangle\right) \cdot \phi_s\left(\left\langle \vec{Z}^{(t)}, \vec{u}_j^{(t)} \right\rangle\right) = \frac{\left\langle\vec{Z}^{(t)}, \vec{u}_i^{(t)}\right\rangle\left\langle\vec{Z}^{(t)}, \vec{u}_j^{(t)}\right\rangle}{s^2}$.
\end{itemize}
In the first case, $\left|\frac{d}{ds}\left( \phi_s\left(\left\langle\vec{Z}^{(t)}, \vec{u}_i^{(t)}\right\rangle\right) \cdot \phi_s\left(\left\langle \vec{Z}^{(t)}, \vec{u}_j^{(t)} \right\rangle\right)\right)\right| = \left|\frac{d}{ds} \pm 1 \right| = 0$. In the second case, \begin{align*}\left|\frac{d}{ds}\left( \phi_s\left(\left\langle\vec{Z}^{(t)}, \vec{u}_i^{(t)}\right\rangle\right) \cdot \phi_s\left(\left\langle \vec{Z}^{(t)}, \vec{u}_j^{(t)} \right\rangle\right)\right)\right| &= \left|\frac{d}{ds}\frac{\left\langle\vec{Z}^{(t)}, \vec{u}_i^{(t)}\right\rangle}{s}\right|\\
&= \left|\frac{\left\langle\vec{Z}^{(t)}, \vec{u}_i^{(t)}\right\rangle}{s^2}\right|\\
&\leq \frac{1}{s^2}\sqrt{2 \ln \left(\sqrt{\frac{8}{\pi}}\frac{2n\numfunctions}{\zeta}\right)} &\text{(Lemma~\ref{lemma:bar_s_UB}})\\
&\leq \frac{16n^2\numfunctions^2}{\zeta^2}\sqrt{2 \ln \left(\sqrt{\frac{8}{\pi}}\frac{2n\numfunctions}{\zeta}\right)}. &\text{(Lemma~\ref{lem:s0_bound}})\end{align*} In the third case, \begin{align*}\left|\frac{d}{ds}\left( \phi_s\left(\left\langle\vec{Z}^{(t)}, \vec{u}_i^{(t)}\right\rangle\right) \cdot \phi_s\left(\left\langle \vec{Z}^{(t)}, \vec{u}_j^{(t)} \right\rangle\right)\right)\right| &= \left|\frac{d}{ds}\frac{\left\langle\vec{Z}^{(t)}, \vec{u}_i^{(t)}\right\rangle\left\langle\vec{Z}^{(t)}, \vec{u}_j^{(t)}\right\rangle}{s^2}\right|\\
&= \left|\frac{2\left\langle\vec{Z}^{(t)}, \vec{u}_i^{(t)}\right\rangle\left\langle\vec{Z}^{(t)}, \vec{u}_j^{(t)}\right\rangle}{s^3}\right|\\
&\leq \frac{2}{s^3}\cdot 2 \ln \left(\sqrt{\frac{8}{\pi}}\frac{2n\numfunctions}{\zeta}\right) &\text{(Lemma~\ref{lemma:bar_s_UB}})\\
&\leq \frac{256n^3\numfunctions^3}{\zeta^3}\ln \left(\sqrt{\frac{8}{\pi}}\frac{2n\numfunctions}{\zeta}\right). &\text{(Lemma~\ref{lem:s0_bound}})\end{align*} Since $\frac{16n^2\numfunctions^2}{\zeta^2}\sqrt{2 \ln \left(\sqrt{\frac{8}{\pi}}\frac{2n\numfunctions}{\zeta}\right)} < \frac{256n^3\numfunctions^3}{\zeta^3}\ln \left(\sqrt{\frac{8}{\pi}}\frac{2n\numfunctions}{\zeta}\right)$, this derivative is maximized in the third case. Noting that $M = \max\left|a_{ij}^{(t)}\right|$, we have that for $s > s_0$, \[\left|\frac{d}{ds} \uslin\left(\qp^{(t)}, \vec{Z}^{(t)}, s\right)\right| \leq n^2M \cdot \frac{256n^3\numfunctions^3}{\zeta^3}\ln \left(\sqrt{\frac{8}{\pi}}\frac{2n\numfunctions}{\zeta}\right) = \frac{256Mn^5\numfunctions^3}{\zeta^3}\ln \left(\sqrt{\frac{8}{\pi}}\frac{2n\numfunctions}{\zeta}\right).\]
\end{proof}

\begin{theorem}[Differential privacy]\label{thm:slinear_DP}
  Given a set of samples
  $\sample = \left\{\left(\qp^{(1)}, \vec{Z}^{(1)}\right), \dots, \left(\qp^{(\numfunctions)}, \vec{Z}^{(\numfunctions)}\right)\right\} \sim \left(\mathcal{D}\times \mathcal{N}_n\right)^\numfunctions$ with $\numfunctions \geq 8H^2 n^2 \ln \frac{8}{\zeta}$, suppose Algorithm~\ref{alg:1dEfficient} takes as input the function $\sum_{t = 1}^T u\left(\qp^{\left(t\right)}, \vec{Z}^{\left(t\right)}, \cdot\right)$ and the set of intervals intersecting $\left(0, \sqrt{2 \ln
  \left(\sqrt{8/\pi}\left(8n\numfunctions/\zeta\right)\right)}\right)$ over which this function is piecewise constant.
  Algorithm~\ref{alg:1dEfficient}
  returns a parameter $\hat{s}$ such that with probability at least $1-\zeta$
  over the draw of $\sample$, \[\E_{\qp, \vec{Z} \sim \dist \times \mathcal{N}_n}\left[\uslin\left(\qp, \vec{Z},
  \hat{s}\right)\right] \geq \max_{s > 0}\E_{\qp, \vec{Z} \sim \dist\times \mathcal{N}_n}[\uslin\left(\qp, \vec{Z}, s\right)] -
  \tilde O\left(\frac{H}{\numfunctions\epsilon} + \frac{Hn}{\sqrt{\numfunctions}}\right).\]
\end{theorem}
\begin{proof}
  First, in Theorem~\ref{thm:slin_dispersed}, we prove that with probability $1-\zeta/4$, the functions \[\uslin(\vec{Z}^{(1)}, \qp^{(1)}, \cdot), \dots, \uslin(\vec{Z}^{(\numfunctions)}, \qp^{(\numfunctions)}, \cdot)\] are piecewise $L$-Lipschitz
  with $L = \frac{16384Mn^5\numfunctions^3}{\zeta^3}\ln \left(\sqrt{\frac{8}{\pi}}\frac{8n\numfunctions}{\zeta}\right)$ and $\uslin$ is $\left(1/\sqrt{\numfunctions},
  O\left(n \sqrt{\numfunctions\log(n/\zeta)}\right)\right)$-dispersed with respect to $\sample$.

  In Lemma~\ref{lem:slin_ball_bound}, we show that with probability $ 1- \zeta/4$, the
  values of $s$ that maximize \[\sum_{t = 1}^\numfunctions \uslin\left(\qp^{(t)}, \vec{Z}^{(t)}, \cdot\right)\] lie within the interval $\left(0, \sqrt{2 \ln
  \left(\sqrt{8/\pi}\left(8n\numfunctions/\zeta\right)\right)}\right)$. Thus, we can restrict Algorithm~\ref{alg:1dEfficient} to searching for a parameter in this range.

We next show that with probability $1 - 3\zeta/4$, \begin{equation}\frac{1}{\numfunctions} \left(\sum_{t = 1}^\numfunctions \uslin\left(\vec{Z}^{(t)}, \qp^{(t)}, \hat{s}\right) - \max_{s > 0} \uslin\left(\vec{Z}^{(t)}, \qp^{(t)}, s\right)\right) = \tilde{O} \left(\frac{H}{\numfunctions\epsilon} + \frac{Hn}{\sqrt{\numfunctions}}\right)\label{eq:slin_utility}\end{equation} If $L < H$, then this follows from Theorem~\ref{thm:cexpmutility}. Otherwise, if $L \geq H$, it follows from Corollary~\ref{cor:expmutilityoptimized}, assuming, as we can with probability $1 - \zeta/4$, that $\log(L) = \tilde{O}(1)$.
 Corollary~\ref{cor:expmutilityoptimized} only holds if $\numfunctions \geq \frac{2H}{weL}$, which is the case when $L \geq H$ because $\frac{2H}{weL} < \frac{1}{w} = \sqrt{\numfunctions} \leq \numfunctions$.
 
In the last
  step of this proof, we show that since $\hat{s}$ is nearly optimal over the
  sample, it is nearly optimal over $\dist$ as well. To do this, we call upon a
  result by \citet{Balcan17:Learning}, which we include here as
  Lemma~\ref{lem:slin_generalization}. It guarantees that with probability at
  least $1-\zeta/4$, for all $s > 0$, $\left|\frac{1}{\numfunctions}\sum_{t = 1}^\numfunctions \uslin\left(\qp^{(t)}, \vec{Z}^{(t)}, s\right) - \E_{\qp, \vec{Z} \sim
  \mathcal{D} \times \mathcal{N}_n}\left[\uslin\left(A, \vec{Z}, s\right)\right]\right| < O\left(H\sqrt{\log
  \left(n/\zeta\right)/\numfunctions}\right)$. Putting this together with
  Equation~\eqref{eq:slin_utility}, the theorem statement holds.
\end{proof}

\begin{theorem}[Full information online optimization]\label{thm:slinear_full_info}
  Let $\uslin\left(\qp^{\left(1\right)}, \vec{Z}^{\left(1\right)}, \cdot\right), \dots,
  \uslin\left(\qp^{\left(T\right)}, \vec{Z}^{\left(T\right)}, \cdot\right)$ be the set of functions
  observed by Algorithm~\ref{alg:1_d_online}, where $T \geq 8H^2 n^2 \ln \frac{6}{\zeta}$ and each vector $\vec{Z}^{(t)}$ is drawn from $\mathcal{N}_n$.
  Further, suppose we limit the parameter search space of
  Algorithm~\ref{alg:1_d_online} to $\left(0, \bar{s}\right)$, where $\bar{s} = \sqrt{2 \ln
  \left(\sqrt{\frac{8}{\pi}}\left(6n\numfunctions/\zeta\right)\right)}.$ Algorithm~\ref{alg:1_d_online}
  with input parameter $\lambda = \frac{1}{H} \sqrt{\frac{\ln\left(\bar{s} \sqrt{\numfunctions}\right)}{\numfunctions}}$ has regret bounded by
  $\tilde{O}\left(H n\sqrt{T}\right).$
\end{theorem}
\begin{proof}
First, in Theorem~\ref{thm:slin_dispersed}, we prove that with probability $1-\zeta/3$, the functions \[\uslin(\vec{Z}^{(1)}, \qp^{(1)}, \cdot), \dots, \uslin(\vec{Z}^{(\numfunctions)}, \qp^{(\numfunctions)}, \cdot)\] are piecewise $L$-Lipschitz
  with $L = O\left(\frac{Mn^5\numfunctions^3}{\zeta^3}\ln \left(\frac{n\numfunctions}{\zeta}\right)\right)$ and $\uslin$ is $\left(1/\sqrt{\numfunctions},
  O\left(n \sqrt{\numfunctions\log(n/\zeta)}\right)\right)$-dispersed with respect to $\sample = \left\{\left(\qp^{(1)}, \vec{Z}^{(1)}\right), \dots, \left(\qp^{(\numfunctions)}, \vec{Z}^{(\numfunctions)}\right)\right\}$.

  In Lemma~\ref{lem:slin_ball_bound}, we show that with probability $ 1- \zeta/3$, the
  values of $s$ that maximize \[\sum_{t = 1}^\numfunctions \uslin\left(\qp^{(t)}, \vec{Z}^{(t)}, \cdot\right)\] lie within the interval $\left(0, \sqrt{2 \ln
  \left(\sqrt{8/\pi}\left(6n\numfunctions/\zeta\right)\right)}\right)$. Thus, we can restrict Algorithm~\ref{alg:1dEfficient} to searching for a parameter in this range.
  
  We now show that the expected regret of Algorithm~\ref{alg:1_d_online} is at most $\tilde{O}\left(H
  n\sqrt{T}\right).$ If $L < 1$, Theorem~\ref{thm:1_d_online} guarantees that with probability at least $1-\zeta$, the
  expected regret of Algorithm~\ref{alg:1_d_online} is at most $\tilde{O}\left(H
  n\sqrt{T}\right).$ Otherwise, if $L \geq 1$, we can apply Corollary~\ref{cor:1_d_online}, which gives the same expected regret bound assuming $\log(L) = \tilde{O}(1)$, which we can assume with probability $1 - \zeta/3$. Corollary~\ref{cor:1_d_online} only holds when $T \geq \frac{1}{Lw}$, which is indeed with probability $1 - \zeta/3$ the case when $L \geq 1$ since $w = \sqrt{\frac{1}{\numfunctions}}$.

If this regret bound does not hold, then the regret is at
  most $H\numfunctions$, but this only happens with probability $\zeta$. Setting $\zeta =
  1/\sqrt{\numfunctions}$ gives the result.
\end{proof}

\begin{theorem}[Differentially private online optimization in the full
information setting]\label{thm:slinear_full_info_DP} $\,$
  Let $\uslin\left(\qp^{\left(1\right)}, \vec{Z}^{\left(1\right)}, \cdot\right), \dots,
  \uslin\left(\qp^{\left(T\right)}, \vec{Z}^{\left(T\right)}, \cdot\right)$ be the set of functions
  observed by Algorithm~\ref{alg:1_d_online}, where $T \geq 8H^2 n^2 \ln \frac{6}{\zeta}$ and each vector $\vec{Z}^{(t)}$ is drawn from $\mathcal{N}_n$. Let $\epsilon, \delta > 0$ be privacy parameters.
  Further, suppose we limit the parameter search space of
  Algorithm~\ref{alg:1_d_online} to $\left(0, \bar{s}\right)$, where $\bar{s} = \sqrt{2 \ln
  \left(\sqrt{\frac{8}{\pi}}\left(6n\numfunctions/\zeta\right)\right)}.$ Algorithm~\ref{alg:1_d_online}
  with input parameter $\lambda =
  \frac{\epsilon}{4H\sqrt{2T \ln \left(1/\delta\right)}}$ is $\left(\epsilon,
  \delta\right)$-differentially private and has regret bounded by
  $\tilde{O}\left(H
  \sqrt{T}\left(1/\epsilon + n\right)\right).$ 
\end{theorem}
\begin{proof}
  The proof is exactly the same as the proof of
  Theorem~\ref{thm:slinear_full_info}, except we  rely on
  Corollary~\ref{cor:online_1d_private} instead of
  Corollary~\ref{cor:1_d_online} to obtain the regret bound.
\end{proof}

\begin{lem}[\citet{Anthony09:Neural}]\label{lem:Bartlett_Gaussian}
  If $Z$ is a standard normal random variable and $x > 0$, then $\Pr[Z \geq x]
  \geq \frac{1}{2}\left(1 - \sqrt{1 - e^{-x^2}}\right).$
\end{lem}

\begin{cor}\label{cor:gaussian}
  If $Z$ is a standard normal random variable and $x > 0$, then $\Pr[|Z| \geq x]
  \geq 1- x.$
\end{cor}
\begin{proof}
\begin{align*}
\Pr[|Z| \leq x] &\leq \sqrt{1 - e^{-x^2}} &\left(\text{Lemma~\ref{lem:Bartlett_Gaussian}}\right)\\
&\leq \sqrt{x^2} = x & \left(1-e^{-\gamma} \leq \gamma \text{ for all } \gamma \in \R\right)
\end{align*}
\end{proof}

\begin{lem}\label{lem:min_Gaussian}
  Suppose $Z_1, \dots, Z_\tau$ are $\tau$ independent standard normal random variables. Then
  \[
  \Pr\left[\min_{i \in [\tau]} \left|Z_i\right| \leq \frac{\zeta}{2\tau}\right] \leq \zeta.
  \]
\end{lem}
\begin{proof}
  From Corollary~\ref{cor:gaussian}, we know that \[\Pr\left[\min_{i \in [\tau]}
  \left|Z_i\right| \geq \frac{\zeta}{2\tau}\right] = \prod_{i = 1}^\tau \Pr\left[\left|Z_i\right| \geq
  \frac{\zeta}{2\tau}\right] \geq \left(1 - \frac{\zeta}{2\tau}\right)^\tau \geq
  e^{-\zeta}.\] The last inequality holds because for $\gamma \in [0, 3/4]$, we
  have that $1-\gamma \geq e^{-2\gamma}$, which is applicable because
  $\frac{\zeta}{2\tau} < \frac{3}{4}$. Therefore,
  \[
  \Pr\left[\min_{i \in [\tau]} \left|Z_i\right| \leq \frac{\zeta}{2\tau}\right] < 1- e^{-\zeta} \leq \zeta.
  \]
\end{proof}

\begin{lem}\label{lem:s0_bound}
  With probability at least $1-\zeta$, $\min \left\{\left|\left\langle
  \vec{Z}^{\left(t\right)}, \vec{u}_i^{\left(t\right)}\right\rangle\right| : i \in [n], t \in
  [\numfunctions]\right\} \geq \frac{\zeta}{2n\numfunctions}$.
\end{lem}
\begin{proof}
  Let $S_1, \dots, S_n$ be $n$ sets of random variables such that $S_i =
  \left\{\left|\left\langle \vec{u}_i^{\left(t\right)} , \vec{Z}^{\left(t\right)}\right\rangle\right| : t \in
  [\numfunctions]\right\}.$ Notice that $\cup_{i = 1}^n S_i$ are all of the boundaries dividing
  the domain of $\sum_{t = 1}^T \uslin\left(\qp^{(t)}, \vec{Z}^{(t)}, \cdot\right)$ into intervals over which the function is differentiable. Also, within each $S_i$, the variables
  are all absolute values of independent Gaussians, since for any unit vector
  $\vec{u}$ and any $\vec{Z} \sim \mathcal{N}_n$, $\vec{u} \cdot \vec{Z}$ is a
  standard Gaussian. Lemma~\ref{lem:min_Gaussian} guarantees that for all $i \in [n]$,
  $\Pr\left[\min_{t \in [T]} \left\{\left|\left\langle \vec{u}_i^{\left(t\right)} , \vec{Z}^{\left(t\right)}\right\rangle\right|\right\} \leq \frac{\zeta}{2n\numfunctions}\right] \leq
  \frac{\zeta}{n}.$ By a union bound, this means that with probability at least
  $1-\zeta$, $\min_{i \in [n], t \in [\numfunctions]} \left\{\left|\left\langle \vec{u}_i^{\left(t\right)} , \vec{Z}^{\left(t\right)}\right\rangle\right|\right\} \geq \frac{\zeta}{2n\numfunctions}$. By
  definition of the sets $S_1, \dots, S_n$ and the value $s_0$, this means that
  with probability at least $1-\zeta$, $s_0 \geq \frac{\zeta}{2n\numfunctions}$.
\end{proof}

\begin{lem}\label{lem:positive_utility}
  If $\numfunctions \geq 8H^2 n^2 \ln \frac{1}{\zeta}$, with probability at least $1-\zeta$, there
  exists $s > 0$ such that \[\sum_{t = 1}^T \uslin\left(\qp^{(t)}, \vec{Z}^{(t)}, s\right) \geq 0.\]
\end{lem}

\begin{proof}
We will prove that with probability $1 - \zeta$ over the draw of $\vec{Z}^{(1)}, \dots, \vec{Z}^{(\numfunctions)} \sim \mathcal{N}_n$, \[\sum_{t = 1}^{\numfunctions} \uslin\left(A^{(t)}, \vec{Z}^{(t)}, \tilde{s}\right) \geq 0,\] where $\tilde{s} = \frac{3}{2nT(10n+8)}$. From Lemma~\ref{lem:s0_bound}, we know that with probability at least $1 - \frac{3}{10n + 8}$, $\min \left\{\left|\left\langle
  \vec{Z}^{\left(t\right)}, \vec{u}_i^{\left(t\right)}\right\rangle\right| : i \in [n], t \in
  [\numfunctions]\right\} > \tilde{s}$. Recall that \[\phi_s(y) = \begin{cases} \sign(y) & \text{if } |y| \geq s\\
  y/s & \text{if } |y|< s.\end{cases}\] Therefore, when $\tilde{s} < \min \left\{\left|\left\langle
  \vec{Z}^{\left(t\right)}, \vec{u}_i^{\left(t\right)}\right\rangle\right| : i \in [n], t \in
  [\numfunctions]\right\}$, for all $t \in [\numfunctions]$, \begin{equation}
  \uslin\left(A^{(t)}, \vec{Z}^{(t)}, \tilde{s}\right) = \sum_{i = 1}^n a_{ii}^2 + \sum_{i \not= j} a_{ij}\sign\left(\left\langle \vec{Z}^{(t)}, \vec{u}_i^{(t)}\right\rangle\right)\sign\left(\left\langle \vec{Z}^{(t)}, \vec{u}_j^{(t)}\right\rangle\right).\label{eq:small_s_slin}
  \end{equation}
  
  Recall that the GW algorithm uses the rounding function $r(y) = \sign(y)$. In other words, when the matrix $A^{(t)}$ is the input to Algorithm~\ref{alg:GW} and $\vec{Z}^{(t)}$ is the hyperplane drawn in Step~\ref{step:draw}, it sets $z_i = 1$ with probability $\frac{1}{2}\left(1 + \sign\left(\left\langle \vec{Z}^{(t)}, \vec{u}_i^{(t)}\right\rangle\right)\right)$ and it sets $z_i = -1$ with probability $\frac{1}{2}\left(1 - \sign\left(\left\langle \vec{Z}^{(t)}, \vec{u}_i^{(t)}\right\rangle\right)\right)$. In other words, it sets $z_i = \sign\left(\left\langle \vec{Z}^{(t)}, \vec{u}_i^{(t)}\right\rangle\right)$. Therefore, Equation~\eqref{eq:small_s_slin} is the objective value of the GW algorithm given the input matrix $A^{(t)}$ and hyperplane $\vec{Z}^{(t)}$. Since the GW algorithm has an expected approximation ratio of $0.878$ (in expectation over the draw of the hyperplane), \begin{align*}&\E_{\vec{Z}^{(t)} \sim \mathcal{N}_n} \left[\left.\uslin\left(A^{(t)}, \vec{Z}^{(t)}, \tilde{s}\right) \ \right| \ \tilde{s} < \min \left\{\left|\left\langle
  \vec{Z}^{\left(t\right)}, \vec{u}_i^{\left(t\right)}\right\rangle\right| : i \in [n], t \in
  [\numfunctions]\right\}\right]\\
  \geq \text{ }&0.878 \max_{\vec{z} \in \{0,1\}^n} \left\{ \sum_{i,j} a_{ij}^{(t)}z_iz_j\right\}.\end{align*}
  
  \citet{Charikar04:Maximizing} prove that $\max_{\vec{z} \in \{0,1\}^n} \left\{ \sum_{i,j} a_{ij}^{(t)}z_iz_j\right\} \geq \frac{1}{n} \sum_{i,j} \left|a_{ij}^{(t)}\right|.$ Therefore, using the notation $E$ to denote the event where  $\tilde{s} < \min \left\{\left|\left\langle
  \vec{Z}^{\left(t\right)}, \vec{u}_i^{\left(t\right)}\right\rangle\right| : i \in [n], t \in
  [\numfunctions]\right\}$, we know that \begin{equation}\E_{\vec{Z}^{(t)} \sim \mathcal{N}_n} \left[\left.\uslin\left(A^{(t)}, \vec{Z}^{(t)}, \tilde{s}\right) \ \right| \ E\right] \geq \frac{0.878}{n}\sum_{i,j} \left|a_{ij}^{(t)}\right| \geq \frac{4}{5n}\sum_{i,j} \left|a_{ij}^{(t)}\right|.\label{eq:bound_cond_exp}\end{equation}
  
  By the law of total expectation, \begin{align*}&\E_{\vec{Z}^{(t)} \sim \mathcal{N}_n} \left[\uslin\left(A^{(t)}, \vec{Z}^{(t)}, \tilde{s}\right)\right]\\
  =\text{ }& \E_{\vec{Z}^{(t)} \sim \mathcal{N}_n} \left[\left.\uslin\left(A^{(t)}, \vec{Z}^{(t)}, \tilde{s}\right) \ \right| \ E\right]\cdot \Pr[E] + \E_{\vec{Z}^{(t)} \sim \mathcal{N}_n} \left[\left.\uslin\left(A^{(t)}, \vec{Z}^{(t)}, \tilde{s}\right) \ \right| \ \neg E\right]\cdot \left(1 - \Pr[E]\right)\\
  \geq\text{ } & \frac{4}{5n}\sum_{i,j} \left|a_{ij}^{(t)}\right| \cdot \Pr[E] + \E_{\vec{Z}^{(t)} \sim \mathcal{N}_n} \left[\left.\uslin\left(A^{(t)}, \vec{Z}^{(t)}, \tilde{s}\right) \ \right| \ \neg E\right]\cdot \left(1 - \Pr[E]\right)\\
  \geq\text{ } & \frac{4}{5n}\sum_{i,j} \left|a_{ij}^{(t)}\right| \cdot \Pr[E] - \sum_{i,j} \left|a_{ij}^{(t)}\right|\cdot \left(1 - \Pr[E]\right)\\
  =\text{ } &\sum_{i,j} \left|a_{ij}^{(t)}\right|\left(\Pr[E] \left(\frac{4}{5n} + 1\right) - 1\right) \end{align*}
  where the second-to-last inequality follows from Equation~\eqref{eq:bound_cond_exp} and the final inequality follows from the fact that with probability 1, $\left|\uslin\left(A^{(t)}, \vec{Z}^{(t)}, \tilde{s}\right)\right| \leq  \sum_{i,j} \left|a_{ij}^{(t)}\right|.$
  
  Since $\Pr[E] \geq 1 - \frac{3}{10n + 8}$, we have that $\E_{\vec{Z}^{(t)} \sim \mathcal{N}_n} \left[\uslin\left(A^{(t)}, \vec{Z}^{(t)}, \tilde{s}\right)\right] \geq \frac{1}{2n} \sum_{i,j} \left|a_{ij}^{(t)}\right| \geq \frac{1}{2n}.$ We now apply Hoeffding's to prove the result:
  \begin{align*}
 & \Pr\left[\sum_{t = 1}^{\numfunctions} \uslin\left(A^{(t)}, \vec{Z}^{(t)}, \tilde{s}\right) \leq 0\right]\\
  =\text{ } &\Pr\left[\E\left[\frac{1}{T}\sum_{t = 1}^{\numfunctions} \uslin\left(A^{(t)}, \vec{Z}^{(t)}, \tilde{s}\right)\right] - \frac{1}{T}\sum_{t = 1}^{\numfunctions} \uslin\left(A^{(t)}, \vec{Z}^{(t)}, \tilde{s}\right)\geq \E\left[\frac{1}{T}\sum_{t = 1}^{\numfunctions} \uslin\left(A^{(t)}, \vec{Z}^{(t)}, \tilde{s}\right)\right]\right]\\
  \leq\text{ }&\Pr\left[\E\left[\frac{1}{T}\sum_{t = 1}^{\numfunctions} \uslin\left(A^{(t)}, \vec{Z}^{(t)}, \tilde{s}\right)\right] - \frac{1}{T}\sum_{t = 1}^{\numfunctions} \uslin\left(A^{(t)}, \vec{Z}^{(t)}, \tilde{s}\right)\geq \frac{1}{2n}\right]\\
  \leq \text{ }&\exp\left(-\frac{2T^2}{16n^2\sum_{t = 1}^{\numfunctions}\left(\sum_{i,j}\left|a_{ij}^{(t)}\right|\right)^2}\right)\\
    \leq \text{ }&\exp\left(-\frac{T^2}{8n^2 \cdot TH^2}\right)\\
    \leq \text{ }& \zeta
  \end{align*} where the second-to-last inequality followsfrom the fact that with probability 1, for all $t \in [T]$, $\left|\uslin\left(A^{(t)}, \vec{Z}^{(t)}, \tilde{s}\right)\right| \leq  \sum_{i,j} \left|a_{ij}^{(t)}\right| \leq H$. The final inequality follows from the fact that $T \geq 8H^2 n^2 \ln \frac{1}{\zeta}$.
\end{proof}

\begin{lem}[\citet{Gordon41:Values}]\label{lem:Gaussian_abs_val}
  Let $Z$ be a standard normal random variable. Then $\Pr[|Z| \geq z] \leq
  \frac{2}{z\sqrt{2 \pi}}e^{-z^2/2}$.
\end{lem}

\begin{lem}\label{lemma:bar_s_UB}
With probability at least $1-\zeta$, $\max \left\{\left|\left\langle \vec{Z}^{\left(t\right)}, \vec{u}_i^{\left(t\right)}\right\rangle \right| : i \in [n], t \in [\numfunctions]\right\} \leq \sqrt{2 \ln \left(\sqrt{\frac{8}{\pi}}\frac{n\numfunctions}{\zeta}\right)}$.
\end{lem}
\begin{proof}
Let $z = \sqrt{2 \ln \left(\sqrt{\frac{8}{\pi}}\frac{n\numfunctions}{\zeta}\right)}$. We may assume that $n \geq 2$, which means that $z \geq 1.$ Therefore, if $Z$ is a standard Gaussian, by Lemma~\ref{lem:Gaussian_abs_val}, we know that $\Pr[|Z| \geq z] \leq \frac{2}{z\sqrt{2 \pi}}e^{-z^2/2} \leq \sqrt{\frac{2}{ \pi}}e^{-z^2/2}$. Let $S_1, \dots, S_n$ be $n$ sets of random variables such that $S_i = \left\{\left|\left\langle \vec{u}_i^{\left(t\right)} , \vec{Z}^{\left(t\right)}\right\rangle\right| : t \in [\numfunctions]\right\}.$ Notice that $\cup_{i = 1}^n S_i$ are all of the boundaries dividing the domain of $\sum_{t = 1}^T \uslin\left(\qp^{(t)}, \vec{Z}^{(t)}, \cdot\right)$ into intervals over which the function is differentiable. Also, within each $S_i$, the variables are all absolute values of independent Gaussians, since for any unit vector $\vec{u}$ and any $\vec{Z} \sim \mathcal{N}_n$, $\vec{u} \cdot \vec{Z}$ is a standard Gaussian. Therefore, for all $i \in [n]$, $\Pr\left[\max_{t \in [\numfunctions]} \left\{\left|\left\langle \vec{Z}^{\left(t\right)}, \vec{u}_i^{\left(t\right)}\right\rangle \right|\right\} \leq z\right] \geq \left(1-\sqrt{\frac{2}{ \pi}}e^{-z^2/2}\right)^\numfunctions.$ By a union bound, this means that \begin{align*}\Pr\left[\max_{i \in [n], t \in [\numfunctions]} \left\{\left|\left\langle \vec{Z}^{\left(t\right)}, \vec{u}_i^{\left(t\right)}\right\rangle \right|\right\} \geq z\right] &\leq n\left(1-\left(1-\sqrt{\frac{2}{ \pi}}e^{-z^2/2}\right)^\numfunctions\right)\\
&= n\left(1-\left(1-\frac{\zeta}{2n\numfunctions}\right)^\numfunctions\right)\\
&\leq n\left(1-e^{-\zeta/n}\right) &\left( \forall x \in \left(0, 3/4\right), e^{-2x} \leq 1- x\right)\\
&\leq \zeta. &\left(\forall x \in \R, 1-e^{-x} \leq x\right)
\end{align*}
\end{proof}

\begin{lem}\label{lem:slin_ball_bound}
If $\numfunctions \geq 8H^2 n^2 \ln \frac{2}{\zeta}$, with probability at least $1-\zeta$, $\argmax_{s> 0} \sum_{t = 1}^T \uslin\left(\qp^{(t)}, \vec{Z}^{(t)}, s\right) \leq \sqrt{2 \ln \left(\sqrt{\frac{8}{\pi}}\frac{2n\numfunctions}{\zeta}\right)}$.
\end{lem}

\begin{proof}
Let $\bar{s} = \max \left\{\left|\left\langle \vec{Z}^{\left(t\right)}, \vec{u}_i^{\left(t\right)}\right\rangle \right| : i \in [n], t \in [\numfunctions]\right\}$. From Lemma~\ref{lemma:bar_s_UB}, we know that with probability at least $1-\zeta/2$, $\bar{s} \leq \sqrt{2 \ln \left(\sqrt{\frac{8}{\pi}}\frac{2n\numfunctions}{\zeta}\right)}$.
By definition of $\phi_s$, when $s > \bar{s}$, $\sum_{t = 1}^T \uslin\left(\qp^{(t)}, \vec{Z}^{(t)}, s\right) = a/s^2$ for some $a \in \R$. If $a \geq 0$, then $\sum_{t = 1}^T \uslin\left(\qp^{(t)}, \vec{Z}^{(t)}, s\right)$ is non-increasing as $s$ grows beyond $\bar{s}$, so the claim holds. If $a < 0$, then $\sum_{t = 1}^T \uslin\left(\qp^{(t)}, \vec{Z}^{(t)}, s\right) < 0$ for all $s > \bar{s}$. However, by Lemma~\ref{lem:positive_utility}, we know that with probability at least $1-\zeta/2$, there exists some $s>0$ such that $\sum_{t = 1}^T \uslin\left(\qp^{(t)}, \vec{Z}^{(t)}, s\right) \geq 0$. Therefore, with probability $1 - \zeta$, $\argmax_{s> 0} \sum_{t = 1}^T \uslin\left(\qp^{(t)}, \vec{Z}^{(t)}, s\right) \leq \bar{s}$.
\end{proof}

\begin{lem}\label{lem:slin_generalization}[\citet{Balcan17:Learning}]
Let $\left(\qp^{\left(1\right)}, \vec{Z}^{\left(1\right)}\right), \dots, \left(\qp^{\left(\numfunctions\right)}, \vec{Z}^{\left(\numfunctions\right)}\right)$ be $\numfunctions$ tuples sampled from $\dist \times \mathcal{N}_n$. With probability at least $1-\zeta$, for all $ s> 0$, \[\left|\frac{1}{T}\sum_{t = 1}^T \uslin\left(\qp^{(t)}, \vec{Z}^{(t)}, s\right) - \E_{\qp, \vec{Z} \sim \mathcal{D} \times \mathcal{N}_n}\left[\uslin\left(A, \vec{Z}, s\right)\right]\right| = O\left(H\sqrt{\frac{\log \left(n/\zeta\right)}{\numfunctions}}\right).\]
\end{lem}

\section{Proofs for auction design (Section~\ref{sec:auctions})}\label{app:auctions}
\paragraph{Notation and definitions.} Suppose that for some valuation vector $\vec{v}$, the abstract utility function
$u\left(\vec{v}, \cdot\right)$ is piecewise $L$-Lipschitz. Let $\partition_{\vec{v}}$ be
the corresponding partition of $\configs$ such that over any $R \in
\partition_{\vec{v}}$, $u\left(\vec{v}, \cdot\right)$ is $L$-Lipschitz.

\begin{defn}[Hyperplane delineation]
Let $\Psi$ be a set of hyperplanes and let $\partition$ be a partition of a set $\configs \subseteq \R^d$. Let $K_1, \dots, K_q$ be the connected components of $\configs \setminus \Psi$. Suppose every set in $\partition$ is the union of some collection of sets $K_{i_1}, \dots, K_{i_j}$ together with their limit points. Then we say that the set $\Psi$ \emph{delineates} $\partition$.
\end{defn}
If a set  $\Psi_{\vec{v}}$ delineates $\partition_{\vec{v}}$, then $u$ can only have discontinuities that fall at points along hyperplanes in $\Psi_{\vec{v}}$.

\begin{theorem}\label{thm:main_auction}
Given a set $\sample = \left\{\vec{v}^{\left(1\right)}, \dots, \vec{v}^{\left(\numfunctions\right)}\right\}$, suppose the sets $\Psi_{\vec{v}^{\left(1\right)}}, \dots, \Psi_{\vec{v}^{\left(\numfunctions\right)}}$ delineate the partitions $\cP_{\vec{v}^{\left(1\right)}}, \dots, \cP_{\vec{v}^{\left(\numfunctions\right)}}$. Suppose the multi-set union of $\Psi_{\vec{v}^{\left(1\right)}}, \dots, \Psi_{\vec{v}^{\left(\numfunctions\right)}}$ can be partitioned into $P$ multi-sets $\cB_1, \dots, \cB_P$ such that for each multi-set $\cB_i$:
\begin{enumerate}
\item The hyperplanes in $\cB_i$ are parallel with probability 1 over the draw of $\sample$.
\item The offsets of the hyperplanes in $\cB_i$ are independently drawn from $\kappa$-bounded distributions.
\end{enumerate}
With probability at least $1-\zeta$ over the draw of $\sample$, $u$ is $\left(\frac{1}{2\kappa\max|\cB_i|^{1-\alpha}}, O\left(P\max|\cB_i|^\alpha\sqrt{\ln \frac{P}{\zeta}}\right)\right)$-dispersed with respect to $\sample$.
\end{theorem}

\begin{proof}
We begin by proving that the hyperplanes within each multi-set $\cB_i$ are well-dispersed. For a multi-set $\cB_i$, let $\Theta_i$ be the multi-set of those hyperplanes' offsets. Also, let  $w_0 = \frac{1}{\kappa\max|\cB_i|^{1-\alpha}}$ and let $k_0 = O\left(\max|\cB_i|^\alpha\sqrt{\ln \frac{P}{\zeta}}\right)$.
By assumption the elements of $\Theta_i$ are independently drawn from $\kappa$-bounded distributions. Therefore, by Lemma~\ref{lem:dispersion}, with probability at least $1-\zeta$, for all $i \in [P]$, the elements of $\Theta_i$ are $\left(w_0,k_0\right)$-dispersed.

Next, let $B \subseteq \configs$ be an arbitrary ball with radius $w_0/2$. For $j \in [\numfunctions]$, $\partition_{\vec{v}^{\left(j\right)}}$ can only split $B$ if there exists a hyperplane in $\Psi_{\vec{v}^{\left(j\right)}}$ passing through $B$. We claim that at most $k_0$ hyperplanes from each multi-set $\cB_i$ pass through $B$. This follows from three facts: First, the hyperplanes in $\cB_i$ are parallel. Second, for any interval $I \subset \R$ of length $w_0$, the intersection of $I$ and $\Theta_i$ has size at most $k_0$. Third, for all $\vec{a}, \vec{b} \in B$, we know that $||\vec{a} - \vec{b}|| \leq w_0$. Therefore, at most $k_0P$ hyperplanes in total pass through $B$. This means that with probability at least $1-\zeta$, the function $u$ is $\left(w_0/2, k_0P\right)$-dispersed with respect to $\sample$.
\end{proof}

\subsection{Posted pricing mechanisms}

We now apply Theorem~\ref{thm:main_auction} to posted pricing mechanisms.

\itemPricing*

\begin{proof}
  We begin by analyzing additive buyers. For a fixed valuation vector $\vec{v}$,
  buyer $j$ will buy any item so long as his value for the item exceeds its
  price. Therefore, the set of items buyer $j$ is willing to buy is defined by
  $m$ hyperplanes: $v_j\left(\left\{1\right\}\right) = \rho_1, \dots, v_j\left(\left\{m\right\}\right) = \rho_m$. Let
  $\Psi_{\vec{v}}$ be the multi-set union of all $m$ hyperplanes for all $n$
  buyers. As we range over prices in one connected component of $\R^m \setminus
  \Psi_{\vec{v}}$, the set of items each buyer is willing to buy is fixed and
  therefore the allocation of the pricing mechanism is fixed. Since revenue and social welfare are Lipschitz when the allocation is fixed, $\Psi_{\vec{v}}$ delineates
  the partition $\partition_{\vec{v}}$.

  Consider a set $\sample = \left\{\vec{v}^{\left(1\right)}, \dots, \vec{v}^{\left(\numfunctions\right)}\right\}$ with
  corresponding multi-sets $\Psi_{\vec{v}^{\left(1\right)}}, \dots,
  \Psi_{\vec{v}^{\left(\numfunctions\right)}}$ of hyperplanes. We now partition the multi-set union of $\Psi_{\vec{v}^{\left(1\right)}},
  \dots, \Psi_{\vec{v}^{\left(\numfunctions\right)}}$ into $nm$ multi-sets $\cB_{i,j}$ for all $j \in
  [n]$ and $i \in [m]$ such that for each $\cB_{i,j}$, the hyperplanes in
  $\cB_{i,j}$ are parallel with probability 1 over the draw of $\sample$ and the
  offsets of the hyperplanes in $\cB_{i,j}$ are independent random variables with $\kappa$-bounded distributions. To this end, define a single multi-set
  $\cB_{i,j}$ to consist of the hyperplanes $\left\{v_j^{\left(1\right)}\left(\left\{i\right\}\right) = \rho_i, \dots,
  v_j^{\left(\numfunctions\right)}\left(\left\{i\right\}\right)= \rho_i\right\}$. Clearly, these hyperplanes are parallel and since
  we assume that the marginal distribution over each buyer's value for each good
  is $\kappa$-bounded, the offsets are independent draws from a $\kappa$-bounded
  distribution. Therefore, the theorem statement holds after applying
  Theorem~\ref{thm:main_auction}.

  When the buyers have unit-demand valuations, we may assume without loss of
  generality that each buyer will only buy one item. For a fixed valuation
  vector $\vec{v}$, buyer $j$'s preference ordering over the items is defined by
  ${m \choose 2}$ hyperplanes: $v_j\left(\left\{i\right\}\right) - \rho_i = v_j\left(\left\{i'\right\}\right) - \rho_{i'}$ because buyer
  $j$ will prefer item $i$ to item $i'$ if and only if $v_j\left(\left\{i\right\}\right) - \rho_i \geq
  v_j\left(\left\{i'\right\}\right) - \rho_{i'}$. Let $\Psi_{\vec{v}}$ be the multi-set union of all ${m
  \choose 2}$ hyperplanes for all $n$ buyers. As we range over prices in one
  connected component of $\R^m \setminus \Psi_{\vec{v}}$, each buyer's
  preference ordering over the items is fixed and therefore the allocation of
  the pricing mechanism is fixed. The set $\Psi_{\vec{v}}$ delineates the partition
  $\partition_{\vec{v}}$.

  Consider a set $\sample = \left\{\vec{v}^{\left(1\right)}, \dots, \vec{v}^{\left(\numfunctions\right)}\right\}$ with
  corresponding multi-sets $\Psi_{\vec{v}^{\left(1\right)}}, \dots,
  \Psi_{\vec{v}^{\left(\numfunctions\right)}}$ of hyperplanes. We now partition the multi-set union of $\Psi_{\vec{v}^{\left(1\right)}},
  \dots, \Psi_{\vec{v}^{\left(\numfunctions\right)}}$ into $n{m \choose 2}$ multi-sets $\cB_{i,i',j}$
  for all $i,i' \in [m]$ and $j \in [n]$ such that for each $\cB_{i,i',j}$, the
  hyperplanes in $\cB_{i,i',j}$ are parallel with probability 1 over the draw of
  $\sample$ and the offsets of the hyperplanes in $\cB_{i,i',j}$ are independent
  random variables with $W\kappa$-bounded distributions. To this end,
  define a single multi-set $\cB_{i,i',j}$ to consist of the hyperplanes
  \[
  \left\{v_j^{\left(1\right)}\left(\left\{i\right\}\right) - \rho_{i} = v_j^{\left(1\right)}\left(\left\{i'\right\}\right) - \rho_{i'}, \dots,
  v_j^{\left(\numfunctions\right)}\left(\left\{i\right\}\right) - \rho_i = v_j^{\left(\numfunctions\right)}\left(\left\{i'\right\}\right) - \rho_{i'}\right\}.
  \]
  Clearly, these hyperplanes are parallel. Recall that we assume the buyers'
  valuations are in the range $[0,W]$ and are drawn from
  pairwise $\kappa$-bounded joint distributions. Therefore, the offsets are
  independent draws from a $W\kappa$-bounded distribution by
  Lemma~\ref{lem:bounded_joint} and the theorem statement holds after applying
  Theorem~\ref{thm:main_auction}.

  Finally, we analyze buyers with general valuations. For a given valuation
  vector $\vec{v}$ and any two bundles $b$ and $b'$ in $2^{[m]}$, buyer $j$'s preference for
  $b$ over $b'$ is defined by the hyperplane $v_j\left(b\right) - \sum_{i \in b}
  \rho_i = v_j\left(b'\right) - \sum_{i \in b'} \rho_i.$ This is true for all pairs of
  bundles, which leaves us with a set $\mathcal{H}_j$ of ${2^m \choose 2}$ hyperplanes
  partitioning $\R^m$. Consider one connected component $R$ of $\R^m \setminus
  \mathcal{H}_j$. As we range over the prices in $R$, buyer $j$'s preference ordering over all $2^m$ bundles
  is fixed. Let $\Psi_{\vec{v}} = \bigcup_{j = 1}^n \mathcal{H}_j$ be the set of
  hyperplanes defining all $n$ buyers' preference orderings over the bundles. As
  we range over the prices in one connected component of $\R^m \setminus
  \Psi_{\vec{v}}$, every buyer's preference ordering is fixed and therefore the
  allocation of the pricing mechanism is fixed. The set $\Psi_{\vec{v}}$ therefore delineates the partition
  $\partition_{\vec{v}}$.

  Consider a set $\sample = \left\{\vec{v}^{\left(1\right)}, \dots, \vec{v}^{\left(\numfunctions\right)}\right\}$ with
  corresponding multi-sets $\Psi_{\vec{v}^{\left(1\right)}}, \dots,
  \Psi_{\vec{v}^{\left(\numfunctions\right)}}$ of hyperplanes. We now partition the multi-set union of $\Psi_{\vec{v}^{\left(1\right)}},
  \dots, \Psi_{\vec{v}^{\left(\numfunctions\right)}}$ into $n{2^m \choose 2}$ multi-sets $\cB_{j,
  b,b'}$ for all $j \in [n]$ and $b,b' \in 2^{[m]}$ such that for each $\cB_{j,
  b,b'}$, the hyperplanes in $\cB_{j, b,b'}$ are parallel with probability 1 over
  the draw of $\sample$ and the offsets of the hyperplanes in $\cB_{j, b,b'}$ are
  independent random variables with $W\kappa$-bounded distributions. To
  this end, for an arbitrary pair of bundles $b$ and $b'$, define a single
  multi-set $\cB_{j, b,b'}$ to consist of the hyperplanes \[\left\{v_j^{\left(1\right)}\left(b\right) -
  \sum_{i \in b} \rho_i = v_j^{\left(1\right)}\left(b'\right) - \sum_{i \in b'} \rho_i, \dots, v_j^{\left(\numfunctions\right)}\left(b\right) -
  \sum_{i \in b} \rho_i = v_j^{\left(\numfunctions\right)}\left(b'\right) - \sum_{i \in b'} \rho_i\right\}.\] Clearly,
  these hyperplanes are parallel. Recall that we assume the buyers' valuations
  are in the range $[0,W]$ and their values for the bundles have pairwise
  $\kappa$-bounded joint distributions. Therefore, the offsets are independent
  draws from a $W\kappa$-bounded distribution by Lemma~\ref{lem:bounded_joint}
  and the theorem statement holds after applying Theorem~\ref{thm:main_auction}.
\end{proof}

\begin{theorem}[Differential privacy for revenue maximization]\label{thm:pricing_DP_rev}
  Suppose that $u(\vec{v}, \vec{\rho})$ is the revenue of the posted price mechanism with prices $\vec{\rho}$ and buyers' values $\vec{v}$. With
  probability at least $1-\delta$, if $\empvec{\rho}$ is the parameter vector
  returned by Algorithm~\ref{alg:efficient}, then the following are true.
\begin{enumerate}
\item Suppose the buyers have additive valuations and for each distribution $\dist^{(t)}$,
the item values have $\kappa$-bounded marginal distributions. 
 Then \[\E_{\vec{v} \sim \mathcal{D}}\left[u\left(\vec{v}, \hat{\vec{\rho}}\right)\right]
 \geq \max_{\vec{\rho}}\E_{\vec{v} \sim \mathcal{D}}[u\left(\vec{v}, \vec{\rho}\right)]
- \tilde O\left(\frac{Hm}{\numfunctions\epsilon} + \frac{1}{\kappa \sqrt{T}} + \frac{Hnm}{\sqrt{T}}\right).\]
\item Suppose the buyers are unit-demand with $v_j(\{i\}) \in [0,W]$ for each buyer $j \in [n]$ and item $i \in [m]$. Also, suppose that for each distribution
$\dist^{(t)}$, each buyer $j$, and every pair of items $i$ and $i'$, $v_j(\{i\})$ and $v_j(\{i'\})$ have a $\kappa$-bounded joint distribution.
 Then \[\E_{\vec{v} \sim \mathcal{D}}\left[u\left(\vec{v}, \hat{\vec{\rho}}\right)\right]
 \geq \max_{\vec{\rho}}\E_{\vec{v} \sim \mathcal{D}}[u\left(\vec{v}, \vec{\rho}\right)]
- \tilde O\left(\frac{Hm}{\numfunctions\epsilon} + \frac{1}{W\kappa \sqrt{T}} + \frac{Hnm^2}{\sqrt{T}}\right).\]
\item Suppose the buyers have general valuations in $[0,W]$. Also, suppose that for each 
$\dist^{(t)}$, each buyer $j$, and every pair of bundles $b$ and $b'$, $v_j(b)$ and $v_j(b')$ have a $\kappa$-bounded joint distribution.
 Then \[\E_{\vec{v} \sim \mathcal{D}}\left[u\left(\vec{v}, \hat{\vec{\rho}}\right)\right]
 \geq \max_{\vec{\rho}}\E_{\vec{v} \sim \mathcal{D}}[u\left(\vec{v}, \vec{\rho}\right)]
- \tilde O\left(\frac{Hm}{\numfunctions\epsilon} + \frac{1}{W\kappa \sqrt{T}} + Hn2^{2m}\sqrt{\frac{m}{T}}\right).\]
\end{enumerate}
\end{theorem}

\begin{proof}
  Privacy follows from Lemma~\ref{lem:efficientPrivacy}. The utility guarantee
  follows from Lemma~\ref{lem:efficientPrivacy}, Theorem~\ref{thm:item_pricing},
  and Lemma~\ref{lem:sample_auctions}.
\end{proof}

\begin{theorem}[Differential privacy for welfare maximization]\label{thm:pricing_DP_SW}
  Suppose that $u(\vec{v}, \vec{\rho})$ is the social welfare of the posted price mechanism with prices $\vec{\rho}$ and buyers' values $\vec{v}$. With
  probability at least $1-\delta$, if $\empvec{\rho}$ is the parameter vector
  returned by Algorithm~\ref{alg:efficient}, then the following are true.
\begin{enumerate}
\item Suppose the buyers have additive valuations and for each distribution $\dist^{(t)}$,
the item values have $\kappa$-bounded marginal distributions. 
 Then \[\E_{\vec{v} \sim \mathcal{D}}\left[u\left(\vec{v}, \hat{\vec{\rho}}\right)\right]
 \geq \max_{\vec{\rho}}\E_{\vec{v} \sim \mathcal{D}}[u\left(\vec{v}, \vec{\rho}\right)]
- \tilde O\left(\frac{Hm}{\numfunctions\epsilon} + \frac{Hnm}{\sqrt{T}}\right).\]
\item Suppose the buyers are unit-demand with $v_j(\{i\}) \in [0,W]$ for each buyer $j \in [n]$ and item $i \in [m]$. Also, suppose that for each distribution
$\dist^{(t)}$, each buyer $j$, and every pair of items $i$ and $i'$, $v_j(\{i\})$ and $v_j(\{i'\})$ have a $\kappa$-bounded joint distribution.
 Then \[\E_{\vec{v} \sim \mathcal{D}}\left[u\left(\vec{v}, \hat{\vec{\rho}}\right)\right]
 \geq \max_{\vec{\rho}}\E_{\vec{v} \sim \mathcal{D}}[u\left(\vec{v}, \vec{\rho}\right)]
- \tilde O\left(\frac{Hm}{\numfunctions\epsilon} + \frac{Hnm^2}{\sqrt{T}}\right).\]
\item Suppose the buyers have general valuations in $[0,W]$. Also, suppose that for each 
$\dist^{(t)}$, each buyer $j$, and every pair of bundles $b$ and $b'$, $v_j(b)$ and $v_j(b')$ have a $\kappa$-bounded joint distribution.
 Then \[\E_{\vec{v} \sim \mathcal{D}}\left[u\left(\vec{v}, \hat{\vec{\rho}}\right)\right]
 \geq \max_{\vec{\rho}}\E_{\vec{v} \sim \mathcal{D}}[u\left(\vec{v}, \vec{\rho}\right)]
- \tilde O\left(\frac{Hm}{\numfunctions\epsilon} + Hn2^{2m}\sqrt{\frac{m}{T}}\right).\]
\end{enumerate}
\end{theorem}

\begin{proof}
  Privacy follows from Lemma~\ref{lem:efficientPrivacy}. The utility guarantee
  follows from Lemma~\ref{lem:efficientPrivacy}, Theorem~\ref{thm:item_pricing}.
\end{proof}

\begin{theorem}[Full information online optimization for revenue maximization]\label{thm:pricing_full_info_revenue}
  Suppose that $u(\vec{v}, \vec{\rho})$ is the revenue of the posted price mechanism with prices $\vec{\rho}$ and buyers' values $\vec{v}$. Let \[u\left(\vec{v}^{(1)}, \cdot\right), \dots, u\left(\vec{v}^{(\numfunctions)}, \cdot\right)\] be the set of functions observed by Algorithm~\ref{alg:multi_d_online}, where each valuation vector $\vec{v}^{(t)}$ is drawn from a distribution $\dist^{(t)}$. Further, suppose we limit the parameter search space of Algorithm~\ref{alg:multi_d_online} to $[0,W]^m$, for some $W \in \R$. Algorithm~\ref{alg:multi_d_online} with input parameter $\lambda = \frac{1}{H}\sqrt{\frac{m}{\numfunctions} \log \left(dW\kappa \numfunctions\right)}$ has regret bounded as follows.
\begin{enumerate}
\item Suppose the buyers have additive valuations and for each distribution $\dist^{(t)}$,
the item values have $\kappa$-bounded marginal distributions. Then regret is bounded by $\tilde{O}\left(\sqrt{T}\left(Hnm + \frac{1}{\kappa}\right)\right)$.
\item Suppose the buyers are unit-demand with $v_j(\{i\}) \in [0,W]$ for each buyer $j \in [n]$ and item $i \in [m]$. Also, suppose that for each distribution
$\dist^{(t)}$, each buyer $j$, and every pair of items $i$ and $i'$, $v_j(\{i\})$ and $v_j(\{i'\})$ have a $\kappa$-bounded joint distribution.
 Then regret is bounded by $\tilde{O}\left(\sqrt{T}\left(Hnm^2 + \frac{1}{W\kappa}\right)\right)$.
\item Suppose the buyers have general valuations in $[0,W]$. Also, suppose that for each 
$\dist^{(t)}$, each buyer $j$, and every pair of bundles $b$ and $b'$, $v_j(b)$ and $v_j(b')$ have a $\kappa$-bounded joint distribution.
 Then regret is bounded by $\tilde{O}\left(\sqrt{T}\left(Hn2^{2m}\sqrt{m} + \frac{1}{W\kappa}\right)\right)$.
\end{enumerate}
\end{theorem}

\begin{theorem}[Full information online optimization for welfare maximization]\label{thm:pricing_full_info_welfare}
  Suppose that $u(\vec{v}, \vec{\rho})$ is the social welfare of the posted price mechanism with prices $\vec{\rho}$ and buyers' values $\vec{v}$. Let $u\left(\vec{v}^{(1)}, \cdot\right), \dots, u\left(\vec{v}^{(\numfunctions)}, \cdot\right)$ be the set of functions observed by Algorithm~\ref{alg:multi_d_online}, where each valuation vector $\vec{v}^{(t)}$ is drawn from a distribution $\dist^{(t)}$. Further, suppose we limit the parameter search space of Algorithm~\ref{alg:multi_d_online} to $[0,W]^m$, for some $W \in \R$. Algorithm~\ref{alg:multi_d_online} with input parameter $\lambda = \frac{1}{H}\sqrt{\frac{m}{\numfunctions} \log \left(dW\kappa \numfunctions\right)}$ has regret bounded as follows.
\begin{enumerate}
\item Suppose the buyers have additive valuations and for each distribution $\dist^{(t)}$,
the item values have $\kappa$-bounded marginal distributions. Then regret is bounded by $\tilde{O}\left(\sqrt{T}Hnm\right)$.
\item Suppose the buyers are unit-demand with $v_j(\{i\}) \in [0,W]$ for each buyer $j \in [n]$ and item $i \in [m]$. Also, suppose that for each distribution
$\dist^{(t)}$, each buyer $j$, and every pair of items $i$ and $i'$, $v_j(\{i\})$ and $v_j(\{i'\})$ have a $\kappa$-bounded joint distribution.
 Then regret is bounded by $\tilde{O}\left(\sqrt{T}Hnm^2\right)$.
\item Suppose the buyers have general valuations in $[0,W]$. Also, suppose that for each 
$\dist^{(t)}$, each buyer $j$, and every pair of bundles $b$ and $b'$, $v_j(b)$ and $v_j(b')$ have a $\kappa$-bounded joint distribution.
 Then regret is bounded by $\tilde{O}\left(\sqrt{Tm}Hn2^{2m}\right)$.
\end{enumerate}
\end{theorem}

\begin{proof}
  The proof follows from Theorem~\ref{thm:item_pricing} and Theorem~\ref{thm:1_d_online}.
\end{proof}

\begin{theorem}[Bandit feedback for revenue maximization]\label{thm:pricing_bandit_revenue} Suppose that $u(\vec{v}, \vec{\rho})$ is the revenue of the posted price mechanism with prices $\vec{\rho}$ and buyers' values $\vec{v}$. Let $u\left(\vec{v}^{(1)}, \cdot\right), \dots, u\left(\vec{v}^{(\numfunctions)}, \cdot\right)$ be the set of functions observed by the bandit algorithm from Section~\ref{sec:online}, where each valuation vector $\vec{v}^{(i)}$ is drawn from a distribution $\dist^{(i)}$. Regret is bounded as follows.
\begin{enumerate}
\item Suppose the buyers have additive valuations and for each distribution $\dist^{(t)}$,
the item values have $\kappa$-bounded marginal distributions.
 Then regret is bounded by \[\tilde{O}\left(T^{\frac{m+1}{m+2}}\left(H\sqrt{m\left(6W\sqrt{d}\kappa\right)^m} + \frac{1}{\kappa} + nm\right)\right).\]
\item Suppose the buyers are unit-demand with $v_j(\{i\}) \in [0,W]$ for each buyer $j \in [n]$ and item $i \in [m]$. Also, suppose that for each distribution
$\dist^{(t)}$, each buyer $j$, and every pair of items $i$ and $i'$, $v_j(\{i\})$ and $v_j(\{i'\})$ have a $\kappa$-bounded joint distribution.
 Then regret is bounded by \[\tilde{O}\left(T^{\frac{m+1}{m+2}}\left(H\sqrt{m\left(6W^2\sqrt{d}\kappa\right)^m} + \frac{1}{W\kappa} + nm^2\right)\right).\]
\item Suppose the buyers have general valuations in $[0,W]$. Also, suppose that for each 
$\dist^{(t)}$, each buyer $j$, and every pair of bundles $b$ and $b'$, $v_j(b)$ and $v_j(b')$ have a $\kappa$-bounded joint distribution.
 Then regret is bounded by \[\tilde{O}\left(T^{\frac{m+1}{m+2}}\left(H\sqrt{m\left(6W^2\sqrt{d}\kappa\right)^m} + \frac{1}{W\kappa} + n2^{2m}\sqrt{m}\right)\right).\]
\end{enumerate}
\end{theorem}

\begin{proof}
  The proof is the same as Theorem~\ref{thm:pricing_full_info_revenue}, except we
  use $\alpha = \frac{m+1}{m+2}-1$
  and apply Theorem~\ref{thm:banditRegret}.
\end{proof}

\begin{theorem}[Bandit feedback for welfare maximization]\label{thm:pricing_bandit_welfare} Suppose that $u(\vec{v}, \vec{\rho})$ is the social welfare of the posted price mechanism with prices $\vec{\rho}$ and buyers' values $\vec{v}$. Let $u\left(\vec{v}^{(1)}, \cdot\right), \dots, u\left(\vec{v}^{(\numfunctions)}, \cdot\right)$ be the set of functions observed by the bandit algorithm from Section~\ref{sec:online}, where each valuation vector $\vec{v}^{(i)}$ is drawn from a distribution $\dist^{(i)}$. Regret is bounded as follows.
\begin{enumerate}
\item Suppose the buyers have additive valuations and for each distribution $\dist^{(t)}$,
the item values have $\kappa$-bounded marginal distributions.
 Then regret is bounded by \[\tilde{O}\left(T^{\frac{m+1}{m+2}}\left(H\sqrt{m\left(6W\sqrt{d}\kappa\right)^m} + nm\right)\right).\]
\item Suppose the buyers are unit-demand with $v_j(\{i\}) \in [0,W]$ for each buyer $j \in [n]$ and item $i \in [m]$. Also, suppose that for each distribution
$\dist^{(t)}$, each buyer $j$, and every pair of items $i$ and $i'$, $v_j(\{i\})$ and $v_j(\{i'\})$ have a $\kappa$-bounded joint distribution.
 Then regret is bounded by \[\tilde{O}\left(T^{\frac{m+1}{m+2}}\left(H\sqrt{m\left(6W^2\sqrt{d}\kappa\right)^m} + nm^2\right)\right).\]
\item Suppose the buyers have general valuations in $[0,W]$. Also, suppose that for each 
$\dist^{(t)}$, each buyer $j$, and every pair of bundles $b$ and $b'$, $v_j(b)$ and $v_j(b')$ have a $\kappa$-bounded joint distribution.
 Then regret is bounded by \[\tilde{O}\left(T^{\frac{m+1}{m+2}}\left(H\sqrt{m\left(6W^2\sqrt{d}\kappa\right)^m} + n2^{2m}\sqrt{m}\right)\right).\]
\end{enumerate}
\end{theorem}

\begin{proof}
  The proof is the same as Theorem~\ref{thm:pricing_full_info_revenue}, except we
  use $\alpha = \frac{m+1}{m+2}-1$
  and apply Theorem~\ref{thm:banditRegret}.
\end{proof}

\subsection{Second-price item auctions with reserve prices}

Next, we turn to second-price item auctions. We prove the following theorem as a result of Theorem~\ref{thm:main_auction}.  The proof is similar to that of Theorem~\ref{thm:item_pricing}.

\secondPrice*

\begin{proof}
Let $\sample = \left\{\vec{v}^{\left(1\right)}, \dots, \vec{v}^{\left(\numfunctions\right)}\right\}$ be a set of valuation vectors and for each $t \in [\numfunctions]$ and $i \in [m]$, let $v^{\left(t\right)}\left(\left\{i\right\}\right) = \max_{j \in [n]} v_j^{\left(t\right)}\left(\left\{i\right\}\right)$. The buyer with the maximum valuation $v^{\left(t\right)}\left(\left\{i\right\}\right)$ for item $i$ under the valuation vector $\vec{v}^{\left(t\right)}$ is the only buyer who has a chance of winning the item and she will win it if and only if $v^{\left(t\right)}\left(\left\{i\right\}\right) \geq \rho_i$. Let $\Psi_{\vec{v}^{\left(t\right)}} = \left\{v^{\left(t\right)}\left(\left\{1\right\}\right) = \rho_1, \dots, v^{\left(t\right)}\left(\left\{m\right\}\right) = \rho_m\right\}$. As we range over prices in one connected component of $\R^m \setminus \Psi_{\vec{v}^{\left(t\right)}}$, the allocation of the auction is fixed. Since revenue and social welfare are Lipschitz when the allocation is fixed, we see that $\Psi_{\vec{v}^{\left(t\right)}}$ delineates the partition $\partition_{\vec{v}^{\left(t\right)}}$.

We now partition the multi-set union of $\Psi_{\vec{v}^{\left(1\right)}}, \dots, \Psi_{\vec{v}^{\left(\numfunctions\right)}}$
into $m$ multi-sets $\cB_1, \dots, \cB_{m}$ such that for each $\cB_i$, the
hyperplanes in $\cB_i$ are parallel with probability 1 over the draw of $\sample$
and the offsets of the hyperplanes in $\cB_i$ are independent random variables
with $\kappa$-bounded distributions. To this end, define $\cB_i = \left\{v^{\left(1\right)}\left(\left\{i\right\}\right) = \rho_i, \dots,
v^{\left(\numfunctions\right)}\left(\left\{i\right\}\right) =\rho_i\right\}.$ Clearly, these hyperplanes are parallel. Since we
assume that the distribution over $\max_{j \in [n]} v_j\left(\left\{i\right\}\right)$ is $\kappa$-bounded,
the offsets are independent draws from a $\kappa$-bounded distribution. Therefore, the theorem statement follows from Theorem~\ref{thm:main_auction}.
\end{proof}

\begin{theorem}[Differential privacy for revenue maximization]\label{thm:2nd_price_DP_rev}
  Let $u$ correspond to revenue.
  Suppose that for each $\dist^{(t)}$ and each item $i$, suppose the
  distribution over $\max_{j \in [n]} v_j(\{i\})$ is $\kappa$-bounded. Let  $\sample \sim \dist^\numfunctions$ be a set of samples. With
  probability at least $1-\delta$, if $\empvec{\rho}$ is the parameter vector
  returned by Algorithm~\ref{alg:efficient}, then
  \begin{align*}
  \E_{\vec{v} \sim \mathcal{D}}\left[u\left(\vec{v}, \hat{\vec{\rho}}\right)\right]
  \geq \max_{\vec{\rho}}\E_{\vec{v} \sim \mathcal{D}}[u\left(\vec{v}, \vec{\rho}\right)] - \tilde O\left(\frac{Hm}{\numfunctions\epsilon} + \frac{1}{\sqrt{\numfunctions}\kappa} + \frac{H m}{\sqrt{T}}\right).
  \end{align*}
  Moreover, this algorithm preserves $\left(\epsilon, \delta\right)$-differential privacy.
\end{theorem}

\begin{proof}
  Privacy follows from Lemma~\ref{lem:efficientPrivacy}. The utility guarantee
  follows from Lemma~\ref{lem:efficientPrivacy}, Theorem~\ref{thm:2nd_price},
  and Lemma~\ref{lem:sample_auctions}.
\end{proof}

\begin{theorem}[Differential privacy for welfare maximization]\label{thm:2nd_price_DP_SW}
  Let $u$ correspond to social welfare.
   Suppose that for each $\dist^{(t)}$ and each item $i$, suppose the
  distribution over $\max_{j \in [n]} v_j(\{i\})$ is $\kappa$-bounded. Let  $\sample \sim \dist^\numfunctions$ be a set of samples. With
  probability at least $1-\delta$, if $\empvec{\rho}$ is the parameter vector
  returned by Algorithm~\ref{alg:efficient}, then
  \begin{align*}
  \frac{1}{\numfunctions}\sum_{\vec{v}\in \sample}u\left(\vec{v}, \hat{\vec{\rho}}\right)
  \geq \max_{\vec{\rho}}\frac{1}{\numfunctions}\sum_{\vec{v}\in \sample}u\left(\vec{v}, \vec{\rho}\right) - \tilde O\left(\frac{Hm}{\numfunctions\epsilon} + \frac{H m}{\sqrt{T}}\right).
  \end{align*}
  Moreover, this algorithm preserves $\left(\epsilon, \delta\right)$-differential privacy.
\end{theorem}

\begin{proof}
  Privacy follows from Lemma~\ref{lem:efficientPrivacy}. The utility guarantee
  follows from Lemma~\ref{lem:efficientPrivacy}, Theorem~\ref{thm:2nd_price}.
\end{proof}

\begin{theorem}[Full information online optimization for revenue maximization]\label{thm:2nd_price_full_info_revenue}
  Let $u$ correspond to revenue. Let \[u\left(\vec{v}^{(1)}, \cdot\right), \dots, u\left(\vec{v}^{(\numfunctions)}, \cdot\right)\] be the set of functions observed by Algorithm~\ref{alg:multi_d_online}, where each valuation vector $\vec{v}^{(t)}$ is drawn from a distribution $\dist^{(t)}$. Suppose that for each $\dist^{(t)}$ and each item $i$, suppose the
  distribution over $\max_{j \in [n]} v_j(\{i\})$ is $\kappa$-bounded.Further, suppose we limit the parameter search space of Algorithm~\ref{alg:multi_d_online} to $[0,W]^m$, for some $W \in \R$. Algorithm~\ref{alg:multi_d_online} with input parameter $\lambda = \frac{1}{H}\sqrt{\frac{m}{\numfunctions} \log \left(dW\kappa \numfunctions\right)}$ has regret bounded by $\tilde{O}\left(\sqrt{T}\left(Hm + \frac{1}{\kappa}\right)\right)$.
\end{theorem}

\begin{theorem}[Full information online optimization for welfare maximization]\label{thm:2nd_price_full_info_welfare}
  Let $u$ correspond to welfare. Let \[u\left(\vec{v}^{(1)}, \cdot\right), \dots, u\left(\vec{v}^{(\numfunctions)}, \cdot\right)\] be the set of functions observed by Algorithm~\ref{alg:multi_d_online}, where each valuation vector $\vec{v}^{(t)}$ is drawn from a distribution $\dist^{(t)}$. Suppose that for each $\dist^{(t)}$ and each item $i$, suppose the
  distribution over $\max_{j \in [n]} v_j(\{i\})$ is $\kappa$-bounded.Further, suppose we limit the parameter search space of Algorithm~\ref{alg:multi_d_online} to $[0,W]^m$, for some $W \in \R$. Algorithm~\ref{alg:multi_d_online} with input parameter $\lambda = \frac{1}{H}\sqrt{\frac{m}{\numfunctions} \log \left(dW\kappa \numfunctions\right)}$ has regret bounded by $\tilde{O}\left(\sqrt{T}Hm\right)$.
\end{theorem}

\begin{proof}
  The proof follows from Theorem~\ref{thm:2nd_price} and Theorem~\ref{thm:1_d_online}.
\end{proof}

\begin{theorem}[Bandit feedback]\label{thm:2nd_price_bandit_revenue} Let $u$ be correspond to revenue. Let $u\left(\vec{v}^{(1)}, \cdot\right), \dots, u\left(\vec{v}^{(\numfunctions)}, \cdot\right)$ be the set of functions observed by the bandit algorithm from Section~\ref{sec:online}, where each valuation vector $\vec{v}^{(i)}$ is drawn from a distribution $\dist^{(i)}$. Suppose that for each $\dist^{(t)}$ and each item $i$, suppose the
  distribution over $\max_{j \in [n]} v_j(\{i\})$ is $\kappa$-bounded. Then regret is bounded by \[\tilde{O}\left(T^{\frac{m+1}{m+2}}\left(H\sqrt{m\left(6W\sqrt{d}\kappa\right)^m} + \frac{1}{\kappa} + m\right)\right).\]
\end{theorem}

\begin{proof}
  The proof is the same as Theorem~\ref{thm:2nd_price_full_info_revenue}, except we
  use $\alpha = \frac{m+1}{m+2}-1$
  and apply Theorem~\ref{thm:banditRegret}.
\end{proof}

\begin{theorem}[Bandit feedback]\label{thm:2nd_price_bandit_welfare} Let $u$ be correspond to social welfare. Let $u\left(\vec{v}^{(1)}, \cdot\right), \dots, u\left(\vec{v}^{(\numfunctions)}, \cdot\right)$ be the set of functions observed by the bandit algorithm from Section~\ref{sec:online}, where each valuation vector $\vec{v}^{(i)}$ is drawn from a distribution $\dist^{(i)}$. Suppose that for each $\dist^{(t)}$ and each item $i$, suppose the
  distribution over $\max_{j \in [n]} v_j(\{i\})$ is $\kappa$-bounded. Then regret is bounded by \[\tilde{O}\left(T^{\frac{m+1}{m+2}}\left(H\sqrt{m\left(6W\sqrt{d}\kappa\right)^m} + m\right)\right).\]
\end{theorem}

\begin{proof}
  The proof is the same as Theorem~\ref{thm:2nd_price_full_info_welfare}, except we
  use $\alpha = \frac{m+1}{m+2}-1$
  and apply Theorem~\ref{thm:banditRegret}.
\end{proof}

\subsection{Sample complexity guarantees}

\begin{lem}[\citet{Morgenstern16:Learning}]\label{lem:sample_auctions}
Let $\mathcal{M}$ be a class of mechanisms. Let $\sample \sim \dist^\numfunctions$ be a set of valuation vectors and let $u\left(\vec{v}, \vec{\rho}\right)$ denote the revenue of the mechanism in $\mathcal{M}$ parameterized by a vector $\vec{\rho}$ given buyer valuations $\vec{v}$.
The following guarantees hold.
\begin{itemize}
\item Suppose $\mathcal{M}$ is the class of item pricing auctions or the class of second price item auctions with anonymous reserves. Also, suppose the buyers have additive valuations. Then with probability at least $1-\zeta$, for all parameter vectors $\vec{\rho}$, \[\left|\frac{1}{\numfunctions}\sum_{\vec{v}\in \sample}u\left(\vec{v}, \vec{\rho}\right) - \E_{\vec{v} \sim \mathcal{D}}[u\left(\vec{v}, \vec{\rho}\right)]\right| \leq O\left(H\left(\sqrt{\frac{m \log m}{\numfunctions}} + \sqrt{\frac{\log\left(1/\zeta\right)}{\numfunctions}}\right)\right).\]
\item Suppose $\mathcal{M}$ is the class of item pricing mechanisms and the buyers have general valuations.
Then with probability at least $1-\zeta$, for all parameter vectors $\vec{\rho}$, \[\left|\frac{1}{\numfunctions}\sum_{\vec{v}\in \sample}u\left(\vec{v}, \vec{\rho}\right) - \E_{\vec{v} \sim \mathcal{D}}[u\left(\vec{v}, \vec{\rho}\right)]\right| \leq O\left(H\left(\sqrt{\frac{m^2}{\numfunctions}} + \sqrt{\frac{\log\left(1/\zeta\right)}{\numfunctions}}\right)\right).\]
\end{itemize}
\end{lem}

\section{Proofs for distributional learning (Section~\ref{sec:generalization})}\label{app:generalization}
We begin by recalling the definition of the pseudo-dimension of a class $\cF =
\{f : \Pi \to \reals\}$ of real-valued functions. We say that the set $\cF$
P-shatters a set $\sample = \{x_1, \dots, x_N\}$ if there exist thresholds $s_1,
\dots, s_N \in \reals$ such that for all subsets $E \subseteq \sample$ there
exists $f \in \cF$ such that $f(x_i) \geq s_i$ if $x_i \in E$ and $f(x_i) < s_i$
if $i \not \in E$. The Pseudo-dimension of a class $\cF$, denoted by
$\pdim(\cF)$ is the cardinality of the largest set $\sample$ that is P-shattered
by $\cF$.

\thmDispersionRademacher*
\begin{proof}
  The key idea is that whenever the functions $u_{x_1}, \dots, u_{x_N}$ are
  $(w,k)$-dispersed, we know that any pair of parameters $\vec{\rho}$ and
  $\vec{\rho}'$ with $\norm{\vec{\rho} - \vec{\rho'}}_2 \leq w$ satisfy
  $|f_{\vec{\rho}}(x_i) - f_{\vec{\rho}'}(x_i)| = |u_{x_i}(\vec{\rho}) -
  u_{x_i}(\vec{\rho}')| \leq Lw$ for all but at most $k$ of the elements in
  $\sample$. Therefore, we can approximate the functions in $\cF$ on the set
  $\sample$ with a finite subset $\hat \cF_w = \{ f_{\hat{\vec{\rho}}} \,:\,
  \hat{\vec{\rho}} \in \hat \configs_w \}$, where $\hat \configs_w$ is a $w$-net
  for $\configs$. Since $\hat \cF_w$ is finite, its empirical Rademacher
  complexity is $O((\log |\hat \cF_w|/N)^{1/2})$. We then argue that the
  empirical Rademacher complexity of $\cF$ is not much larger, since all
  functions in $\cF$ are approximated by some function in $\hat \cF_w$.

  In particular, we know that there exists a subset $\hat \configs_w \subset
  \configs$ of size $|\hat \configs_w| \leq (3R/w)^d$ (see
  Lemma~\ref{lem:netSize}) such that for every $\rho \in \configs$ there exists
  $\hat \rho \in \hat \configs_w$ satisfying $\norm{\rho - \hat \rho}_2 \leq w$.
  For any point $\rho \in \configs$, let $\nn(\rho)$ denote a point in $\hat
  \configs_w$ with $\norm{\rho - \nn(\rho)}_2 \leq w$. Let $\hat \cF_w = \{
  u_\rho : \Pi \to [0,1] \,|\, \rho \in \hat \configs-w\}$ be the corresponding
  finite subset of $\cF$.

  Since $\hat \cF_w$ is finite and the function range is $[0,1]$, we know that
  its empirical Rademacher complexity is at most
  \[
  O\biggl(\sqrt{\frac{\log |\hat \cF_w|}{N}}\biggr) = O\biggl(\sqrt{\frac{d\log(R/w)}{N}}\biggr).
  \]
  Next, fix any $f_\rho \in \cF$ and any vector $\sigma \in \{\pm 1\}^N$ of
  signs. We use $(w,k)$-dispersion to show that the correlation of
  $(f_\rho(x_1), \dots, f_\rho(x_N))$ with $\sigma$ cannot be substantially
  greater than the correlation of $(f_{\nn(\rho)}(x_1), \dots,
  f_{\nn(\rho)}(x_N))$ with $\sigma$.
  \begin{align*}
    \frac{1}{N} \sum_{i=1}^N \sigma_i f_\rho(x_i)
    &= \frac{1}{N} \sum_{i=1}^N \sigma_i u_{x_i}(\rho) \\
    &= \frac{1}{N} \sum_{i=1}^N \sigma_i u_{x_i}(\nn(\rho)) + \sum_{i=1}^N \sigma_i (u_{x_i}(\rho) - u_{x_i}(\nn(\rho))) \\
    &\leq \frac{1}{N} \sum_{i=1}^N \sigma_i u_{x_i}(\nn(\rho)) + \sum_{i=1}^N |u_{x_i}(\rho) - u_{x_i}(\nn(\rho))| \\
    &\leq \frac{1}{N} \sum_{i=1}^N \sigma_i u_{x_i}(\nn(\rho)) + Lw + \frac{k}{N} \\
    &= \frac{1}{N} \sum_{i=1}^N \sigma_i f_{\nn(\rho)}(x_i) + Lw + \frac{k}{N}
  \end{align*}
  Finally, we have
  \begin{align*}
    \hat R(\cF, S)
    &= \expect_{\sigma \sim \{\pm 1\}^N}\biggl[\sup_{f_\rho \in \cF}  \frac{1}{N} \sum_{i=1}^N \sigma_i f_{\rho}(x_i) \biggr] \\
    &\leq \expect_{\sigma \sim \{\pm 1\}^N}\biggl[\sup_{f_\rho \in \cF} \frac{1}{N} \sum_{i=1}^N \sigma_i f_{\nn(\rho)}(x_i) \biggr] + Lw + \frac{k}{N} \\
    &= \expect_{\sigma \sim \{\pm 1\}^N}\biggl[\sup_{f_{\hat \rho} \in \hat \cF_w} \frac{1}{N} \sum_{i=1}^N \sigma_i f_{\hat \rho}(x_i) \biggr] + Lw + \frac{k}{N} \\
    &= O\biggl(\sqrt{\frac{d \log(R/w)}{N}} + Lw + \frac{k}{N}\biggr),
  \end{align*}
  as required.

  The bound on $\hat R(\cF, \sample)$ in terms of the pseudo-dimension can be found in \citep{Pollard84:Convergence,Dudley67:Sizes}.
\end{proof}

\section{Discretization-based algorithm}\label{app:discretized}
In this section we provide an implementation of the exponential mechanism
achieving $(\epsilon,0)$-differential privacy. It applies to multi-dimensional
parameter spaces. First, we discretize the parameter space $\configs$ using a
regular grid (or any other net). We then apply the exponential mechanism to the
resulting finite set of outcomes. Let $\hat{\vec{\rho}}$ be the resulting
parameter. Standard guarantees for the exponential mechanism ensure that
$\hat{\vec{\rho}}$ is neraly optimal over the discretized set. Therefore, the
main challenge is showing that the net contains a parameter competitive with the
optimal parameter in $\configs$.

\begin{thm}
  Let $\sample = \{x_1, \dots, x_N\} \in \Pi$ be a collection of problem
  instances such that $u$ is piecewise $L$-Lipschitz and $(w,k)$-disperse. Let
  $\vec{\rho}_1$, \dots, $\vec{\rho}_K$ be a $w$-net for the parameter space
  $\configs$. Let $\hat{\vec{\rho}}$ be set to $\vec{\rho}_i$ with probability
  proportional to $\expmf^{\sample,\epsilon}(\vec{\rho}_i)$. Outputting
  $\hat{\vec{\rho}}$ satisfies $(\epsilon,0)$-differential privacy and with
  probability at least $1-\delta$ we have \[\frac{1}{N}\sum_{i = 1}^N u(x_i, \hat{\vec{\rho}}) \geq
 \frac{1}{N}\sum_{i = 1}^N u(x_i, \vec{\rho}^*) - \frac{2H}{N\epsilon} \log \frac{K}{\delta} -
  Lw - \frac{Hk}{N}. \]
\end{thm}
\begin{proof}[Proof sketch.]
  Since $\vec{\rho}_1, \dots, \vec{\rho}_K$ is a $w$-net for the parameter space
  $\configs$, we know there is some $\vec{\rho}_j$ within distance $w$ of
  $\vec{\rho^*}$. Also, since $B(\vec{\rho^*},w) \subset \configs$, we know
  that $\vec{\rho}_j$ is a valid parameter vector. As in the proof of
  Theorem~\ref{thm:cexpmutility} we know that $\frac{1}{N}\sum_{i = 1}^Nu(x_i, \vec{\rho}_j) \geq
  \frac{1}{N}\sum_{i = 1}^N u(x_i, \vec{\rho}^*) - \frac{Hk}{N} - Lw$. The result then follows
  from the standard analysis of the exponential mechanism, which guarantees that
  $\hat{\vec{\rho}}$ is competitive with the best $\vec{\rho}_j$ for $j \in \{1,
  \dots, K\}$.
\end{proof}

This algorithm has strengths and weaknesses when compared to
Algorithm~\ref{alg:efficient}. Recall,  Algorithm~\ref{alg:efficient} also
applies to the multi-dimensional setting. The main strength is that this
algorithm preserves pure $(\epsilon,0)$-differential privacy. However, there are
two significant disadvantages. First, it has running time exponential in the
dimension since a $w$-net for $\configs$ typically grows exponentially with
dimension. Second, it requires knowledge of an upper bound on the dispersion
parameter $w$ in order to choose the granularity of the net. This prevents us
from optimizing the utility guarantee over $w$ as we did in
Corollary~\ref{cor:expmutilityoptimized}. Moreover, decreasing the parameter $w$
increases the running time of the algorithm. This forces us to trade between
computational cost and accuracy.

\end{document}